%% file: main.tex
\documentclass{article}

\usepackage[preprint]{neurips_2025}

\usepackage{glossaries}

\input{acronyms}

\input{preamble}
\title{From Atoms to Entropy: Optimal Noise Allocation for Diffusion Training in the Convex Regime}

\author{%
  \begin{minipage}{\dimexpr\textwidth-2\tabcolsep\relax}\centering\bfseries
    Luca Ambrogioni\textsuperscript{1}\thanks{Equal contribution.}, Giulio Franzese\textsuperscript{2}\footnotemark[1], Alberto Foresti\textsuperscript{2}, Gabriel Raya\textsuperscript{3}, Bac Nguyen\textsuperscript{4}, Georgios Batzolis\textsuperscript{5}, Yuhta Takida\textsuperscript{4}, Naoki Murata\textsuperscript{4}, Chieh-Hsin Lai\textsuperscript{4} and Yuki Mitsufuji\textsuperscript{4,6}\\[6pt]
    \normalfont\small
    \textsuperscript{1}Radboud University \quad
    \textsuperscript{2}EURECOM \quad
    \textsuperscript{3}Tilburg University \& JADS \quad
    \textsuperscript{4}Sony AI \quad
    \textsuperscript{5}University of Cambridge \quad
    \textsuperscript{6}Sony Group Corporation
  \end{minipage}%
}

\begin{document}
\maketitle

\begin{abstract}
How should a diffusion model decide which noise levels to train on, and how much? Despite the importance of this choice, current noise schedules are based largely on heuristics or empirical tuning. Here, we develop a general statistical framework for studying asymptotically optimal noise-level allocation in diffusion training. Our first main result concerns the fully coupled regime, where information can spread between different time points. Under convexity or Polyak–Łojasiewicz-type assumptions, we show that the optimized training schedule admits an atomic minimizer, concentrated on finitely many noise levels. Our second main result specializes this framework to an idealized independent-learner regime, intended to model temporal specialization in neural networks. Under an additional feature-noise decoupling condition, a random-matrix analysis leads to an information-theoretic proxy: the decoupled sampling density is proportional to the square root of the generative entropy rate, the rate at which conditional entropy grows along the forward process. We test these predictions in controlled settings where the coupled objective can be optimized directly, including Dirac mixtures, low-dimensional manifolds, and MNIST. In these settings, the optimized schedules are consistently finite-support, while the smooth entropic proxy closely tracks the atomic optimum in neural-network models and breaks down mainly in the fully coupled parametric case, as the theory suggests. We then evaluate the entropic schedule in larger-scale experiments, where full schedule optimization is currently intractable. The results indicate that square-root entropy scheduling can substantially improve training efficiency on discrete domains and remains competitive with standard EDM-style heuristics on continuous images.

\end{abstract}

\input{sections/introduction}

\input{sections/related_work}

\input{sections/diffusion_models}
\input{sections/temporal_weighting}       
\input{sections/asymptotic_scaling}       
\input{sections/optimal_scheduling}       
\input{sections/uncoupled_limit}          
\input{sections/experiments}              
\input{sections/discussion}              

\bibliography{ref}
\bibliographystyle{plainnat}

\appendix

\input{appendices/elbo_denoiser}
\input{appendices/fokker_planck}
\input{appendices/random_jacobian}
\input{sections/feature_noise_decoupling} 
\input{appendices/toy_experiments}    
\input{appendices/experimental_details}

\end{document}

%% file: acronyms.tex
\newacronym{ODE}{\textsc{ode}}{Ordinary Differential Equation}
\newacronym{CTMC}{\textsc{ctmc}}{Continuous Time Markov Chain}
\newacronym{SEM}{\textsc{sem}}{Standard Error of the Mean}
\newacronym{KL}{\textsc{kl}}{Kullback-Leibler}
\newacronym{MLP}{\textsc{MLP}}{Multi-Layer-Perceptron}
\newacronym{NLP}{\textsc{NLP}}{Natural Language Processing}
\newacronym{GWAS}{\textsc{GWAS}}{Genome-Wide Association Study}
\newacronym{WSSS}{\textsc{WSSS}}{Weakly Supervised Semantic Segmentation}
\newacronym{CAM}{\textsc{CAM}}{Class Activation Map}
\newacronym{SDE}{\textsc{SDE}}{Stochastic Differential Equation}
\newacronym{DTC}{\textsc{DTC}}{Dual Total Correlation}
\newacronym{TC}{\textsc{TC}}{Total Correlation}
\newacronym{MI}{\textsc{MI}}{Mutual Information}
\newacronym{DWDSE}{\textsc{dwdse}}{Diffusion Weighted Denoising Score Entropy}
\newacronym{GP}{\textsc{GP}}{Gaussian Process}
\newacronym{CI}{\textsc{CI}}{Confidence Interval}
\newacronym{MEC}{\textsc{MEC}}{Minimum Entropy Coupling}

%% file: preamble.tex
\usepackage[utf8]{inputenc}
\usepackage[T1]{fontenc}
\usepackage{hyperref}
\usepackage{url}
\usepackage{booktabs}
\usepackage{amsfonts}
\usepackage{nicefrac}
\usepackage{microtype}
\usepackage{xcolor}
\usepackage{graphicx}
\usepackage{tikz}
\usetikzlibrary{arrows.meta, patterns}
\usepackage{float}
\usepackage{amsthm}
\usepackage{amsmath}
\usepackage{caption}
\usepackage{subcaption}
\usepackage{comment}
\usepackage{siunitx}
\usepackage{multirow}
\usepackage{mathrsfs}
\usepackage{amssymb}
\usepackage{algorithm}
\usepackage{algpseudocode}
\usepackage{enumitem}
\usepackage{wrapfig}
\usepackage{lipsum}

\usepackage{cleveref}

\newtheorem{theorem}{Theorem}[section]
\newtheorem{proposition}[theorem]{Proposition}

\newtheorem{lemma}[theorem]{Lemma}

\newcommand{\mean}[2]{\mathbb{E}_{#2}\left[ #1 \right]}

\newcommand{\norm}[1]{\left\lVert#1\right\rVert}

\DeclareMathOperator{\Ent}{Ent}
\DeclareMathOperator{\conv}{conv}

\input{macros}

%% file: sections/introduction.tex
\section{Introduction}
Diffusion models have become a leading framework for generative modeling across images, audio, and other high-dimensional domains \citep{sohl2015deep, ho2020ddpm, song2019generative, song2021score, dhariwal2021diffusion}. Central to their performance is the choice of how training samples are allocated across noise levels. Although numerous schedules have been proposed, ranging from heuristic designs \citep{nichol2021improved, kingma2021variational} to empirically tuned allocations \citep{karras2022elucidating} and adaptive or learned methods \citep{kingma2021variational,hang2023efficient, raya2026information}, there remains no generally accepted theoretical principle that explains why certain schedules work well while other are highly sub-optimal. 

In this paper, we develop a general statistical framework for asymptotically optimal noise-level allocation in diffusion training under streaming SGD \citep{robbins1951stochastic, polyak1992acceleration, bottou1999online}. To keep the problem theoretically tractable, we consider a convex loss landscape around the global optimum, or more generally, Polyak–Łojasiewicz assumptions \citep{polyak1963gradient, karimi2016linear}. Even under these restrictions, we will show that the analysis results in a very rich and counterintuitive phenomenology, where the coupled objective always admits an optimal schedule that is atomic, supported on finitely many noise levels. In the larger scale experiments, we apply some of the resulting formulas heuristically to unrestricted deep learning training. 

\subsection{Contributions}
While the work contains a substantial amount of experimental results ranging from toy models to natural image dataset, our main focus is in the theoretical analysis, which lays the foundations for the theoretical study of optimal noise allocation. We formulate the problem as choosing a training schedule, namely a distribution over noise levels, that minimizes the ELBO-weighted asymptotic prediction error. A key conceptual issue is the distinction between the error weighting function and the optimal training schedule, a distinction that is partially conflated in the previous literature and is central to our analysis.

The theoretical backbone of our work is an asymptotic analysis of averaged stochastic gradient descent in the equilibrium regime, where we track the long-run average of the iterates to obtain stable estimates of the parameter vector. This allows us to quantify how training on a data point at a given noise level asymptotically reduces prediction error along the full denoising path. As our analysis shows, this induces a trade-off between directly reducing error at challenging noise levels and “spreading” information across the time axis by sampling from easier noise levels. 

While the problem definition is general, our main results are derived in a frozen-feature regime: a local linearization around a converged parameter (the lazy/NTK regime \citep{jacot2018neural}), where the induced objective is convex. In this regime, our analysis yields an explicit asymptotic formula for the prediction error at any target noise level as a function of the training schedule, capturing how efficiently training steps are converted into denoising accuracy across the denoising path (Section~\ref{sec:asymptotic_scaling}). Using this formalism, we characterize two main regimes:

\paragraph{Coupled regime: atomic schedules.}
When the features learned by the network (or, more generally, the curvature geometry of the parametric model) couple inference at different time points, we prove the existence of an atomic optimizer, namely an optimized schedule that is a mixture of delta functions centered at finitely many noise levels. The result bounds the number of atoms but does not by itself prescribe their locations, which are determined by the curvature-and-noise geometry of the model and which we identify numerically through a matrix algorithm.

\paragraph{Decoupled regime: the square-root density.}
The coupled theory simplifies considerably when the network exhibits \emph{temporal specialization} where the internal representations that the network uses to denoise at one noise level are largely independent of those it uses at distant noise levels. In this regime the bin-level allocation problem admits a closed-form (water-filling) minimizer; combined with a feature--noise decoupling approximation, this yields the square-root entropy schedule
$$
m^*(t) \propto \sqrt{\dot{\Ent}[x_0 \mid x_t]},
$$
where $\dot{\Ent}[x_0 \mid x_t]$ is the rate at which conditional entropy grows. This identifies a stylized regime in which the entropy-based schedule of \citep{raya2026information} is provably optimal, lending theoretical support to that heuristic.

The coupled and decoupled theories are complementary: the former gives a general characterization of optimal schedules in the frozen-feature setting, while the latter yields an explicit data-dependent formula under simplifying conditions. Although these conditions are strong, they isolate the intrinsic information geometry of diffusion training and lead to an interpretable practical heuristic. 

We also apply the theory heuristically to actual deep learning training in the unconstrained regime, where we find that the uncoupled entropic schedulers has high performance when compared to highly optimized schedulers such as the ones provided in \cite{karras2022elucidating}.

%% file: sections/related_work.tex
\section{Related Work}
Optimizing how diffusion models allocate training or sampling effort across noise levels has been explored from several complementary perspectives. Early theoretical work derives principled temporal weightings from the continuous-time ELBO. \cite{kingma2021variational} show that the ELBO decomposes into an integral of the denoising loss weighted by the signal-to-noise ratio (SNR), motivating the widely used SNR-based loss weighting. While this establishes a meaningful notion of time importance, it does not yield a data-dependent optimal sampling density. A large empirical literature proposes improved noise schedules. \cite{nichol2021improved} introduced the cosine schedule, and \cite{karras2022elucidating} analyzed the diffusion design space empirically to recommend specific distributions of noise levels that improve stability and sampling accuracy. Perceptual and gradient-stability-based schemes, such as P2 weighting~\cite{choi2022perception} and Min-SNR weighting~\cite{hang2023efficient}, similarly adjust the emphasis placed on different noise regions during training, often prioritizing mid-noise intervals. A similar SNR based approach is discussed in \citep{hang2025improved}. Another line of work optimizes schedules algorithmically. \cite{wang2023learning} frame timestep selection as a reinforcement-learning problem, learning adaptive sampling schedules, and \cite{liu2023omsdpm} search over ``model schedules'' by training a surrogate quality predictor. These approaches show that time allocation affects performance, but they offer procedural rather than theoretical solutions and primarily target sampling rather than training-time allocation. A distinct set of works addresses numerical optimality in sampling. \cite{lu2022dpm} derive high-order ODE solvers for diffusion sampling, implicitly inducing non-uniform step distributions, while \cite{williams2024scoreoptimal} optimize discretization from a transport-cost perspective, producing ``score-optimal'' sampling schedules. Most closely related is the literature on information-theoretic time reparameterizations. \cite{stancevic2026entropic} show that the conditional generative entropy and its rate can be used to define an ``entropic time'' coordinate that improves sampling when steps are placed uniformly in that transformed space. This approach was recently used in \citep{raya2026information}, where entropy is used heuristically as noise scheduling distribution. Our work offers a theoretical justification of this latter work and also shows that the technique is sub-optimal under strong temporal coupling.

%% file: sections/diffusion_models.tex
\section{Generative diffusion models}
For the sake of notational simplicity, we will consider a generative diffusion model with forward process governed by a Brownian dynamics:
$
    \text{d} x_{t} = a(t) \text{d} w_t~,
$
where $a(t)$ is a positive-valued noise scheduling function and $w_t$ is a Wiener process. All results in the paper extend to general linear forward processes. The transition finite-time distribution of this model is given by the formula
$
    x_t = x_0 + \sigma(t) ~z,~
$
where $z \sim \mathcal{N}(0,1)$ and $\sigma^2(t)$ is the cumulative variance of the noise:
$
    \sigma^2(t) = \int_0^t a^2(\tau) \text{d} \tau
$
Now consider that the data at $t = 0$ is sampled from a data distribution $\phi(x_0)$. Given a data distribution $\phi$, the generative dynamics, the so called \emph{reverse process}, follows a reverse time stochastic differential equation~\citep{anderson1982reverse}:
\begin{equation}
    \text{d} x_t = \frac{a^2(t)}{\sigma^2(t)} \left(x_t - \mean{x_0 \mid x_t, \phi}{} \right)  \text{d} t + a(t) \text{d} \tilde{w}_t
\end{equation}
where $\tilde{w}_t$ is a separate Wiener process. In generative diffusion training, the aim is to learn how to sample from the distribution $\phi$ only accessible through a finite dataset $D = \{y^{1}, \dots, y^N\}$ independently sampled from $\phi$. To this purpose, a denoising neural network $s_\theta(x_t,t)$ is trained to approximate the denoising expected value $\mean{x_0 \mid x_t}{}$, which is then used to integrate the reverse dynamics. For a given time $t$, the ideal loss (density) is
\begin{equation}\label{eq: denoising loss}
    \mathscr{L}_t(\theta) = \mean{\norm{s_\theta(x_t,t) - \mean{x_0 \mid x_t, \phi}{}}^2}{x_t} = \underbrace{\mean{\norm{s_\theta(x_t,t) - x_0}^2}{x_t, x_0}}_{\hat{\mathscr{L}}_t(\theta)~: \text{ denoising autoencoder loss}} - c~,
\end{equation}
where the optimal error $c$ is a constant independent from the network $s_\theta(x_t,t)$ and can therefore be ignored during optimization. Note that the \emph{denoising autoencoder} error density $\hat{\mathscr{L}}_t(\theta)$~\citep{vincent2011connection} can be estimated from the data, being averaged over the joint distribution of $x_t$ and $x_0$.

%% file: sections/temporal_weighting.tex
\section{Temporal weighting and optimal sampling densities}\label{sec:temporal_weighting}
Using $\hat{\mathscr{L}}_t$ as a loss function for a generative diffusion model is theoretically principled but requires a measure of relative importance between time points. In fact, in a generative diffusion model, the time parameter $t$ is ultimately arbitrarily defined and it is therefore not invariant under a change of time variable. We can obtain a well-defined loss on the whole diffusion trajectory by introducing a weighting function $w(t)$, which can be used to define the global error by integrating over the diffusion interval $[0,T]$. A principled weighting is given by the continuous ELBO, which provides an upper bound on the negative log-likelihood (see \citep{huang2021variational} and Supp.~\ref{supp sec: ELBO}):
\begin{equation}
    \mathcal{L}_{\text{ELBO}} = \int_0^T \hat{\mathscr{L}_t} ~\dot{\text{SNR}}(t) \text{d} t
\end{equation}
where for our forward process we define the (signed) SNR rate as $\dot{\text{SNR}}(t) = \dot{\sigma}(t)/\sigma^3(t)$. It is tempting to interpret $\dot{\text{SNR}}(t)$ as a time sampling density. However, optimal sampling must be kept separate from time weighting, since the performance of a time sampler depends on how the average error scales with the size of the training set at different time points, and it is therefore a data-dependent quantity. 

The empirical autoencoder loss $\hat{\mathscr{L}}_t(\theta)$ decomposes as
$
  \hat{\mathscr{L}}_t(\theta) \;=\; \mathscr{L}_t(\theta) \;+\; \mathrm{MMSE}(t),
$
where, for a parameter vector $\theta$,
\begin{equation}\label{eq:Lt_def}
  \mathscr{L}_t(\theta) \;:=\; \mathbb{E}_{x_t}\bigl[\|s_\theta(x_t,t)-f_t(x_t)\|^2\bigr]
\end{equation}
is the prediction error of $s_\theta$ relative to the Bayes-optimal denoiser
$f_t(x_t):=\mathbb{E}[x_0\mid x_t,t]$, and
$\mathrm{MMSE}(t):=\mathbb{E}[\|x_0-f_t(x_t)\|^2]$ is the irreducible
Bayes error (Supp.~\ref{supp sec: ELBO}). Minimizing the ELBO over $\theta$ is
therefore equivalent to minimizing the SNR-weighted prediction error
$\int_0^T\mathscr{L}_\tau(\theta)\,\dot{\mathrm{SNR}}(\tau)\,d\tau$. 

\subsection{SGD optimization}
In practice, $\theta$ is not optimized exactly but is approximated by a \emph{random}
iterate produced by stochastic gradient descent. Let $\bar\theta_K$ denote the
Polyak--Ruppert averaged iterate (see \citep{polyak1992acceleration}) after $K$ SGD steps under sampling measure $\mu$.
Because $\bar\theta_K$ is a
random variable whose distribution depends on both $K$ and $\mu$, the prediction error
$\mathscr{L}_\tau(\bar\theta_K)$ is itself random. The natural optimization target is
therefore the \emph{expected} prediction error
\begin{equation}\label{eq:Ltau_Kmu_def}
  \mathscr{L}_\tau(K;\mu)
  \;:=\;
  \mathbb{E}\bigl[\mathscr{L}_\tau(\bar\theta_K)\bigr]
  \;=\;
  \mathbb{E}\!\left[\bigl\|s_{\bar\theta_K}(x_\tau,\tau)-f_\tau(x_\tau)\bigr\|^2\right],
\end{equation}
where the expectation is joint over the SGD trajectory and the test point $x_\tau$.

\subsection{Optimal noise allocation}
We formulate the optimal time allocation as a probability measure
$\mu^* \in \mathcal{P}([0,T])$ minimizing the ELBO-weighted expected prediction error \citep{kingma2021variational}:
\begin{equation}\label{eq:minimizer}
  \mu^* = \arg\min_{\mu \in \mathcal{P}([0,T])} \mathcal{J}(\mu),
  \qquad
  \mathcal{J}(\mu) = \int_0^T \dot{\mathrm{SNR}}(\tau)\,
    \lim_{K \to \infty} K\,\mathscr{L}_\tau(K;\mu)\;d\tau,
\end{equation}
where $K$ is the total sampling budget.
The factor $\lim_{K\to\infty}K\,\mathscr{L}_\tau(K;\mu)$ is the \emph{normalized
asymptotic prediction rate}, a finite quantity independent of $K$: for any large budget
$K$, the expected prediction error satisfies $\mathscr{L}_\tau(K;\mu)=C_\tau(\mu)/K+O(K^{-2})$,
where $C_\tau(\mu):=\lim_{K\to\infty}K\,\mathscr{L}_\tau(K;\mu)$. Interchanging the limit with
the $\tau$-integral in~\eqref{eq:minimizer} is justified when the $O(K^{-2})$ remainder is
uniform in $\tau$ and $\dot{\mathrm{SNR}}(\tau)\,C_\tau(\mu)$ is integrable on $[0,T]$, so that
dominated convergence applies.
Note that $\mu$ need not be absolutely continuous: it may concentrate on finitely many
noise levels (Dirac atoms), form a mixture of atoms and a smooth component, or be any
Borel probability measure on $[0,T]$. The optimal schedule in the coupled regime is
generically atomic, while in the independent-learner regime the optimizer spreads mass
into a smooth density. Note that the specification of a weighting function is essential in
the definition of an optimal sampling measure.

%% file: sections/asymptotic_scaling.tex
\section{Asymptotic scaling of the score matching error}
\label{sec:asymptotic_scaling}
We now consider the realistic setting in which a single denoising network is trained across all
diffusion times with shared parameters. This induces temporal coupling: stochastic gradient updates
obtained at one noise level modify the denoiser at all other noise levels. Our goal in this section is
to derive the asymptotic prediction error at an arbitrary evaluation time $\tau$ in a streaming SGD
regime, and to express this error in a form that makes the dependence on the time sampling measure
$\mu\in\mathcal{P}([0,T])$ explicit. Let $s_\theta(x,t)$ denote a denoiser network with a single parameter vector
$\theta\in\mathbb{R}^p$. Training uses the standard squared denoising loss
$\ell(\theta;t,x_0,x_t)=\frac12\|s_\theta(x_t,t)-x_0\|^2$.
Each streaming SGD update is generated by sampling
$t\sim \mu$, $x_0\sim\phi$, and $x_t\sim q_t(\cdot\mid x_0)$,
where $\mu\in\mathcal{P}([0,T])$ is the time-sampling measure. Let $\theta^\star\in\mathbb{R}^p$ denote the global population minimizer of the
squared denoising loss.

\paragraph{Assumption A (Well-specified model; $\mu$-independence of
$\theta^\star$).}\label{ass:well_specified}
We assume there exists a unique $\theta^\star$ such that
$s_{\theta^\star}(x_t,t)=f_t(x_t):=\mathbb{E}[x_0\mid x_t,t]$
for \emph{all} $t\in[0,T]$ simultaneously.
In particular, $\theta^\star$ does not depend on the time-sampling measure~$\mu$:
varying $\mu$ reshapes the geometry of the SGD trajectory but leaves the
target parameter fixed.
This assumption is essential for the atomic-optimizer result: varying $\mu$ changes the
optimization geometry but not the target, so the asymptotic prediction error depends on
$\mu$ only through the aggregate second-moment operators $A_t$ and $B_t$.

Define the residual $\varepsilon_t:=x_0-f_t(x_t)$, which satisfies
$\mathbb{E}[\varepsilon_t\mid x_t,t]=0$ and
$\mathrm{Var}(\varepsilon_t\mid x_t,t)=\Sigma_t(x_t)$.

\paragraph{Linearization and local geometry.}
We work in the local linear regime around $\theta^\star$. Define the Jacobian
$J_t(x_t):=\partial_\theta s_\theta(x_t,t)|_{\theta=\theta^\star}\in\mathbb{R}^{d\times p}$,
and the associated Gram matrix $G_t(x_t):=J_t(x_t)^\top J_t(x_t)$.
A first-order expansion gives
$s_\theta(x_t,t)=s_{\theta^\star}(x_t,t)+J_t(x_t)(\theta-\theta^\star)+R_t(\theta,x_t)$,
where the remainder satisfies $R_t(\theta,x_t)=O(\|\theta-\theta^\star\|^2)$ in $L_2$. We consider streaming SGD with a decreasing step size $\eta_n\to 0$,
\begin{equation}
\theta_{n+1}=\theta_n-\eta_n\nabla_\theta\ell(\theta_n;t_n,x_0^{(n)},x_{t_n}^{(n)}),
\label{eq:sgd_update_general}
\end{equation}
where $(t_n,x_0^{(n)},x_{t_n}^{(n)})$ are sampled i.i.d.\ as described above. Although deep networks are nonconvex, a growing body of theory supports the use of Polyak--\L{}ojasiewicz-type
assumptions in overparameterized regimes, implying stability of streaming SGD near $\theta^\star$.
We therefore analyze the Polyak--Ruppert averaged iterate
$\bar\theta_K=\frac1K\sum_{n=1}^K\theta_n$.

Using the linearization, the per-sample gradient takes the form
$\nabla_\theta\ell(\theta;t,x_0,x_t)=J_t(x_t)^\top(s_\theta(x_t,t)-x_0)$.
At $\theta^\star$ this reduces to $-J_t(x_t)^\top\varepsilon_t$. Define the pointwise curvature and noise operators
\begin{equation}
A_t := \mathbb{E}_{x_t\mid t}[G_t(x_t)],
\qquad
B_t := \mathbb{E}_{x_t\mid t}\!\left[J_t(x_t)^\top\Sigma_t(x_t)J_t(x_t)\right],
\label{eq:pointwise_ops}
\end{equation}
Although $A_t$ is only the Gauss--Newton term of the loss Hessian, at $\theta^\star$ it coincides with the \emph{exact} population Hessian: the residual $s_{\theta^\star}(x_t,t)-x_0=-\varepsilon_t$ is conditionally centered, $\mathbb{E}[\varepsilon_t\mid x_t]=0$, while $\nabla_\theta^2 s$ is $x_t$-measurable, so the neglected second-order term vanishes in expectation. The global curvature and gradient-noise operators induced by the sampling measure $\mu$ are
\begin{equation}
H[\mu]
=\int A_t\,d\mu(t),
\label{eq:global_curvature}
\end{equation}
\begin{equation}
\Gamma[\mu]
=\int B_t\,d\mu(t).
\label{eq:global_noise}
\end{equation}
When $\mu$ is absolutely continuous with density $m(t)=d\mu/dt$, these reduce to
$H[m]=\int m(t)A_t\,dt$ and $\Gamma[m]=\int m(t)B_t\,dt$.
Under standard regularity assumptions (finite second moments, local smoothness, and nondegeneracy
of $H[\mu]$), Polyak--Ruppert averaging yields the asymptotic
normality~\citep{polyak1992acceleration}
\begin{equation}
\sqrt{K}\,(\bar\theta_K-\theta^\star)
\;\Rightarrow\;
\mathcal{N}\!\left(0,\,H[\mu]^{-1}\Gamma[\mu]H[\mu]^{-1}\right).
\label{eq:parameter_clt_general}
\end{equation}

Fix an evaluation time $\tau$. Consistent with the definition in
Section~\ref{sec:temporal_weighting} (Eq.~\eqref{eq:Ltau_Kmu_def}), the expected
denoiser-matching error is
\begin{equation}
\mathscr{L}_\tau(K;\mu)
\;=\;
\mathbb{E}\!\left[\bigl\|s_{\bar\theta_K}(x_\tau,\tau)-f_\tau(x_\tau)\bigr\|^2\right],
\label{eq:denoise_error_def}
\end{equation}
where the expectation is \emph{joint} over $x_\tau\sim q_\tau$ (a fresh test point,
independent of the training sequence) and the SGD randomness in $\bar\theta_K$.
The rescaled quantity $K\,\mathscr{L}_\tau(K;\mu)$ is the natural figure of merit for
comparing time-sampling strategies, converging to a finite, schedule-dependent constant
that characterizes the asymptotic efficiency of~$\mu$. We now derive this constant.

\paragraph{Step 1 -- Linearization.}
The first-order expansion at $\theta^\star$ gives
\[
s_{\bar\theta_K}(x_\tau,\tau)-f_\tau(x_\tau)
=J_\tau(x_\tau)(\bar\theta_K-\theta^\star)+R_\tau(\bar\theta_K,x_\tau),
\]
where $\|R_\tau(\bar\theta_K,x_\tau)\|=O_p(\|\bar\theta_K-\theta^\star\|^2)=O_p(K^{-1})$
in $L_2$. Since $\bar\theta_K-\theta^\star=O_p(K^{-1/2})$ from \eqref{eq:parameter_clt_general},
\[
\bigl\|s_{\bar\theta_K}(x_\tau,\tau)-f_\tau(x_\tau)\bigr\|^2
=\bigl\|J_\tau(x_\tau)(\bar\theta_K-\theta^\star)\bigr\|^2+o_p(K^{-1}).
\]

\paragraph{Step 2 -- Factorization of the joint expectation.}
Writing the squared norm as a quadratic form and taking expectations,
\begin{align*}
\mathscr{L}_\tau(K;\mu)
&=\mathbb{E}\!\left[
  \mathrm{Tr}\!\bigl(G_\tau(x_\tau)\,
  (\bar\theta_K-\theta^\star)(\bar\theta_K-\theta^\star)^\top\bigr)\right]+o(K^{-1}).
\end{align*}
Since $x_\tau$ is a fresh test point drawn independently of the training sequence,
$x_\tau\perp\bar\theta_K$ and the expectation factorizes:
\begin{equation}
\mathscr{L}_\tau(K;\mu)
=\mathrm{Tr}\!\left(A_\tau\,\Sigma_K\right)+o(K^{-1}),
\label{eq:Lt_trace_form}
\end{equation}
where $\Sigma_K:=\mathbb{E}[(\bar\theta_K-\theta^\star)(\bar\theta_K-\theta^\star)^\top]$
is the parameter covariance and $A_\tau=\mathbb{E}_{x_\tau}[G_\tau(x_\tau)]$
as in \eqref{eq:pointwise_ops}.

\paragraph{Step 3 -- CLT substitution.}
The Polyak--Ruppert CLT \eqref{eq:parameter_clt_general} controls the distribution of
$\bar\theta_K$; under the uniform integrability of $\{K\|\bar\theta_K-\theta^\star\|^2\}$
supplied by the bounded-moment conditions of~\citet{polyak1992acceleration}, weak
convergence upgrades to convergence of second moments, so
$K\,\Sigma_K\to H[\mu]^{-1}\Gamma[\mu]H[\mu]^{-1}$ as $K\to\infty$.
Substituting into \eqref{eq:Lt_trace_form}:
\begin{equation}
\lim_{K\to\infty} K\,\mathscr{L}_\tau(K;\mu)
=
\mathrm{Tr}\!\left(
A_\tau
\,H[\mu]^{-1}\Gamma[\mu]H[\mu]^{-1}
\right).
\label{eq:general_prediction_error}
\end{equation}

Integrating~\eqref{eq:general_prediction_error} against the ELBO weight
$\dot{\mathrm{SNR}}(\tau)\,d\tau$ gives the full schedule-optimization objective:
\begin{equation}
\mathcal{J}(\mu)
\;=\;
\int_0^T \dot{\mathrm{SNR}}(\tau)\,
\mathrm{Tr}\!\left(
A_\tau\,H[\mu]^{-1}\,\Gamma[\mu]\,H[\mu]^{-1}
\right)d\tau
\;=\;
\mathrm{Tr}\!\left(
\bar A\,H[\mu]^{-1}\,\Gamma[\mu]\,H[\mu]^{-1}
\right),
\label{eq:general_obj}
\end{equation}
where $\bar A:=\int_0^T\dot{\mathrm{SNR}}(\tau)\,A_\tau\,d\tau$ is the SNR-weighted
average curvature. The second equality uses the linearity of the trace and shows that
$\mathcal{J}(\mu)$ depends on~$\mu$ only through the pair
$(H[\mu],\Gamma[\mu])$.

\paragraph{Cross-time impact: the kernel representation.}
Both $\mu$-dependent operators in~\eqref{eq:general_prediction_error} aggregate
information across \emph{all} training times. The gradient-noise operator
$\Gamma[\mu]=\int B_t\,d\mu(t)$ is the \emph{source} term: it collects the noise injected
at every training time~$t$. The global inverse $H[\mu]^{-1}$ is the \emph{propagator}:
being the inverse of an aggregate curvature, it routes each time's contribution through
the shared parameters. Consequently the prediction error at any evaluation time~$\tau$
depends on the curvature and noise at all training times, not just on~$A_\tau$. To make
this explicit, substitute $\Gamma[\mu]=\int B_t\,d\mu(t)$
into~\eqref{eq:general_prediction_error} and use, in turn, the linearity of the trace and
integral, the definition $B_t=\mathbb{E}_{x_t\mid t}[J_t^\top\Sigma_t J_t]$
of~\eqref{eq:pointwise_ops}, and the cyclic invariance of the trace (arguments~$x_t$
suppressed under the expectation):
\begin{align*}
\lim_{K\to\infty} K\,\mathscr{L}_\tau(K;\mu)
&=\mathrm{Tr}\!\bigl(A_\tau\,H[\mu]^{-1}\,\Gamma[\mu]\,H[\mu]^{-1}\bigr)
 =\int \mathrm{Tr}\!\bigl(A_\tau\,H[\mu]^{-1}B_t\,H[\mu]^{-1}\bigr)\,d\mu(t)
\\[2pt]
&=\int \mathbb{E}_{x_t\mid t}\!\left[
   \mathrm{Tr}\!\bigl(A_\tau\,H[\mu]^{-1}J_t^\top\Sigma_t J_t\,H[\mu]^{-1}\bigr)
   \right]d\mu(t)
\\[2pt]
&=\int \mathbb{E}_{x_t\mid t}\!\left[
   \mathrm{Tr}\!\bigl(
   \underbrace{J_t\,H[\mu]^{-1}A_\tau\,H[\mu]^{-1}J_t^\top}_{=\;P_{\tau,t}}\,
   \Sigma_t\bigr)
   \right]d\mu(t).
\end{align*}
The final step cycles the trailing factor $J_t\,H[\mu]^{-1}$ to the front, regrouping the
sandwich into the single operator~$P_{\tau,t}$. The prediction error at~$\tau$ is
therefore an integral over training times~$t$, weighted by the sampling measure~$\mu$,
\begin{equation}
\lim_{K\to\infty} K\,\mathscr{L}_\tau(K;\mu)
=
\int
\mathbb{E}_{x_t\mid t}
\!\left[
\mathrm{Tr}\!\left(P_{\tau,t}(x_t)\,\Sigma_t(x_t)\right)
\right]d\mu(t),
\label{eq:kernel_error_general}
\end{equation}
where
\begin{equation}
P_{\tau,t}(x_t)
=
J_t(x_t)\,H[\mu]^{-1}\,
A_\tau\,
H[\mu]^{-1}\,J_t(x_t)^\top
\label{eq:projection_operator_general}
\end{equation}
is a projection-like operator that measures the \emph{cross-time overlap} between
training at time~$t$ and evaluation at time~$\tau$: the columns of $J_t$ (features
learned at~$t$) are projected through the shared-parameter geometry ($H[\mu]^{-1}$) and
weighted by the evaluation curvature ($A_\tau$). The trace
$\mathrm{Tr}(P_{\tau,t}\,\Sigma_t)$ then quantifies how much of the conditional
uncertainty $\Sigma_t(x_t)$ at training time~$t$ leaks into prediction error at
evaluation time~$\tau$. When features at~$t$ and~$\tau$ use the same parameter
directions, $P_{\tau,t}$ is large and uncertainty propagates across times; when they use
orthogonal directions, $P_{\tau,t}$ vanishes and the two noise levels decouple.
The two $\mu$-dependent operators play distinct roles in this coupling. The source
$\Gamma[\mu]=\int B_t\,d\mu(t)$ supplies the integral over training times---a sum of
per-time noise contributions---but it is \emph{linear} in~$\mu$ and so, on its own, would
contribute additively. The propagator $H[\mu]^{-1}$ is the only factor \emph{nonlinear}
in~$\mu$: as the inverse of an aggregate curvature it does not factor over times, so noise
injected at~$t$ generically contaminates the prediction at $\tau\neq t$. This
non-factoring is what keeps the off-diagonal ($t\neq\tau$) part of $P_{\tau,t}$ alive and
is the source of the atomic-schedule phenomenon. When $P_{\tau,t}$ localizes to the diagonal
$t=\tau$, the integral over training times~$t$ collapses to that single point, and the
optimal schedule transitions from atomic to smooth.

Since $\lim_{K\to\infty}K\,\mathscr{L}_\tau(K;\mu)$ depends on $\mu$ only through the
pair $(H[\mu],\Gamma[\mu])\in\mathbb{S}^p\times\mathbb{S}^p$, any continuous objective of
the form $F(H[\mu],\Gamma[\mu])$, including $\mathcal{J}(\mu)$, admits an atomic minimizer
supported on at most $p(p+1)+1$ noise levels---a consequence of the convex geometry of the
moment map $\mu\mapsto(H[\mu],\Gamma[\mu])$.

%% file: sections/optimal_scheduling.tex
\section{Optimal asymptotic time scheduling in the idealized regime}\label{sec:optimal_scheduling}
The objective $\mathcal{J}(\mu)=\mathrm{Tr}\!\bigl(\bar A\,H[\mu]^{-1}\Gamma[\mu]\,
H[\mu]^{-1}\bigr)$ from~\eqref{eq:general_obj} has a convex-geometric structure that guarantees the existence of an \emph{atomic}
minimizer, supported on finitely many noise levels. The mechanism
is that $\mathcal{J}$ sees $\mu$ only through a finite-dimensional summary.

\paragraph{Geometric framework.}
Collect the pointwise operators of~\eqref{eq:pointwise_ops} into the curve
$c(t):=(A_t,B_t)\in V$, where $V:=\mathbb{S}^p\times\mathbb{S}^p$ is the space of pairs of
symmetric matrices. Via the $\mathrm{vech}$ map (stacking upper-triangular entries, with
off-diagonal entries scaled by $\sqrt2$ to preserve the Frobenius inner product), $V$ is
isomorphic to $\mathbb{R}^{d}$ with $d:=p(p+1)$. Define the \emph{barycenter map}
\begin{equation}\label{eq:barycenter}
\Phi(\mu):=\int_0^T c(t)\,d\mu(t)=\bigl(H[\mu],\,\Gamma[\mu]\bigr)\in V,
\end{equation}
a Bochner integral that is well defined for every $\mu\in\mathcal{P}([0,T])$ because $c$ is
continuous, hence bounded and strongly measurable, on the compact interval $[0,T]$. Both
$H[\mu]$ and $\Gamma[\mu]$ are linear in $\mu$, so $\Phi$ is affine, and by~\eqref{eq:general_obj}
the objective factors through it:
\begin{equation}\label{eq:abstract_problem}
\mathcal{J}(\mu)=F\!\bigl(\Phi(\mu)\bigr),
\qquad
F(H,\Gamma):=\mathrm{Tr}\!\bigl(\bar A\,H^{-1}\Gamma\,H^{-1}\bigr),
\end{equation}
where $F$ is continuous on $\{(H,\Gamma):H\succ0\}$. The scheduling problem thus lives in
the fixed, finite-dimensional space $V\cong\mathbb{R}^{p(p+1)}$, no matter how complex the
measure $\mu$ (Figure~\ref{fig:caratheodory}).

\begin{figure}[t]
  \centering
  \input{figures/caratheodory_tikz}
  \caption{\textbf{Convex-geometry picture of the atomic optimizer.} The operator curve
  $c(t)=(A_t,B_t)$ traces a generally non-convex set in the moment space $V$.
  The barycenter map $\Phi$ carries every sampling measure into the convex hull of the curve,
  and every hull point is realized this way: the shaded region is \emph{both}
  $\Phi(\mathcal{P}([0,T]))$ (green hatching) and $\conv(c([0,T]))$ (solid), the two coinciding. By
  Carath\'eodory's theorem each hull point is a convex combination of at most $d+1$ curve
  points (with $d=\dim V=p(p+1)$; the planar sketch has $d=2$), so the minimizer of
  $\mathcal{J}=F\circ\Phi$, whose argument is the barycenter $(H,\Gamma)$, in orange, is
  realized by an atomic schedule. Two example barycenters $\Phi(\mu_1),\Phi(\mu_2)$, each a
  combination of two curve points, are shown.}
  \label{fig:caratheodory}
\end{figure}
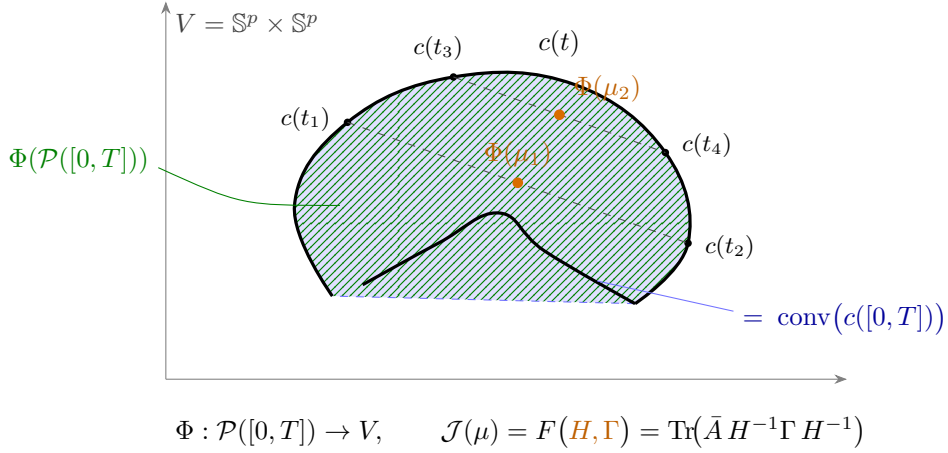

\begin{theorem}[Carathéodory]\label{thm:caratheodory}
Let $S\subset\mathbb{R}^{d}$. Every point of $\operatorname{conv}(S)$ is a convex
combination of at most $d+1$ points of $S$.
\end{theorem}

\begin{proposition}[Existence of an atomic optimizer]\label{prop:atomic_optimizer}
Assume that $t \mapsto A_t$ and $t \mapsto B_t$ are continuous on $[0,T]$, and that $A_t \succ 0$ for every $t \in [0,T]$. Then the coupled ELBO minimization
$\min_{\mu \in \mathcal{P}([0,T])} \mathcal{J}(\mu)$
admits an atomic optimizer $\mu^*_{\mathrm{at}} = \sum_{i=1}^r w_i \delta_{t_i}$ with $r \le p(p+1)+1$ atoms.
\end{proposition}
\begin{proof}
By~\eqref{eq:abstract_problem} it suffices to minimize $F$ over the image
$\Phi(\mathcal{P}([0,T]))$. This image is exactly $\operatorname{conv}(c([0,T]))$: any
finite convex combination $\sum_i\lambda_i c(t_i)$ is realized by the atomic measure
$\sum_i\lambda_i\delta_{t_i}$, and conversely every barycenter $\Phi(\mu)$ lies in
$\operatorname{conv}(c([0,T]))$ (approximate the Bochner integral by Lebesgue sums of $c$,
each a convex combination of values of $c$; since $c([0,T])$ is compact---a continuous image
of $[0,T]$---its convex hull is compact, hence closed, so the limit stays in
it). To ensure the minimum is attained inside the region where
$F$ is continuous, note that $A_t\succ0$ together with continuity of
$t\mapsto\lambda_{\min}(A_t)$ on the compact interval $[0,T]$ gives
$\varepsilon:=\min_{t\in[0,T]}\lambda_{\min}(A_t)>0$, whence
$H[\mu]=\int A_t\,d\mu\succeq\varepsilon I$ for \emph{every} $\mu\in\mathcal{P}([0,T])$.
Every point of $\operatorname{conv}(c([0,T]))$ is such a barycenter, so its $H$-block is
uniformly positive definite and $F$ is continuous on this compact set; it therefore attains
its minimum at some $y^\star\in\operatorname{conv}(c([0,T]))$. Applying
Theorem~\ref{thm:caratheodory} with $S=c([0,T])\subset\mathbb{R}^{p(p+1)}$ writes
$y^\star=\sum_{i=1}^r\lambda_i c(t_i)$ with $r\le p(p+1)+1$; the atomic measure
$\mu^*_{\mathrm{at}}=\sum_i\lambda_i\delta_{t_i}$ satisfies $\Phi(\mu^*_{\mathrm{at}})=y^\star$
and hence $\mathcal{J}(\mu^*_{\mathrm{at}})=F(y^\star)=\min_\mu\mathcal{J}(\mu)$.
\end{proof}

The bound $r\le p(p+1)+1$ reflects the dimension of the ambient moment space $V$, not the
geometry of the specific curve $c$; it guarantees \emph{existence} of an atomic minimizer
but does not imply that every minimizer is atomic.

Proposition~\ref{prop:atomic_optimizer} reduces coupled-regime schedule optimization to a
finite-dimensional problem: one jointly optimizes over atom locations
$t_1,\dots,t_r\in[0,T]$ and weights $w_1,\dots,w_r\ge0$ with $\sum_i w_i=1$, with
$r\le p(p+1)+1$ as a worst-case upper bound.
In practice, the effective number of atoms is far smaller (often single digits) because the
curve $t\mapsto(A_t,B_t)$ is low-rank in realistic network geometries, and the optimization
over $(t_i,w_i)$ is tractable by standard gradient-based methods.

Temporal specialization decouples the noise
levels, block-diagonalizing the curvature so its inverse distributes across times, and
recovers the smooth entropic schedule
$m^*(t)\propto\sqrt{\dot{\mathrm{Ent}}[x_0|x_t]}$ as the decoupled limit of the
coupled theory.

%% file: figures/caratheodory_tikz.tex
\begin{tikzpicture}[>=Stealth, line join=round]
  \draw[->, gray] (0,0) -- (9,0);
  \draw[->, gray] (0,0) -- (0,5.0);

  \begin{scope}[yshift=-1.9cm]
  \node[anchor=west, gray!50!black] at (0.0,6.6) {$V=\mathbb{S}^p\times\mathbb{S}^p$};

  \def\outer{plot[smooth, tension=0.7] coordinates
    {(2.2,3.0) (1.7,4.2) (2.4,5.3) (3.8,5.9) (5.4,5.8) (6.6,4.9) (6.9,3.7) (6.2,2.9)}}
  \def\dentopen{plot[smooth, tension=0.7] coordinates
    {(6.2,2.9) (5.0,3.6) (4.4,4.1) (3.6,3.7) (2.6,3.15)}}

  \fill[blue!12] \outer -- cycle;
  \fill[pattern=north east lines, pattern color=green!50!black] \outer -- cycle;
  \draw[dashed, blue!55] (6.2,2.9) -- (2.2,3.0);

  \draw[very thick, black] \outer;
  \draw[very thick, black] \dentopen;
  \node[black, anchor=south] at (5.2,6.05) {$c(t)$};

  \node[green!50!black, anchor=east] at (-0.1,4.8) {$\Phi(\mathcal{P}([0,T]))$};
  \draw[green!50!black] (-0.1,4.7) .. controls (1.0,4.2) .. (2.3,4.2);
  \node[blue!60!black, anchor=west] at (7.5,2.7) {$=\ \conv\!\big(c([0,T])\big)$};
  \draw[blue!55] (7.5,2.8) .. controls (6.6,3.0) .. (5.7,3.2);

  \coordinate (a1) at (2.4,5.3);  \coordinate (b1) at (6.9,3.7);
  \coordinate (a2) at (3.8,5.9);  \coordinate (b2) at (6.6,4.9);
  \foreach \p in {a1,b1,a2,b2} \fill[black] (\p) circle (1.4pt);
  \node[font=\small, anchor=east]  at (2.3,5.35) {$c(t_1)$};
  \node[font=\small, anchor=west]  at (7.0,3.65) {$c(t_2)$};
  \node[font=\small, anchor=south] at (3.6,6.0)  {$c(t_3)$};
  \node[font=\small, anchor=west]  at (6.7,5.0)  {$c(t_4)$};

  \draw[dashed, black!70] (a1) -- (b1);
  \fill[orange!85!black] (4.65,4.5) circle (2pt);
  \node[orange!75!black, anchor=south] at (4.65,4.58) {$\Phi(\mu_1)$};
  \draw[dashed, black!70] (a2) -- (b2);
  \fill[orange!85!black] (5.2,5.4) circle (2pt);
  \node[orange!75!black, anchor=south west] at (5.26,5.45) {$\Phi(\mu_2)$};
  \end{scope}

  \node[anchor=west] at (0.0,-0.7)
    {$\Phi:\mathcal{P}([0,T])\to V,
      \qquad
      \mathcal{J}(\mu)=F\big({\color{orange!75!black}H,\Gamma}\big)
      =\mathrm{Tr}\!\big(\bar A\,H^{-1}\Gamma\,H^{-1}\big)$};
\end{tikzpicture}

%% file: sections/uncoupled_limit.tex
\section{From atomic to entropic schedules: the role of temporal coupling}
\label{sec:uncoupled}

As shown in Section~\ref{sec:asymptotic_scaling}, the cross-time operator
$P_{\tau,t}$~\eqref{eq:projection_operator_general} controls how gradient noise at
training time~$t$ leaks into prediction error at evaluation time~$\tau$. When $P_{\tau,t}$
is nonzero across a wide range of $(t,\tau)$ pairs, the schedule optimization is
inherently nonlocal and the optimal schedule is atomic
(Section~\ref{sec:optimal_scheduling}). We now analyze the opposite regime: when
$P_{\tau,t}$ localizes to the diagonal $t=\tau$, the objective simplifies dramatically
and the optimal schedule becomes a smooth density.

\subsection{Localization of the cross-time kernel}

Recall from~\eqref{eq:kernel_error_general} that the prediction error at evaluation
time~$\tau$ is an integral over training times~$t$, coupled by the kernel
$P_{\tau,t}$~\eqref{eq:projection_operator_general}.
Suppose the Jacobians at different times use \emph{orthogonal} parameter subspaces:
$\mathrm{Im}(J_t^\top)$ and $\mathrm{Im}(J_\tau^\top)$ are orthogonal for $t\neq\tau$.
This is the condition of \emph{temporal specialization}, the network features used to
denoise at one noise level are disjoint from those used at a distant noise level.
Write $V_s:=\mathrm{Im}(J_s^\top)$ for the parameter subspace active at time~$s$, so the
assumption is $V_t\perp V_\tau$ for $t\neq\tau$. We claim that the cross-time operator
$P_{\tau,t}=J_t\,H[\mu]^{-1}A_\tau\,H[\mu]^{-1}J_t^\top$
(Eq.~\eqref{eq:projection_operator_general}) then vanishes for $t\neq\tau$. The argument
hinges on the inner $H[\mu]^{-1}$: orthogonality of the Jacobian ranges alone is not
enough, because that inverse could in principle rotate the range of $J_t^\top$ out of
$V_t$ before $A_\tau$ acts on it. The point is that it does not.

First, each curvature block is supported on its own subspace. Since
$A_s=\mathbb{E}[J_s^\top J_s]$, any $w\perp V_s$ lies in $\ker J_s$, so $A_s w=0$; hence
$A_s$ has range in $V_s$ and annihilates every orthogonal direction. The same holds for
$B_s=\mathbb{E}[J_s^\top\Sigma_s J_s]$, whose range also lies in $V_s$. Consequently,
because each $A_s$ acts within $V_s$ and the $V_s$ are mutually orthogonal, the integral
$H[\mu]=\int A_s\,d\mu(s)$ has no entries linking distinct subspaces: in an orthonormal
basis adapted to $\bigoplus_s V_s$ it is \emph{block-diagonal}. The inverse of a
block-diagonal matrix is block-diagonal, each block being the inverse of the corresponding
block of $H[\mu]$; hence $H[\mu]^{-1}$ is block-diagonal in the same basis and maps each
subspace $V_s$ into itself.

This block-diagonality is the missing ingredient. Reading $P_{\tau,t}$ right-to-left,
$J_t^\top$ has range $V_t$, and $H[\mu]^{-1}$ preserves it, so $H[\mu]^{-1}J_t^\top$ still
has its range in $V_t$. For $\tau\neq t$ the evaluation curvature $A_\tau$ annihilates
$V_t$ (because $V_t\perp V_\tau$), whence $A_\tau\,H[\mu]^{-1}J_t^\top=0$ and therefore
$P_{\tau,t}=0$. The same subspace bookkeeping gives the trace form:
$\mathrm{Tr}(A_\tau\,H[\mu]^{-1}B_t\,H[\mu]^{-1})
=\mathrm{Tr}\!\bigl((H[\mu]^{-1}A_\tau H[\mu]^{-1})\,B_t\bigr)$ by cyclicity, a product of
an operator supported on $V_\tau$ with one supported on $V_t$, which vanishes when
$V_t\perp V_\tau$. The integral over~$t$ in~\eqref{eq:kernel_error_general} therefore
collapses to the single point $t=\tau$.

\paragraph{Exact orthogonality is an idealization.}
Taken literally, the continuum decomposition $\bigoplus_{s\in[0,T]} V_s$ invoked above
cannot exist. Pairwise-orthogonal nonzero subspaces of~$\mathbb{R}^p$ have dimensions
summing to at most~$p$, so there can be at most~$p$ of them; an uncountable family
$\{V_\tau\}_{\tau\in[0,T]}$ that is pairwise orthogonal must therefore satisfy
$V_\tau=\{0\}$ for all but at most~$p$ times~$\tau$, i.e.\ the model would learn nothing at
almost every noise level. The localization above is thus a limiting idealization. Its
finite-bin form imposes orthogonality on a \emph{finite} partition into $B\le p$ bins,
where the block-diagonal argument is exact; the continuum picture is
recovered as the partition is refined. Refining past~$p$ bins necessarily relaxes exact to
\emph{approximate} orthogonality: features at times separated by more than a specialization
scale are nearly orthogonal, so $P_{\tau,t}$ concentrates near the diagonal rather than
vanishing exactly off it. The residual off-diagonal coupling is small when specialization
is strong.

\subsection{The independent-learner limit via binned orthogonality}
\label{subsec:binned_orth}

The localization argument above is heuristic, it invokes a continuum of orthogonal
subspaces. Its finite-bin counterpart discretizes the time axis and imposes orthogonality on
finitely many bins. We partition $[0,T]$ into disjoint bins $\{I_b\}_{b=1}^B$ with widths
$\Delta_b:=|I_b|$ and representative times $t_b\in I_b$; the sampling measure induces bin
masses $\mu_b:=\mu(I_b)$.

\paragraph{Assumption B (Binned orthogonality of time features).}
There exist \emph{orthogonal} projectors
$\{\Pi_b\}_{b=1}^B\subset\mathbb{R}^{p\times p}$ such that for all $t\in I_b$ and all
$x_t$,
\begin{equation}
  J_t(x_t) \;=\; J_t(x_t)\,\Pi_b,
  \qquad
  \Pi_b\Pi_{b'}=0 \ \ \text{for } b\neq b',
  \label{eq:sec7_binned_orth}
\end{equation}
so the parameter space splits into the orthogonal direct sum
$\mathbb{R}^p=\bigoplus_{b=1}^B \mathrm{Im}(\Pi_b)$. This is the legitimate finite-rank
realization of the heuristic decomposition $\bigoplus_s V_s$, necessarily with $B\le p$
blocks. Perturbations outside $\mathrm{Im}(\Pi_b)$ leave the predictor at times in~$I_b$
unchanged to first order, so the global curvature $H[\mu]$ is block-diagonal across bins
and its inverse is the corresponding block-diagonal inverse: the within-bin problems
decouple exactly.

Write the bin-conditioned curvature and noise operators
\[
  H_b:=\mathbb{E}[G_t(x_t)\mid t\in I_b],
  \qquad
  \Gamma_b:=\mathbb{E}[J_t(x_t)^\top\Sigma_t(x_t)J_t(x_t)\mid t\in I_b].
\]
Under Assumption~B the linearized SGD recursion splits into one independent blockwise
recursion per active bin with no cross-talk, so the Polyak--Ruppert CLT of
Section~\ref{sec:asymptotic_scaling}~\eqref{eq:parameter_clt_general} applies within each
block under $H[\mu]\to H_b$, $\Gamma[\mu]\to\Gamma_b$. Each of the $K$ updates draws its
time independently from~$\mu$, so the expected number allocated to bin~$b$ is
\[
N_b(K)
:=
\mathbb{E}\Bigl[\textstyle\sum_{n\le K}\mathbf{1}\{t_n\in I_b\}\Bigr]
=K\mu_b,
\]
about which the empirical count concentrates by the law of large numbers; the averaged
block-$b$ iterate thus has effective sample size $N_b(K)=K\mu_b$.

To see how the bin masses enter the error constant, evaluate the kernel form of the
per-time error~\eqref{eq:kernel_error_general} at the representative time $\tau=t_b$,
\[
\lim_{K\to\infty} K\,\mathscr{L}_{t_b}(K;\mu)
=
\int \mathbb{E}_{x_t\mid t}\bigl[\mathrm{Tr}(P_{t_b,t}\,\Sigma_t)\bigr]\,d\mu(t),
\qquad
P_{t_b,t}=J_t\,H[\mu]^{-1}A_{t_b}H[\mu]^{-1}J_t^\top .
\]
The evaluation curvature $A_{t_b}$ is supported on $\mathrm{Im}(\Pi_b)$, while $J_t$ (for
$t\in I_c$) is supported on $\mathrm{Im}(\Pi_c)$; since $H[\mu]^{-1}$ is block-diagonal,
$P_{t_b,t}=0$ unless $t\in I_b$, and the integral collapses to bin~$b$. The block
restrictions are then exact by definition of the bin averages:
$H[\mu]|_b=\int_{I_b}A_t\,d\mu=\mu_b H_b$ and
$\int_{I_b}\mathbb{E}[J_t^\top\Sigma_t J_t]\,d\mu=\mu_b\Gamma_b$, so
$H[\mu]^{-1}|_b=\mu_b^{-1}H_b^{-1}$. The \emph{only} approximation is that the curvature is
roughly constant across a bin, $A_{t_b}\approx H_b$ (exact in the fine-bin limit); with it,
cyclic invariance of the trace gives
\[
\lim_{K\to\infty} K\,\mathscr{L}_{t_b}(K;\mu)
=
\mu_b^{-1}\,\mathrm{Tr}\!\bigl(A_{t_b}H_b^{-1}\Gamma_b H_b^{-1}\bigr)
\;\approx\;
\mu_b^{-1}\,\mathrm{Tr}\!\left(H_b^{-1}\Gamma_b\right).
\]
Dividing by $K$, the bin-$b$ prediction error is
\begin{equation}
  \mathscr{L}_b(K;\mu)
  \;=\;
  \frac{1}{K\,\mu_b}\,
  \mathrm{Tr}\!\left(H_b^{-1}\Gamma_b\right)
  \;+\; o(K^{-1}),
  \label{eq:block_scaling_main}
\end{equation}
the bin-level analogue of the coupled error
formula~\eqref{eq:general_prediction_error}: each bin's prediction error scales as
$1/(K\mu_b)$ times a bin-specific difficulty constant.

\begin{proposition}[Bin-level optimal schedule in the independent-learner regime]%
  \label{prop:continuum_limit}
Under Assumptions~A and~B, define the \emph{bin difficulty}
$
  \alpha_b
  \;:=\;
  \dot{\mathrm{SNR}}(t_b)\,\mathrm{Tr}\!\left(H_b^{-1}\Gamma_b\right)\,\Delta_b,
  \quad b=1,\dots,B.
$
Assume $\alpha_b>0$ for all active bins. Then the ELBO-weighted bin objective
$
  \sum_{b=1}^B {\alpha_b}/{\mu_b}
$
subject to
$
  \mu_b>0,\ \sum_{b=1}^B\mu_b=1
$
has unique minimizer
$
  \mu_b^*
  =
  {\sqrt{\alpha_b}}\big/{\sum_{j=1}^B \sqrt{\alpha_j}}.
$
\end{proposition}
\begin{proof}
The objective $\sum_{b}\alpha_b/\mu_b$ is separable and strictly convex on the simplex
interior, since $\partial^2(\alpha_b/\mu_b)/\partial\mu_b^2=2\alpha_b/\mu_b^3>0$ for
$\mu_b>0$. Introducing a multiplier $\lambda$ for the constraint $\sum_b\mu_b=1$, the
stationarity condition $-\alpha_b/\mu_b^2+\lambda=0$ gives $\mu_b=\sqrt{\alpha_b/\lambda}$,
and normalization fixes $\lambda=\bigl(\sum_{j}\sqrt{\alpha_j}\bigr)^2$, hence
$\mu_b^*=\sqrt{\alpha_b}\big/\sum_{j}\sqrt{\alpha_j}$; strict convexity makes this the unique
minimizer. For uniform bins of width $\Delta=T/B$ the corresponding density is
$m^*(t_b)=\mu_b^*/\Delta\propto\sqrt{\dot{\mathrm{SNR}}(t_b)\,\mathrm{Tr}(H_b^{-1}\Gamma_b)}$,
which converges as $B\to\infty$ to
$m^*(\tau)\propto\sqrt{\dot{\mathrm{SNR}}(\tau)\,\mathrm{Tr}(A_\tau^{-1}B_\tau)}$.
\end{proof}

\paragraph{From the bin difficulty to the entropy rate.}
The optimal masses depend on the data only through the bin difficulties
$\alpha_b=\dot{\mathrm{SNR}}(t_b)\,\mathrm{Tr}(H_b^{-1}\Gamma_b)\,\Delta_b$, which are still
abstract operator traces. Two reductions, used in the rest of the section, make them
concrete. First, \emph{feature--noise decoupling}
(Appendix~\ref{sec:feature_noise_decoupling}): the per-bin error sees the Jacobian only
through the projector $P_t=J_t(\mathbb{E}[G_t])^{-1}J_t^\top$, which in the random-row model
of Appendix~\ref{supp sec: random matrix} has isotropic expectation $\mathbb{E}[P_t]=r\,I$
with $r=p/d$; when $P_t$ is weakly correlated with the residual covariance,
\begin{equation}\label{eq:fnd_approx_main}
\mathrm{Tr}(H_b^{-1}\Gamma_b)\;\approx\; r\cdot\mathrm{MMSE}(t_b),
\end{equation}
and Marchenko--Pastur concentration (Appendix~\ref{supp sec: random matrix}) bounds the
remainder uniformly in time. Second, the \emph{conditional de~Bruijn identity}
(Appendix~\ref{supp sec: conditional entropy}),
\begin{equation}\label{eq:debruijn_main}
\dot{\mathrm{Ent}}[x_0\mid x_t]
\;=\;
\dot{\mathrm{SNR}}(t)\,\mathrm{MMSE}(t),
\end{equation}
identifies the SNR-weighted MMSE with the generative entropy rate, equivalently, with the
rate $-\tfrac{d}{dt}I(X_0;X_t)$ at which mutual information between the clean sample and its
noisy version decays along the forward process. Combining the two,
$\alpha_b\approx r\,\dot{\mathrm{Ent}}[x_0\mid x_{t_b}]\,\Delta_b$: up to the constant~$r$,
the bin difficulty \emph{is} the generative entropy rate.

\subsection{The entropic schedule}

By Proposition~\ref{prop:continuum_limit} the optimal sampling density is
$m^*(\tau)\propto\sqrt{\alpha(\tau)}$, with local difficulty
$\alpha(\tau)=\dot{\mathrm{SNR}}(\tau)\,\mathrm{Tr}(A_\tau^{-1}B_\tau)$. By the
feature--noise decoupling and de~Bruijn reductions just established
(\eqref{eq:fnd_approx_main}--\eqref{eq:debruijn_main}),
$\alpha(\tau)\approx r\,\dot{\mathrm{Ent}}[x_0\mid x_\tau]$, so the square-root law becomes
the \emph{entropic schedule}
\begin{equation}\label{eq:entropic_schedule}
m^*(t)\;\propto\;\sqrt{\dot{\mathrm{Ent}}[x_0\mid x_t]}.
\end{equation}
This is the central practical prescription of the paper: allocate training effort in
proportion to the square root of the generative entropy rate. This follows from the famous I-MMSE formula first derived in \citep{guo2005mutual} and discussed in \citep{kong2023information} and \citep{stancevic2026entropic, stanvcevic2026information} in the context of generative diffusion. 

\subsection{Unifying view: the Carath\'eodory dimension}

The transition from atomic to smooth schedules can now be understood through a single
geometric lens. In the coupled regime, the objective depends on~$\mu$ through the linear
moment map $\mu\mapsto(H[\mu],\Gamma[\mu])\in\mathbb{S}^p\times\mathbb{S}^p$, whose
image is a convex body of dimension at most $p(p+1)$. Carath\'eodory's theorem then
guarantees that the optimizer is supported on at most $p(p+1)+1$ points, an atomic
schedule. The smooth limit is precisely where this finite-dimensional picture breaks down.
A parameter space of dimension~$p$ admits at most~$p$ mutually orthogonal feature blocks
(Section~\ref{subsec:binned_orth}); realizing a genuine \emph{continuum} of orthogonal
subspaces, one degree of temporal specialization per noise level, would require the
parameter count to grow without bound. In that $p\to\infty$ limit the Carath\'eodory bound
$p(p+1)+1$ diverges, the effective moment space becomes infinite-dimensional with one
scalar degree of freedom (the local density $m(\tau)$) per noise level, and the atomic
optimizer gives way to the smooth entropic density.

The two cases are thus endpoints of a coupling-bandwidth spectrum, indexed by how many
orthogonal feature blocks the network supports.

%% file: sections/experiments.tex
\section{Experiments}\label{sec:experiments}

\begin{figure*}[t]
    \begin{minipage}[t]{0.48\textwidth}
        \centering
        \includegraphics[width=\linewidth]{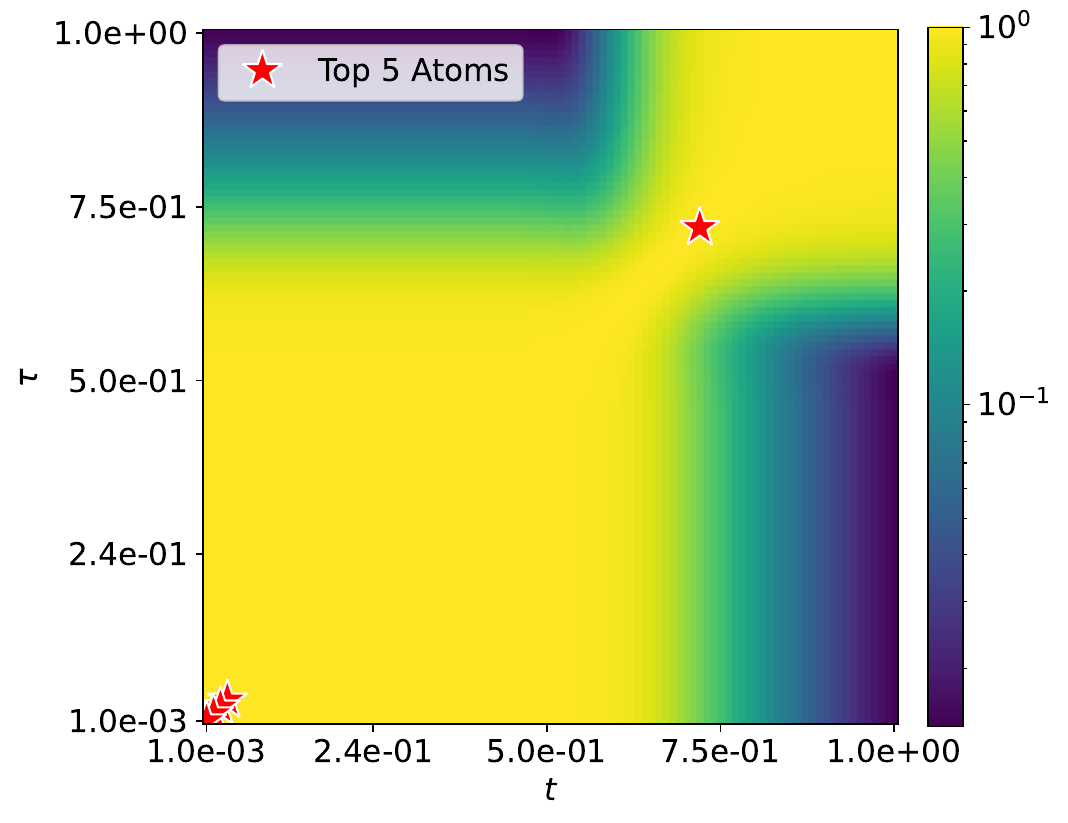}
        \subcaption{Parametric Model - Dirac Mixture of 2 Components}
    \end{minipage}\hfill
    \begin{minipage}[t]{0.48\textwidth}
        \centering
        \includegraphics[width=\linewidth]{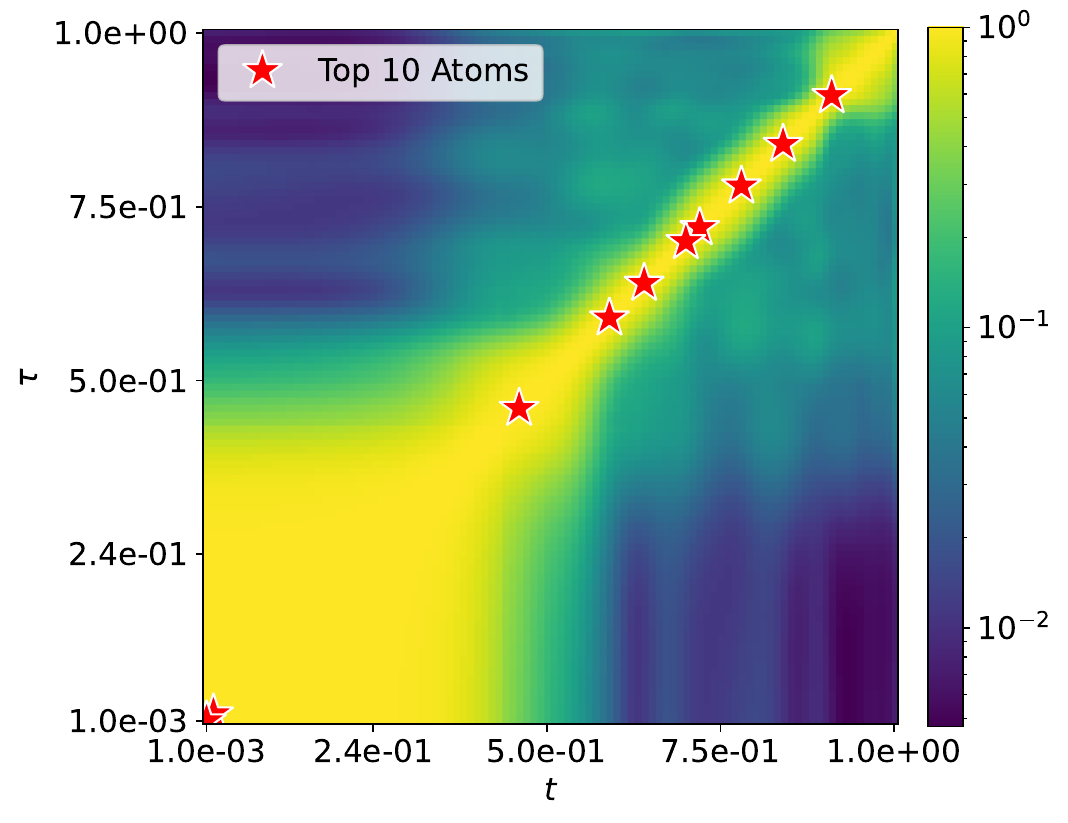}
        \subcaption{Neural Network - Dirac Mixture of 2 Components}
        \label{fig:coupling_main_delta_2_mlp}
    \end{minipage}

    \begin{minipage}[t]{0.48\textwidth}
        \centering
        \includegraphics[width=\linewidth]{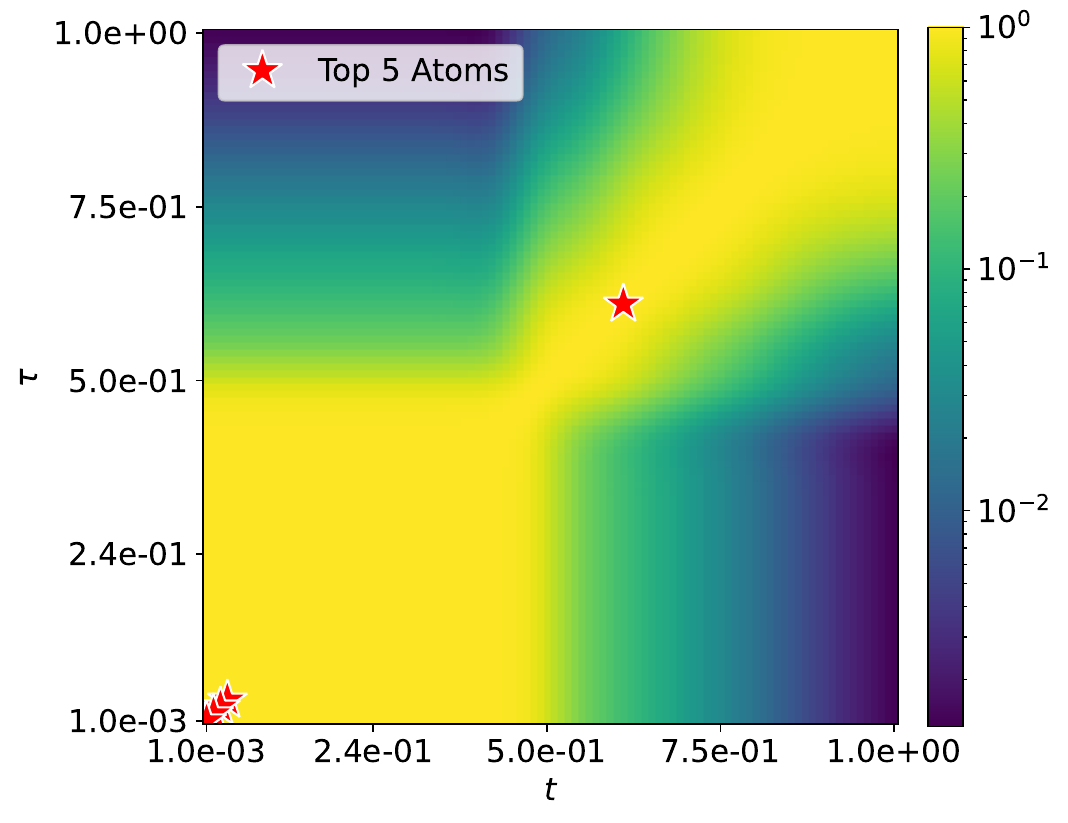}
        \subcaption{Neural Network - Dirac Mixture of 4 Components}
    \end{minipage}\hfill
    \begin{minipage}[t]{0.48\textwidth}
        \centering
        \includegraphics[width=\linewidth]{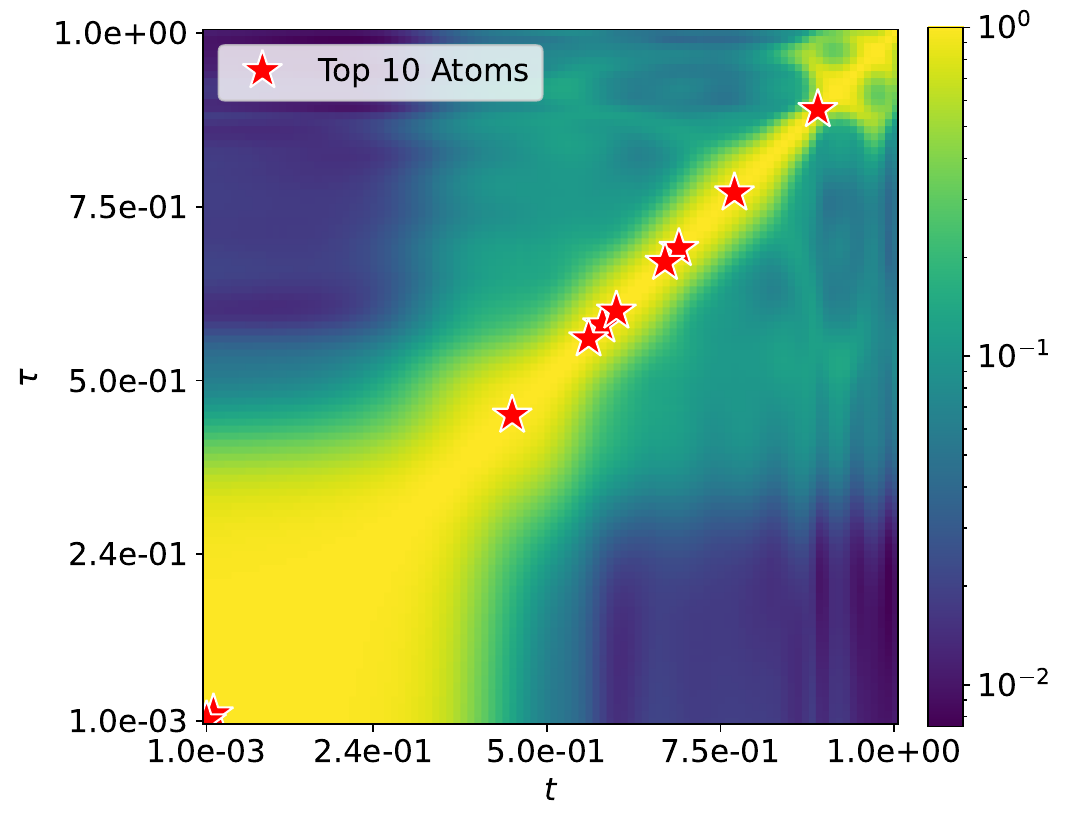}
        \subcaption{Neural Network - Dirac Mixture of 4 Components}
        \label{fig:coupling_main_delta_4_mlp}
    \end{minipage}
    \caption{\textbf{Cross-time feature coupling and optimized schedule atoms.} The heatmaps display the strength of feature coupling between different times using the expected trace norm of \cref{eq:projection_operator_general} with a uniform time schedule. Heat maps have been normalized to have ones on the diagonal. Overlaid on the diagonal of each heatmap are the corresponding time locations of the most active atoms of the optimized schedule. The setup compares two different architectures: a fully parametric model and a neural network, trained on 1D Dirac mixture distributions consisting of either two or four components.}
    \label{fig:feature_coupling_and_sched}
\end{figure*}

We organise the experiments in two stages. We first give controlled low-dimensional experiments where the atomic schedule can be computed exactly and ground-truth scores are available. These experiments are meant to provide validation for the theory, not to provide practical training solutions. Then, we then turn to real-data benchmarks across discrete and continuous image domains.

\textbf{Low-dimensional setup.} We validate the two structural predictions of the theory, %
(i) in the coupled regime the optimizer is atomic, and (ii) in the decoupled regime the optimum
is the smooth entropic density $m^*(t)\propto\sqrt{\dot{\mathrm{Ent}}[x_0|x_t]}$, on a range of
controlled settings of increasing complexity: one-dimensional Dirac mixtures with $2$, $3$, and $4$
components (both a fully parametric model with closed-form score and a $50$k-parameter
\gls{MLP}), the $2$D Swiss Roll and Moons manifolds (a $50$k-parameter \gls{MLP}), and MNIST (a
$14$M-parameter 2D U-Net); all use a variance-exploding diffusion. Following the frozen-feature
regime of the coupled theory, for the neural models we optimize the schedule of the last layer
only, $32$ parameters for the \gls{MLP}, the $288$-weight final convolutional kernel for the
U-Net, treating earlier layers as fixed feature extractors. For each setting we train to
convergence, perturb the target parameters with Gaussian noise to sit near the optimum (as the
local theory requires), and retrain under each schedule over $30$ seeds ($20$ for MNIST),
reporting the negative log-likelihood in bits per dimension (BPD) with its \gls{SEM}. Beyond the
uniform baseline we also compare against the CosMap and Logit-Normal heuristics of
\citet{esser2024scaling}. The full protocol, including the closed-form Dirac-mixture score, is
given in \cref{appx:experiments}.

\begin{figure}[tb]
  \centering
  \includegraphics[width=0.42\textwidth]{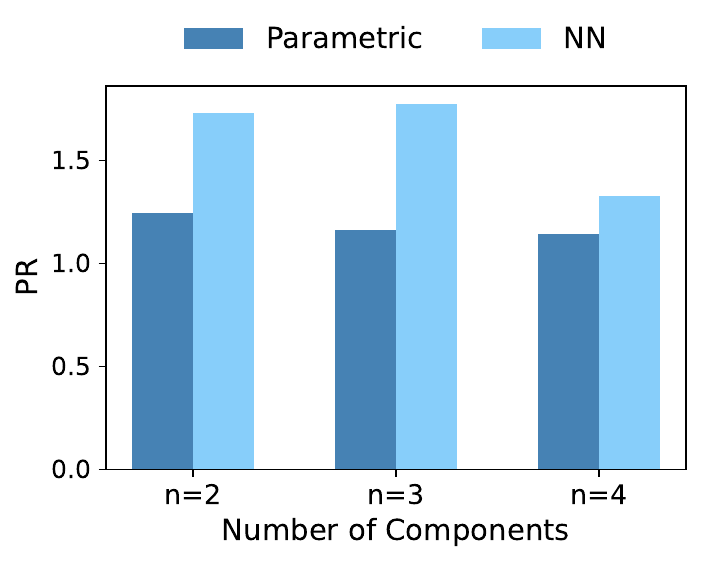}
  \caption{Participation ratio \(\sum_i\nicefrac{1}{m_i^2}\) comparing Parametric vs. neural networks for Dirac Mixtures of \(n\) Components.}
  \label{fig:dirac_pr_bar}
\end{figure}
\textbf{Optimized schedules.} \Cref{fig:toy_schedules} overlays, for the four headline settings, the optimized coupled schedule with the entropic proxy. By inspecting the optimized noise scheduling functions, we observe that, in most datasets, a small set of atoms tend to receive exponentially larger weights, while the remaining probability mass organizes in quasi-continuous densities, roughly following the entropy profiles. \Cref{tab:sparsity_schedule} quantifies this phenomenon in terms of participation ratio of the schedule \( \mathrm{PR} = \nicefrac{1}{\sum_i m_i^2}\) and the coverage of the Top-K atoms, showing that across all our experiments we consistently obtain very few active atoms. In most settings, we see a very large isolated atom at very low noise levels, where Fisher information is very high. \Cref{fig:feature_coupling_and_sched} contextualizes the placement of the remaining atoms through the lens of cross-time feature coupling, $P_{\tau,t}$. While the parametric model is structurally fully coupled, we observe a remarkable drop in functional coupling between the high and low-noise regimes. This block structure reflects a fundamental rotation in the active subspace of the Jacobian: at high noise, parameter gradients push to globally shift the center of mass, whereas at low noise, they isolate specific mixture components. Consequently, parameter perturbations at high $t$ project weakly onto the denoiser objective at low $\tau$. Because of this, the operator curve $c(t)=(A_t,B_t)$, which remains continuous due to residual cross-talk in $P_{\tau, t}$, must undergo a sharp geometric bend in moment space. By the geometric arguments of \cref{fig:caratheodory}, it follows that the optimized schedule must explicitly deploy at least two atoms to anchor these diverging noise regimes. This strategic allocation extends to the neural network experiments; \cref{fig:dirac_pr_bar} shows that, on the same Mixture of Dirac Delta distribution, the neural network atomic schedule spreads the atoms mass more heterogeneously compared to the parametric model due to weaker temporal coupling in high noise regimes. In neural networks, features become complex transformations of the input and develop \textit{temporal specialization}, with several nearly orthogonal feature blocks (\cref{fig:coupling_main_delta_2_mlp} and \cref{fig:coupling_main_delta_4_mlp}). Furthermore, as these orthogonal feature blocks become narrower in time, the piecewise constant-curvature approximation of \cref{subsec:binned_orth} becomes increasingly exact. This continuum limit explains why the MNIST atomic schedule (\cref{fig:toy_schedules}) follows the smooth, square-root generative entropy rate so closely. The reader is referred to \cref{appx:experiments} for supplementary cross-time feature coupling visualizations.

\input{tables_toy_case/pr_table}

\textbf{ELBO objective.} \Cref{fig:toy_sched_objs} plots the objective integrand
$\dot{\mathrm{SNR}}(\tau)\cdot\lim_{K\to\infty}K\,\mathscr{L}_\tau$ against diffusion time $\tau$
for the uniform, atomic, and entropic schedules. Across the neural-network settings both the
atomic and the entropic schedules reduce the peak of the integrand relative to the uniform
baseline, with the atomic schedule attaining the smallest values at its support points, so the
gains concentrate where the theory predicts them. In the parametric setting the entropic
integrand instead exceeds the uniform one, exposing the proxy's failure mode in a fully coupled problem.  Overall, \Cref{fig:toy_lcs} proves that \textit{a lower asymptotic objective  directly translates in better performances in terms of NLL across all settings}.

\textbf{Generative performance.} \Cref{fig:toy_lcs} reports the resulting NLL learning curves (in BPD), comparing the atomic and entropic schedules against the uniform, CosMap, and Logit-Normal baselines. In the neural-network experiments the entropic and atomic schedules are the top performers, and the entropic schedule, far cheaper to compute, essentially matches the atomic one. In the parametric case the atomic schedule remains the best while the entropic is the worst, consistent with the schedule and objective plots above. Crucially, across all our experiments, the atomic schedule shows faster convergence than all the baselines, with the entropic schedule following in neural network settings. This result shows practical significance: while the assumptions under the atomic schedule and its estimation are difficult to realize in practice, the entropic schedule opens the door to practical and compute-efficient diffusion model training.

\begin{figure*}[t]
    \begin{minipage}[t]{0.48\textwidth}
        \centering
        \includegraphics[width=\linewidth]{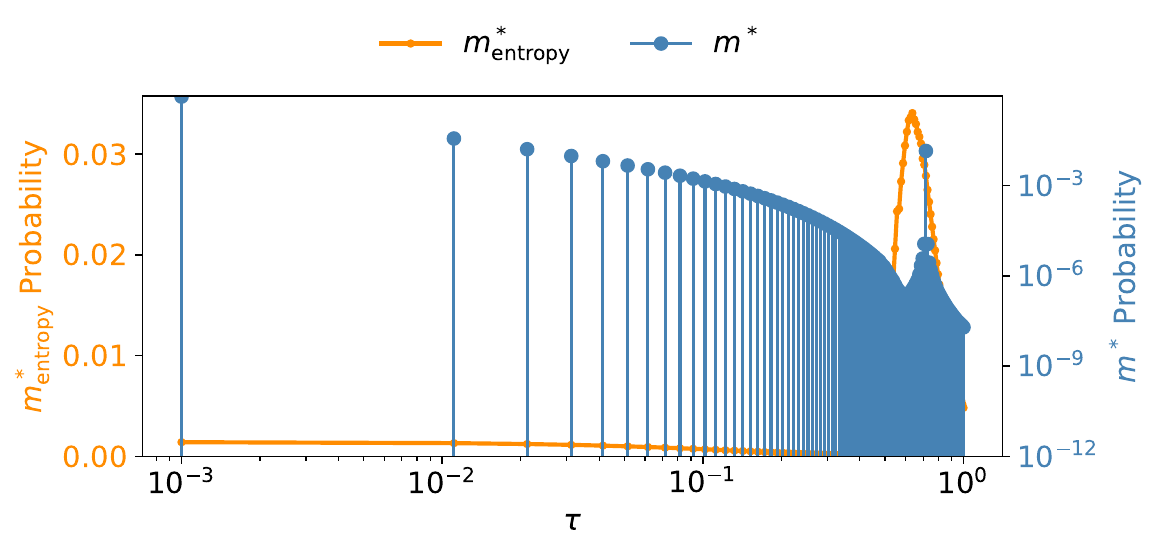}
        \subcaption{Parametric Model - Dirac Mixture}
    \end{minipage}\hfill
    \begin{minipage}[t]{0.48\textwidth}
        \centering
        \includegraphics[width=\linewidth]{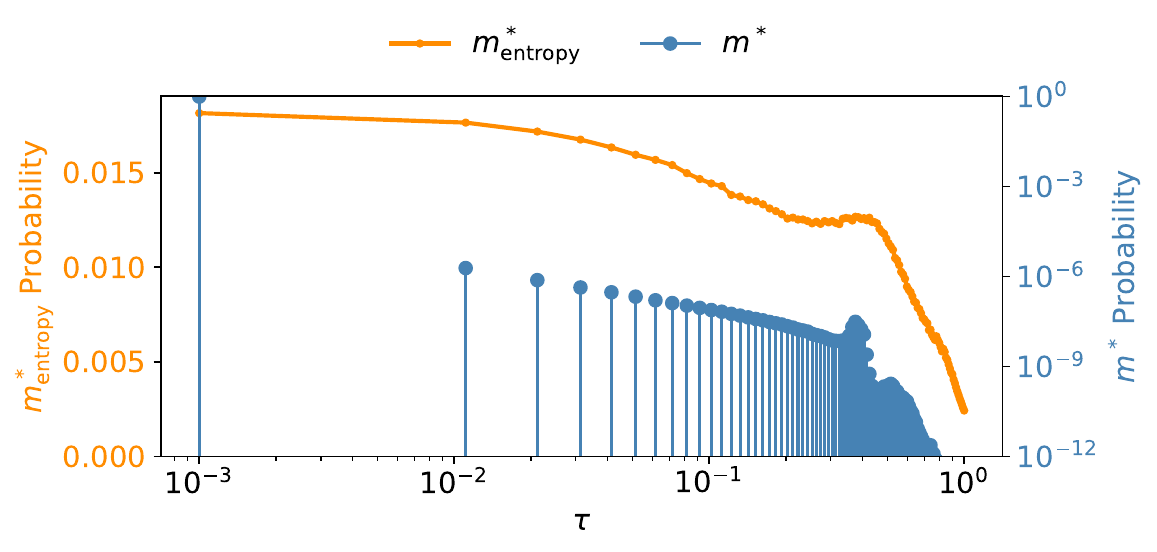}
        \subcaption{Neural Network - MNIST}
        \label{fig:mnist_schedule}
    \end{minipage}

    \begin{minipage}[t]{0.48\textwidth}
        \centering
        \includegraphics[width=\linewidth]{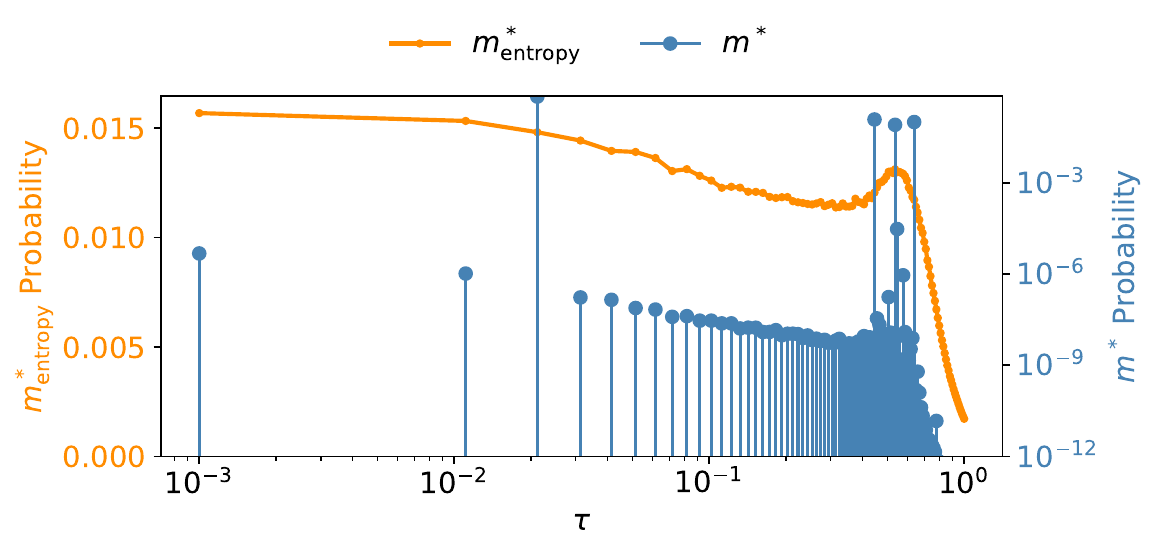}
        \subcaption{Neural Network - Moons}
    \end{minipage}\hfill
    \begin{minipage}[t]{0.48\textwidth}
        \centering
        \includegraphics[width=\linewidth]{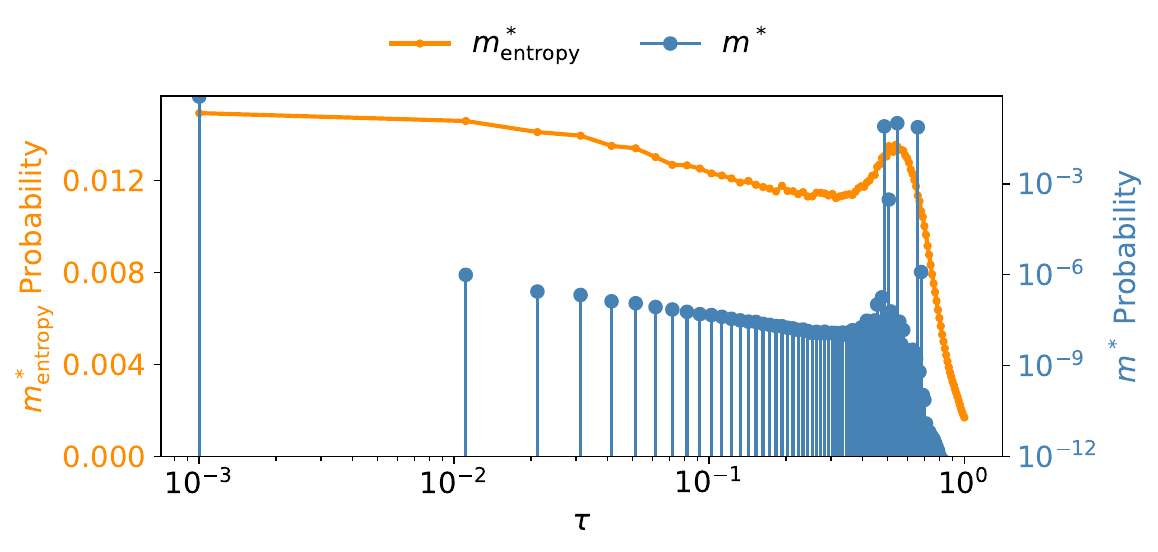}
        \subcaption{Neural Network - Swissroll}
    \end{minipage}
    \caption{\textbf{Atomic schedules versus the entropic proxy.} In all four panels the coupled optimizer appears atomic; the entropic schedule is smooth and tracks the same informative regions without matching the atomic support exactly.}
    \label{fig:toy_schedules}
\end{figure*}

\begin{figure*}[t]
    \begin{minipage}[t]{0.48\textwidth}
        \centering
        \includegraphics[width=\linewidth]{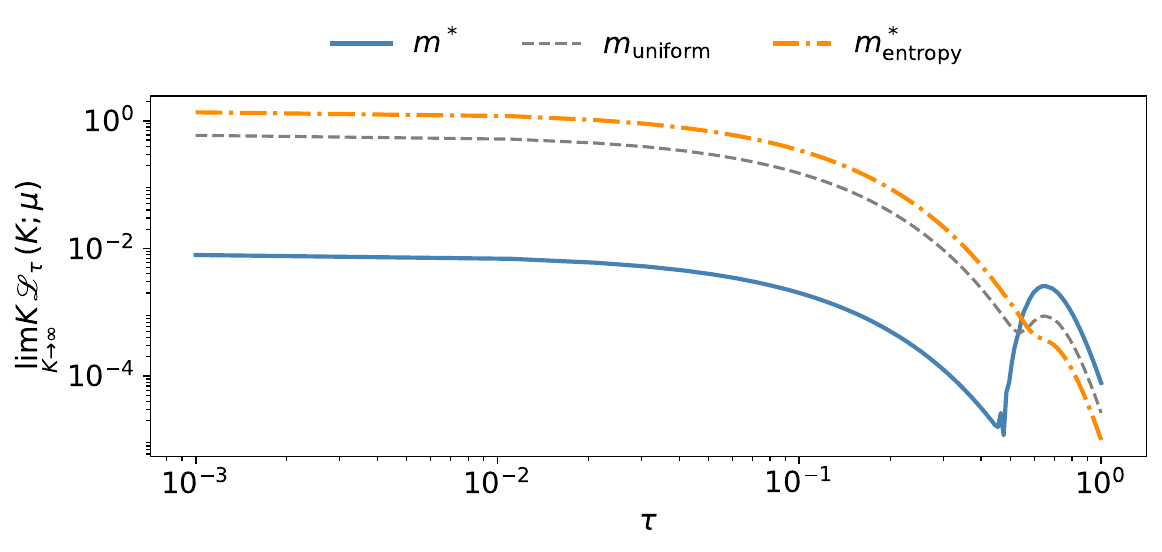}
        \subcaption{Parametric Model - Dirac Mixture}
    \end{minipage}\hfill
    \begin{minipage}[t]{0.48\textwidth}
        \centering
        \includegraphics[width=\linewidth]{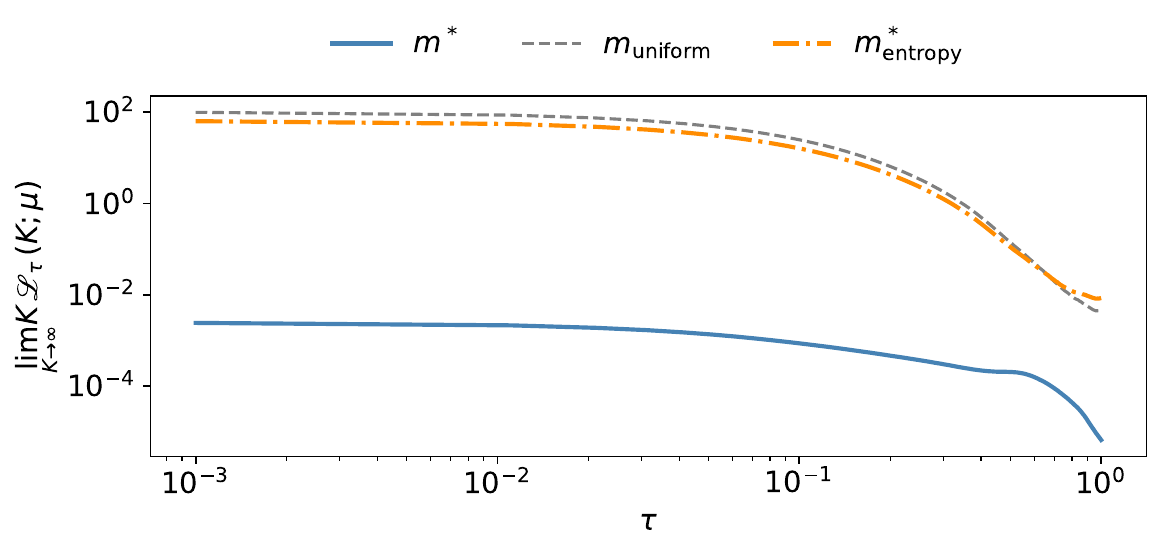}
        \subcaption{Neural Network - MNIST}
    \end{minipage}

    \begin{minipage}[t]{0.48\textwidth}
        \centering
        \includegraphics[width=\linewidth]{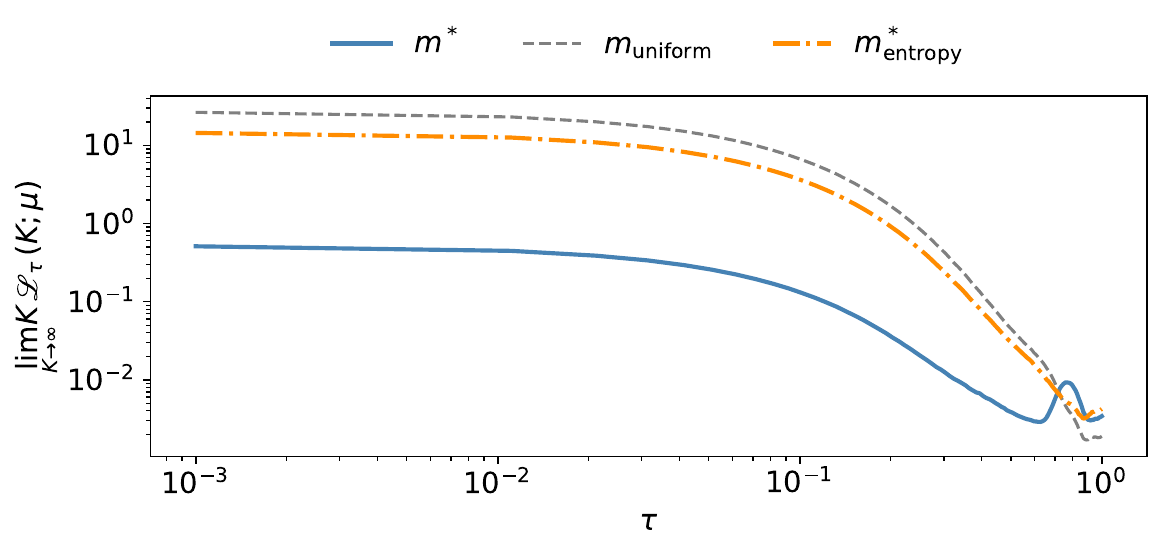}
        \subcaption{Neural Network - Moons}
    \end{minipage}\hfill
    \begin{minipage}[t]{0.48\textwidth}
        \centering
        \includegraphics[width=\linewidth]{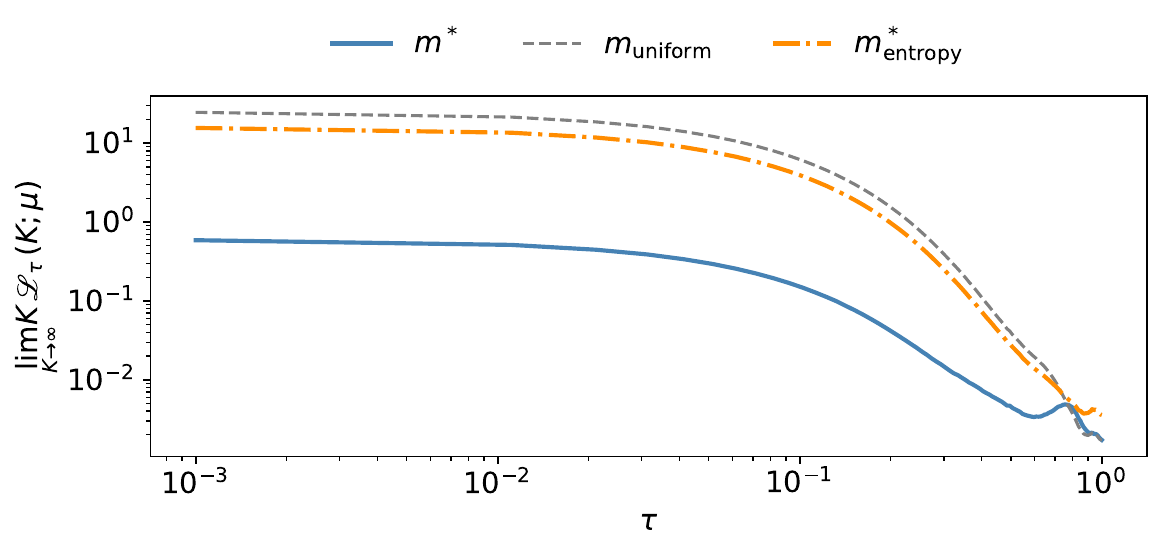}
        \subcaption{Neural Network - Swissroll}
    \end{minipage}
    \caption{\textbf{ELBO objective integrand under different schedules.} Each panel shows $\dot{\mathrm{SNR}}(\tau)\cdot\lim_{K\to\infty}K\,\mathscr{L}_\tau$ vs.\ diffusion time $\tau$ for the uniform schedule, the atomic, and the smooth entropic schedule.}
    \label{fig:toy_sched_objs}
\end{figure*}

\begin{figure*}[t]
    \begin{minipage}[t]{0.48\textwidth}
        \centering
        \includegraphics[width=\linewidth]{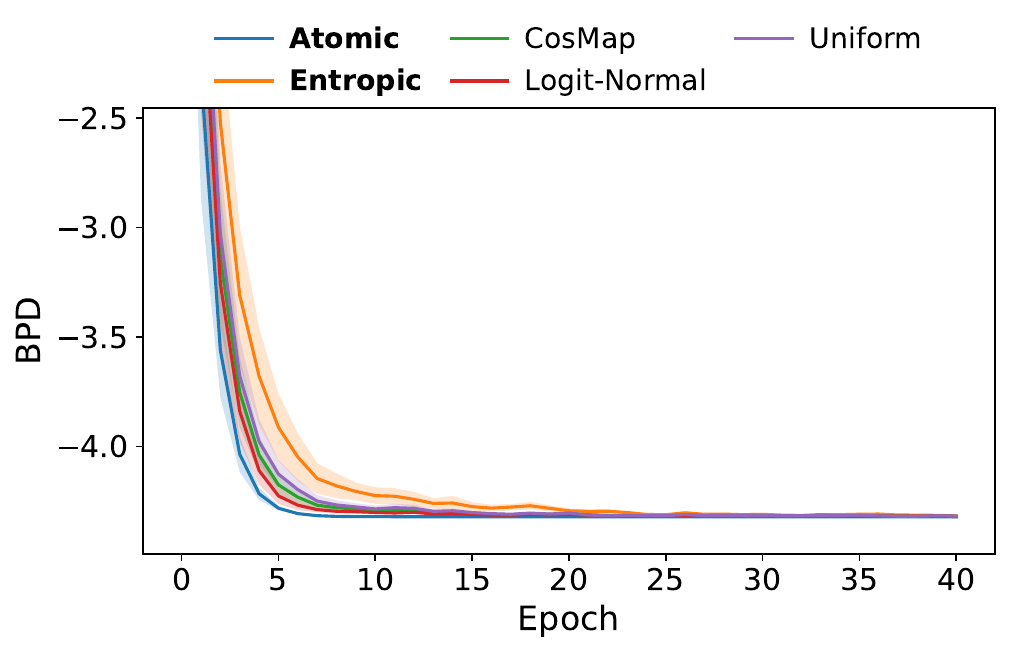}
        \subcaption{Parametric Model - Dirac Mixture}
    \end{minipage}\hfill
    \begin{minipage}[t]{0.48\textwidth}
        \centering
        \includegraphics[width=\linewidth]{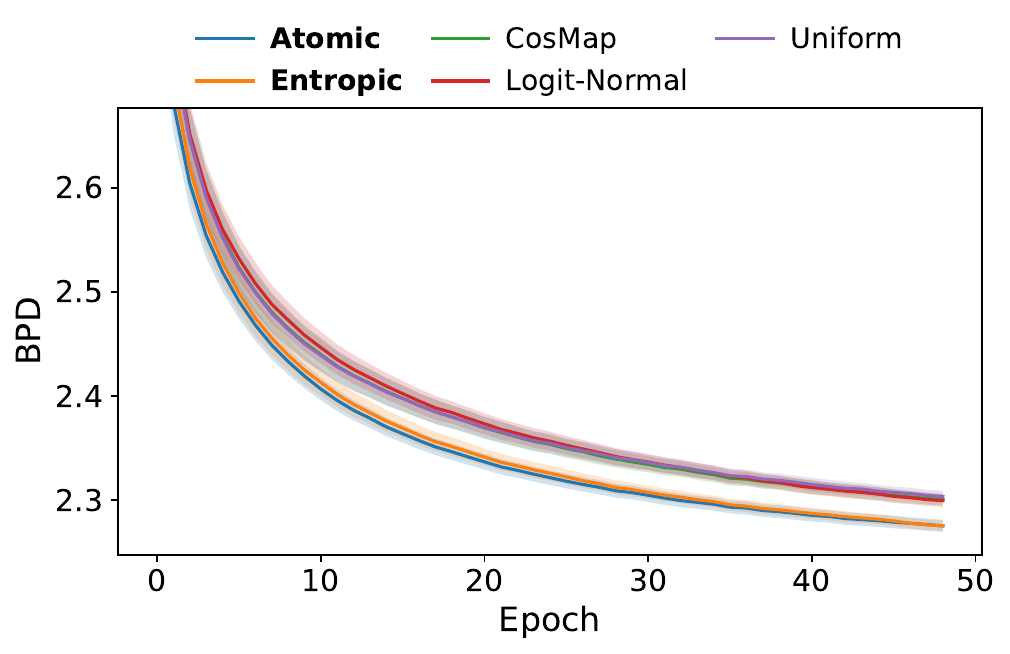}
        \subcaption{Neural Network - MNIST}
    \end{minipage}

    \begin{minipage}[t]{0.48\textwidth}
        \centering
        \includegraphics[width=\linewidth]{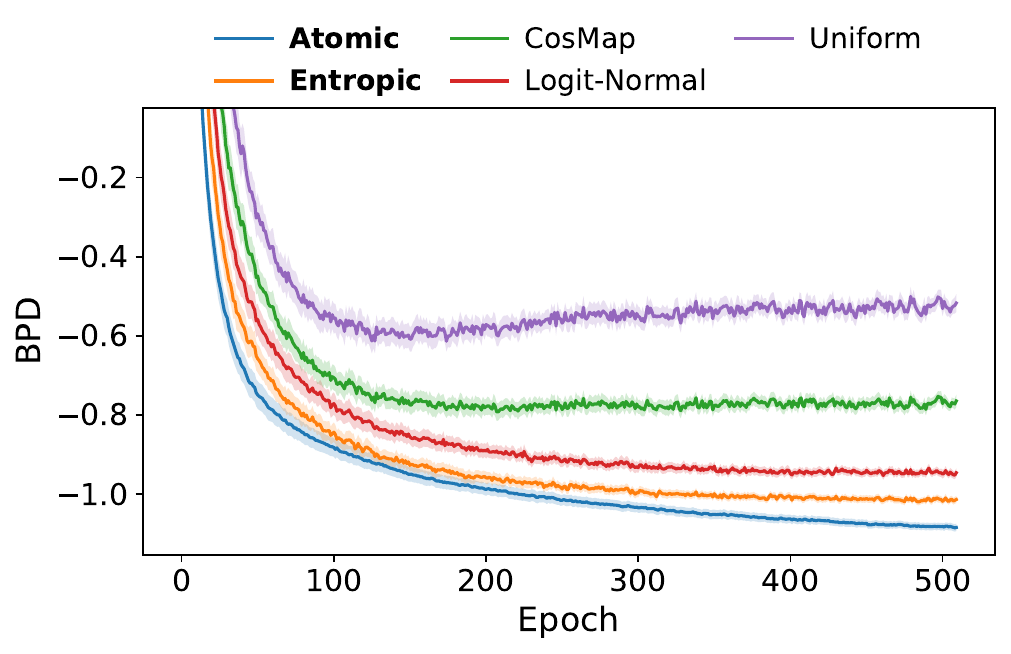}
        \subcaption{Neural Network - Moons}
    \end{minipage}\hfill
    \begin{minipage}[t]{0.48\textwidth}
        \centering
        \includegraphics[width=\linewidth]{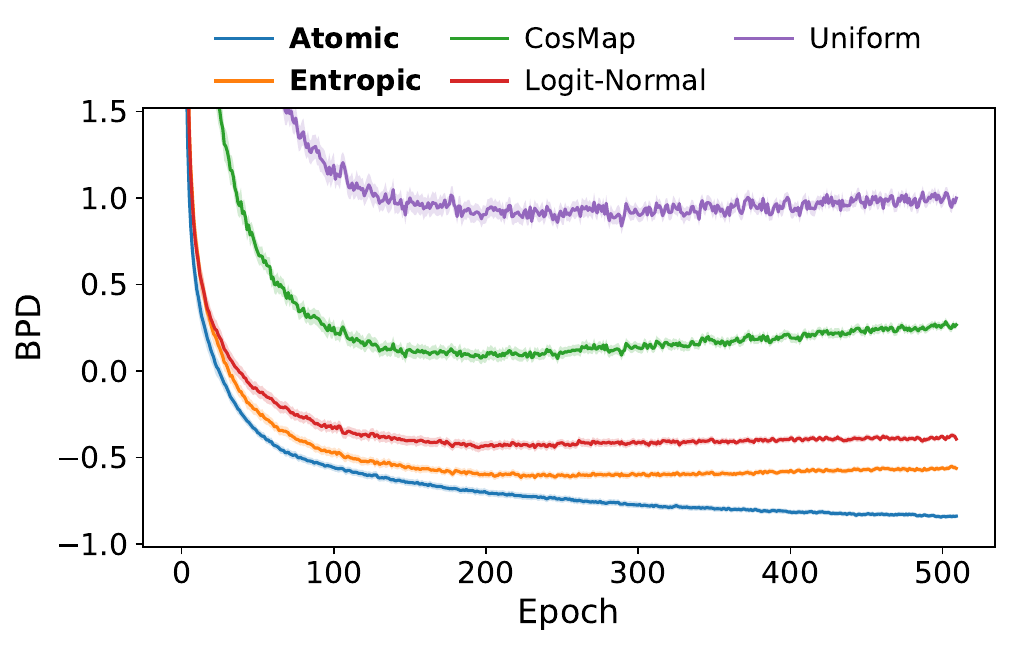}
        \subcaption{Neural Network - Swissroll}
    \end{minipage}
    \caption{\textbf{NLL learning curves (BPD).} Negative log-likelihood in bits per dimension as training progresses, comparing the atomic and entropic schedules against the uniform, CosMap, and Logit-Normal baselines. Curves are averaged over 30 seeds (20 for MNIST); shaded areas show the \gls{SEM}.}
    \label{fig:toy_lcs}
\end{figure*}

\subsection{Experiments on real data}
The full optimization procedure used in the previous experiments is currently not feasible for large neural networks due to the size of the Jacobi matrix and the difficulty in stabilizing its inverse. While these issues can be ultimately fixed by borrowing tools from second-order optimization, here we instead suggest the use of the entropy-based solver for larger scale experiments.

\textbf{Setup.} Across all real-data benchmarks we hold the architecture, training objective,
optimizer, and sampler fixed at the standard per-dataset configurations and vary \emph{only}
the time-sampling schedule, so that the reported differences are attributable to the schedule
alone. Discrete-domain models use the unweighted objective $w\equiv 1$ (with the EDM recipe
used for transfer), while continuous-image models use the EDM objective. The entropic schedule
is precomputed from a pretrained model through the held-out denoising-error estimator of
\cref{appx:experimental_details}~\eqref{eq:entropy_rate_estimator}, and is compared
against a log-uniform baseline (\textsc{LogUnif}) and the EDM/EDM2-style heuristic
(\textsc{EDM}). These entropy profiles can be easily estimated online using running averages of the loss, as shown in ref~\citep{raya2026information}. However, we use the pre-computed approach here for convenience. 

We report final sample quality (FID, or Sei-FID for DNA) and, on discrete
domains, the number of processed examples (in thousands, ``kEx'') needed to reach a fixed
per-dataset quality target; the speedup ``Sp.\ $\times$'' is the corresponding ratio of
processed examples to matched quality relative to the stronger baseline. Dashes denote
configurations or thresholds we did not measure.

\Cref{tab:efficiency_across_domains} compares the entropic square-root schedule against a log-uniform baseline and an EDM/EDM2-style heuristic across discrete and continuous benchmarks. On discrete domains the entropic schedule improves both the final score and the compute-to-target metric, with measured speedups between $2.0\times$ and $2.7\times$. On continuous-image models the comparison is closer: the entropic schedule slightly improves on EDM for MNIST, FashionMNIST, and both CIFAR-10 settings, while FFHQ slightly favors EDM. The gains are strongest on discrete data and competitive on continuous images.

\begin{table*}[t]
  \centering
  \footnotesize
  \setlength{\tabcolsep}{3.1pt}
  \renewcommand{\arraystretch}{1.06}
  \sisetup{detect-weight=true, detect-family=true}

  \begin{minipage}[t]{0.64\linewidth}
    \centering
    {\textsc{Discrete domains}}\hfill{\scriptsize $w\equiv 1$ (main), EDM recipe (transfer)}\par\vspace{0.15em}
    \resizebox{\linewidth}{!}{%
      \begin{tabular}{l
                      S[table-format=1.2] S[table-format=1.2] S[table-format=1.2]
                      S[table-format=3.0] S[table-format=3.0] S[table-format=3.0]
                      S[table-format=1.1]}
        \toprule
        & \multicolumn{3}{c}{Final score $\downarrow$}
        & \multicolumn{3}{c}{kEx to target $\downarrow$}
        & {Sp.\ $\times$} \\
        \cmidrule(lr){2-4}\cmidrule(lr){5-7}\cmidrule(lr){8-8}
        Dataset
          & {\textsc{LogUnif}} & {\textsc{EDM}} & {\textbf{$\sqrt{\phantom{1}}$ Ours}}
          & {\textsc{LogUnif}} & {\textsc{EDM}} & {\textbf{$\sqrt{\phantom{1}}$ Ours}}
          & {} \\
        \midrule
        bMNIST (FID)        & 0.63 & 0.53 & \bfseries 0.40
                           & 100  & 80   & \bfseries 40
                           & 2.0 \\
        bFashionMNIST (FID) & 1.10 & 0.94 & \bfseries 0.89
                           & 100  & 100  & \bfseries 40
                           & 2.5 \\
        DNA (Sei-FID)       & 0.61 & 0.45 & \bfseries 0.42
                           & 100  & 100  & \bfseries 40
                           & 2.7 \\
        \bottomrule
      \end{tabular}%
    }
  \end{minipage}\hfill
  \begin{minipage}[t]{0.35\linewidth}
    \centering
    {\textsc{Continuous images}}\hfill{\scriptsize EDM objective}\par\vspace{0.15em}
    \resizebox{0.98\linewidth}{!}{%
      \begin{tabular}{l
                      S[table-format=1.2] S[table-format=1.2] S[table-format=1.2]
                      S[table-format=1.1]}
        \toprule
        & \multicolumn{3}{c}{Final FID $\downarrow$} & {Sp.\ $\times$} \\
        \cmidrule(lr){2-4}\cmidrule(lr){5-5}
        Dataset
          & {\textsc{LogUnif}} & {\textsc{EDM}} & {\textbf{$\sqrt{\phantom{1}}$ Ours}}
          & {} \\
        \midrule
        MNIST                 & 1.02 & 0.44 & \bfseries 0.43 & 1.0 \\
        FashionMNIST          & 2.75 & 1.78 & \bfseries 1.71 & 1.0 \\
        CIFAR-10 (uncond.)    & 5.93 & 2.04 & \bfseries 1.98 & 1.4 \\
        CIFAR-10 (cond.)      & 4.45 & 1.85 & \bfseries 1.84 & 1.5 \\
        FFHQ $64\times 64$    & 4.68 & \bfseries 2.53 & 2.56 & \textemdash \\
        \bottomrule
      \end{tabular}%
    }
  \end{minipage}

  \caption{\textbf{Training efficiency across domains.}
  \textsc{LogUnif} denotes a log-uniform baseline, \textsc{EDM} the EDM/EDM2-style heuristic, and $\sqrt{\phantom{1}}$ Ours the entropic square-root schedule. Left: discrete domains trained with the unweighted objective $w\equiv 1$, reporting final score and thousands of processed examples (kEx) required to reach a fixed target when that comparison was measured. Right: continuous-image models trained with the EDM objective, reporting final FID; the last column summarizes matched-quality speedups when available. Dashes indicate configurations or threshold comparisons that were not measured.}
  \label{tab:efficiency_across_domains}
  \vspace{-0.6em}
\end{table*}

%% file: tables_toy_case/pr_table.tex
\begin{table*}[t!]
  \centering
  \scriptsize
  \setlength{\tabcolsep}{2.5pt}
  \renewcommand{\arraystretch}{1.06}
  \sisetup{detect-weight=true, detect-family=true}

  \begin{tabular}{l S[table-format=3.3] S[table-format=3.3] S[table-format=3.3] S[table-format=3.3] S[table-format=3.3] S[table-format=3.3] S[table-format=3.3] S[table-format=3.3] S[table-format=3.3]}
    \toprule
    & \multicolumn{9}{c}{Model + Distribution} \\
    \cmidrule(lr){2-10}
    Metric & {Swiss} & {Moons} & {MNIST} & {Dirac-Param-2} & {Dirac-Param-3} & {Dirac-Param-4} & {Dirac-NN-2} & {Dirac-NN-3} & {Dirac-NN-4} \\
    \midrule
    PR $\downarrow$ & 1.715 & 1.943 & 1.000 & 1.245 & 1.161 & 1.144 & 1.733 & 1.774 & 1.327 \\
    Top-1 $\uparrow$ & 0.750 & 0.695 & 1.000 & 0.895 & 0.927 & 0.933 & 0.749 & 0.731 & 0.862 \\
    Top-2 $\uparrow$ & 0.849 & 0.817 & 1.000 & 0.931 & 0.967 & 0.992 & 0.830 & 0.877 & 0.949 \\
    Top-3 $\uparrow$ & 0.927 & 0.918 & 1.000 & 0.947 & 0.980 & 0.996 & 0.892 & 0.957 & 0.996 \\
    Top-4 $\uparrow$ & 1.000 & 1.000 & 1.000 & 0.961 & 0.986 & 0.997 & 0.950 & 1.000 & 1.000 \\
    Top-5 $\uparrow$ & 1.000 & 1.000 & 1.000 & 0.971 & 0.990 & 0.998 & 0.999 & 1.000 & 1.000 \\
    \bottomrule
  \end{tabular}
  \caption{\textbf{Sparsity of the atomic schedule.} Reporting the participation ratio (PR) and the cumulative probability mass covered by the top-$k$ atoms for $k \in \{1, \dots, 5\}$ across various domains and schedules. $\mathrm{PR} = 1/\sum_i m_i^2$ measures the effective number of atoms. \textit{Param} and \textit{NN} denote Parametric and MLP neural network parameterizations respectively. $\downarrow$ means that as the metric diminishes, the effective number of atoms diminishes as well, $\uparrow$ has the opposite meaning.}
  \label{tab:sparsity_schedule}
  \vspace{-0.6em}
\end{table*}

%% file: sections/discussion.tex
\section{Discussion}\label{sec:discussion}

We have developed a general statistical framework for asymptotically optimal noise-level allocation in diffusion-model training.  In the frozen-feature regime, the ELBO-weighted scheduling objective depends on the sampling measure only through a finite-dimensional pair of curvature and noise operators, so its optimum can always be realized by an \emph{atomic} schedule supported on at most $p(p+1)+1$ noise levels, a consequence of the convex geometry of the operator curve and Carathéodory's theorem.  In the idealized independent-learner limit, under an additional feature-noise decoupling condition, the decoupled proxy collapses to the tractable square-root density $m^*(t)\propto\sqrt{\dot{\mathrm{Ent}}[x_0|x_t]}$, which is estimable directly from the denoising loss.

The two regimes are governed by how strongly the network's features couple across noise levels (\cref{sec:uncoupled}): when the coupling is strong the optimal schedule is atomic, and when the features specialize by noise level the optimal mass spreads into the smooth entropic density.  The experiments span this range.  Some cross-time coupling is present in every setting we consider, and as it weakens the atomic optimizer places its mass where the entropic density is large, so the much cheaper entropic schedule closely tracks the atomic optimum on the neural models and is clearly worse only in the fully coupled, low-parameter case.  In practice this suggests a simple rule: allocate training effort in proportion to the square root of the generative entropy rate.  It requires no schedule optimization and can be estimated from the training loss, and in our experiments it is competitive with hand-tuned weightings (\cref{sec:experiments}).

\paragraph{Limitations.}
The atomic-optimizer result (\cref{prop:atomic_optimizer}) relies on a well-specification assumption and the frozen-feature approximation: in the exact problem the Jacobians depend on $\mu$ through $\theta^\star(\mu)$, so the feasible set is no longer convex and the Carathéodory reduction does not directly apply.  The local-linear regime assumed throughout is exact for convex losses and well-supported for overparameterized networks near convergence, but deviations early in training may change which schedule is optimal.  The feature-noise decoupling condition underlying the square-root proxy is useful but not universal; when the network develops strongly input-conditional submodules with heterogeneous entropy profiles, higher-order corrections to the proxy become necessary.

\paragraph{Open problems.}
The most pressing open question is whether the frozen-feature atomic result can be generalized to the self-consistent setting, in which the schedule and the Jacobian geometry are optimized jointly.  A related question concerns genericity: our analysis guarantees that an atomic optimizer always \emph{exists}, but not that every optimizer is atomic, and characterizing when the objective is minimized at an extreme point of the operator curve's convex hull would settle whether sparsity is guaranteed or merely typical.  A second direction is the design of online algorithms that refine the entropy estimate and the schedule at the same time, without a pre-trained reference model, as in the information-based scheduling of \citep{raya2026information}.  Extending the framework beyond Polyak--Ruppert averaging, to momentum methods, adaptive optimizers, or mini-batch training, would broaden its applicability; achieving practically optimal schedulers in that more general setting, ideally by accumulating gradient information as in second-order optimization, remains for future work.

%% file: appendices/elbo_denoiser.tex
\section{Continuous-Time ELBO and Denoiser Matching}\label{supp sec: ELBO}

This appendix records the continuous-time ELBO in the form used throughout the paper.
We follow the standard variational diffusion derivation of \citet{huang2021variational}
and the path-space KL / score-based viewpoint emphasized by
\citet{franzese2023minde}, but specialize directly to our variance-exploding forward
channel and denoiser parameterization.

\subsection{Forward channel and path-space ELBO}

We consider the forward diffusion
\begin{equation}
dx_t = a(t)\,dw_t,
\qquad
x_t = x_0 + \sigma(t)\,z,
\qquad
z\sim\mathcal{N}(0,I),
\label{eq:forward_process}
\end{equation}
with
\begin{equation}
\sigma^2(t) = \int_0^t a^2(\tau)\,d\tau.
\label{eq:sigma_def}
\end{equation}
For a fixed clean sample $x_0$, let $\mathbb{Q}_{x_0}$ denote the corresponding
conditional path measure. A reverse-time generative model with score approximation
$\widetilde{s}_\theta(x_t,t)$ induces a path measure $\mathbb{P}_\theta$ through
\[
dx_t = -a^2(t)\,\widetilde{s}_\theta(x_t,t)\,dt + a(t)\,d\bar{w}_t.
\]
The usual variational inequality gives
\begin{equation}
-\log p_\theta(x_0)
\le
\mathrm{KL}\!\left(\mathbb{Q}_{x_0}\,\|\,\mathbb{P}_\theta\right).
\label{eq:path_vlb}
\end{equation}
By Girsanov's theorem, this path-space KL can be written as
\begin{equation}
\mathrm{KL}\!\left(\mathbb{Q}_{x_0}\,\|\,\mathbb{P}_\theta\right)
=
\frac{1}{2}\int_0^T
a^2(t)\,
\mathbb{E}\!\left[
\Bigl\|\nabla_x\log q(x_t\mid x_0)-\widetilde{s}_\theta(x_t,t)\Bigr\|^2
\right]dt
\;+\;\mathrm{const},
\label{eq:path_kl_score}
\end{equation}
where the constant is independent of $\theta$
\citep{huang2021variational,franzese2023minde}.

\subsection{Denoiser form and SNR weighting}

For the Gaussian forward channel,
\[
x_t\mid x_0 \sim \mathcal{N}(x_0,\sigma^2(t)I),
\qquad
\nabla_x\log q(x_t\mid x_0)
=
-\frac{x_t-x_0}{\sigma^2(t)}.
\]
We parametrize the reverse model by an $x_0$-prediction network $s_\theta(x_t,t)$ via
\[
\widetilde{s}_\theta(x_t,t)
=
\frac{s_\theta(x_t,t)-x_t}{\sigma^2(t)}.
\]
Substituting into \eqref{eq:path_kl_score} gives
\[
\nabla_x\log q(x_t\mid x_0)-\widetilde{s}_\theta(x_t,t)
=
\frac{x_0-s_\theta(x_t,t)}{\sigma^2(t)}.
\]
After averaging over $x_0\sim p_{\mathrm{data}}$ and $x_t\mid x_0$, the ELBO becomes
\[
\mathcal{L}(\theta)
:=
\mathbb{E}_{x_0}\!\left[
\mathrm{KL}\!\left(\mathbb{Q}_{x_0}\,\|\,\mathbb{P}_\theta\right)
\right]
=
c
+
\int_0^T
\frac{a^2(t)}{2\sigma^4(t)}\,
\hat{\mathscr{L}}_t(\theta)\,dt,
\]
where
\begin{equation}
\hat{\mathscr{L}}_t(\theta)
:=
\mathbb{E}\!\left[
\|x_0-s_\theta(x_t,t)\|^2
\right]
\label{eq:denoising_loss}
\end{equation}
is the denoising loss at time $t$.
Using $a^2(t)=2\sigma(t)\dot{\sigma}(t)$, we obtain
\begin{equation}
\frac{a^2(t)}{2\sigma^4(t)}
=
\frac{\dot{\sigma}(t)}{\sigma^3(t)}
=:
\dot{\mathrm{SNR}}(t),
\label{eq:snr_def}
\end{equation}
and therefore
\begin{equation}
\mathcal{L}(\theta)
=
c
+
\int_0^T \hat{\mathscr{L}}_t(\theta)\,\dot{\mathrm{SNR}}(t)\,dt.
\label{eq:continuous_elbo}
\end{equation}
Equivalently, $\dot{\mathrm{SNR}}(t)\,dt$ is the intrinsic weighting measure of the
forward channel, and is unchanged by smooth reparameterizations of time.

\subsection{Bayes decomposition}

Let
\begin{equation}
f_t(x_t) := \mathbb{E}[x_0\mid x_t]
\label{eq:bayes_denoiser}
\end{equation}
be the Bayes-optimal denoiser, and define the prediction error relative to it by
\begin{equation}
\mathscr{L}_t(\theta)
:=
\mathbb{E}_{x_t}\!\left[
\|s_\theta(x_t,t)-f_t(x_t)\|^2
\right].
\label{eq:matching_loss}
\end{equation}
Adding and subtracting $f_t(x_t)$ inside \eqref{eq:denoising_loss}, and using the
orthogonality condition
\(
\mathbb{E}[x_0-f_t(x_t)\mid x_t]=0,
\)
gives
\begin{equation}
\hat{\mathscr{L}}_t(\theta)
=
\mathrm{MMSE}(t) + \mathscr{L}_t(\theta),
\label{eq:loss_decomp}
\end{equation}
where
\begin{equation}
\mathrm{MMSE}(t)
:=
\mathbb{E}\!\left[
\|x_0-f_t(x_t)\|^2
\right]
\label{eq:mmse_def}
\end{equation}
depends only on the forward channel.
Substituting \eqref{eq:loss_decomp} into \eqref{eq:continuous_elbo},
\[
\mathcal{L}(\theta)
=
c
+
\int_0^T
\bigl(\mathrm{MMSE}(t)+\mathscr{L}_t(\theta)\bigr)\,
\dot{\mathrm{SNR}}(t)\,dt.
\]
Defining
\begin{equation}
c'
:=
c
+
\int_0^T \mathrm{MMSE}(t)\,\dot{\mathrm{SNR}}(t)\,dt,
\label{eq:cprime}
\end{equation}
we obtain the final form
\begin{equation}
\mathcal{L}(\theta)
=
c'
+
\int_0^T \mathscr{L}_t(\theta)\,\dot{\mathrm{SNR}}(t)\,dt.
\label{eq:elbo_final}
\end{equation}
Thus, minimizing the continuous-time ELBO is equivalent to minimizing the
$\dot{\mathrm{SNR}}$-weighted denoiser-matching objective, up to the model-independent
constant $c'$.

%% file: appendices/fokker_planck.tex
\section{Conditional Entropy Rate and Mutual-Information Decay}\label{supp sec: conditional entropy}

This appendix derives the information-theoretic identity used in the main text. A
rigorous treatment based on Girsanov's theorem, valid under minimal assumptions on the data
distribution, is given by \citet{franzese2023minde} (see also the variational
diffusion viewpoint of \citet{huang2021variational}). Here we give a short, self-contained
derivation that assumes directly the existence of smooth densities and enough integrability
to differentiate under the integral sign and to integrate by parts with vanishing boundary
terms.

We use the forward channel of Appendix~\ref{supp sec: ELBO},
\[
dx_t = a(t)\,dw_t,
\qquad
x_t = x_0 + \sigma(t)\,z,
\qquad
\sigma^2(t)=\int_0^t a^2(\tau)\,d\tau,
\]
in which only the noisy coordinate diffuses. Let $p_t(x)$ denote the marginal density of
$x_t$, $p_t(x\mid x_0)=\mathcal{N}(x;x_0,\sigma^2(t)I)$ the conditional, and
$p_t(x_0\mid x)$ the posterior. Because the forward SDE has no drift, both the conditional
$p_t(\cdot\mid x_0)$ and the marginal $p_t(\cdot)=\int p_t(\cdot\mid x_0)\,p_0(x_0)\,dx_0$
evolve by the same Fokker--Planck (heat) equation,
\begin{equation}
\partial_t\,p_t(x) = \frac{a^2(t)}{2}\,\Delta_x\,p_t(x),
\label{eq:fokker_planck}
\end{equation}
the conditional from the initial condition $\delta_{x_0}$ and the marginal by linearity.

\paragraph{Relative-entropy production.}
Let $p_t,q_t$ be two densities that both solve~\eqref{eq:fokker_planck}. Differentiating
$\mathrm{KL}(p_t\,\|\,q_t)=\int p_t\log(p_t/q_t)\,dx$ in time, substituting
\eqref{eq:fokker_planck}, and integrating by parts (boundary terms vanish by the assumed
decay) gives the relative-entropy production identity
\begin{equation}
\frac{d}{dt}\,\mathrm{KL}\!\left(p_t\,\|\,q_t\right)
=
-\,\frac{a^2(t)}{2}\,
\mathbb{E}_{x\sim p_t}\!\left[
\bigl\|\nabla_x\log p_t(x)-\nabla_x\log q_t(x)\bigr\|^2
\right].
\label{eq:rel_entropy_production}
\end{equation}
The score-difference (de~Bruijn) form of the KL divergence between diffused measures is
derived from Girsanov's theorem by \citet{franzese2023minde}.

\paragraph{Conditional entropy rate.}
Writing the mutual information as an average relative entropy and applying
\eqref{eq:rel_entropy_production} with $p_t=p_t(\cdot\mid x_0)$ and $q_t=p_t(\cdot)$,
\[
I(X_0;X_t)
=
\mathbb{E}_{x_0}\!\left[
\mathrm{KL}\!\left(p_t(\cdot\mid x_0)\,\|\,p_t(\cdot)\right)
\right],
\qquad
\frac{d}{dt}I(X_0;X_t)
=
-\,\frac{a^2(t)}{2}\,
\mathbb{E}\!\left[
\bigl\|\nabla_x\log p_t(x_t\mid x_0)-\nabla_x\log p_t(x_t)\bigr\|^2
\right].
\]
By Bayes' rule $p_t(x_0\mid x)\propto p_t(x\mid x_0)\,p_0(x_0)$, and since $p_0(x_0)$ does
not depend on $x$,
\begin{equation}
\nabla_x\log p_t(x_0\mid x_t)
=
\nabla_x\log p_t(x_t\mid x_0)-\nabla_x\log p_t(x_t).
\label{eq:bayes_score}
\end{equation}
Since $H(X_0)$ is constant along the forward process,
$H(X_0\mid X_t)=H(X_0)-I(X_0;X_t)$ gives
$\dot{\Ent}[x_0\mid x_t]=-\tfrac{d}{dt}I(X_0;X_t)$, hence
\begin{equation}
\frac{d}{dt}\Ent[x_0\mid x_t]
=
\frac{a^2(t)}{2}\,
\mathbb{E}\!\left[
\|\nabla_x \log p_t(x_0\mid x_t)\|^2
\right].
\label{eq:cond_entropy_score}
\end{equation}

\paragraph{Gaussian channel and the MMSE form.}
For the Gaussian channel $\nabla_x \log p_t(x_t\mid x_0)=-(x_t-x_0)/\sigma^2(t)$, so
\eqref{eq:bayes_score} gives
\[
\nabla_x \log p_t(x_0\mid x_t)
=
\frac{x_0-\mathbb{E}[x_0\mid x_t]}{\sigma^2(t)}.
\]
Substituting into \eqref{eq:cond_entropy_score} and using the definitions of
$\dot{\mathrm{SNR}}(t)$ \eqref{eq:snr_def} and $\mathrm{MMSE}(t)$ \eqref{eq:mmse_def},
\begin{equation}
\dot{\Ent}[x_0\mid x_t]
=
\frac{a^2(t)}{2\sigma^4(t)}\,
\mathbb{E}\!\left[\|x_0-\mathbb{E}[x_0\mid x_t]\|^2\right]
=
\dot{\mathrm{SNR}}(t)\,\mathrm{MMSE}(t).
\label{eq:cond_entropy_mmse}
\end{equation}
Equivalently, since $I(X_0;X_t)=H(X_0)-H(X_0\mid X_t)$ with $H(X_0)$ constant,
\begin{equation}
\frac{d}{dt}I(X_0;X_t)
=
-\dot{\mathrm{SNR}}(t)\,\mathrm{MMSE}(t).
\label{eq:mi_decay}
\end{equation}
Thus the ``generative entropy rate'' used in the paper can be read equivalently as the
growth rate of conditional entropy or as the loss of mutual information between the clean
sample and its noisy version along the forward diffusion.

%% file: appendices/random_jacobian.tex
\section{Random Jacobian Model and Concentration Bounds}
\label{supp sec: random matrix}

This section records the random-matrix heuristic underlying the feature--noise decoupling
reduction used in the main text. We show that under an isotropic random-row model the
projector $P(x_t)$ has scalar expectation, and that even without independence one obtains a
useful operator-norm bound.

Throughout, $J_\theta(x_t)\in\mathbb{R}^{d\times p}$ and $r:=p/d$.

\subsection{Heuristic random-row model}

For a fixed time $t$, we model the data-induced randomness of $J_\theta(x_t)$ as follows.

\paragraph{Assumption A (random rows, heuristic).}
As $x_t$ varies under its marginal law, the rows of $J_\theta(x_t)$ are modeled as independent,
centered random vectors $j_i(x_t)\in\mathbb{R}^p$ satisfying
\[
\mathbb{E}[j_i(x_t)j_i(x_t)^\top] = \Lambda_t \succ 0.
\]
Define the whitened rows
\[
x_i(x_t):=\Lambda_t^{-1/2}j_i(x_t),
\]
and the whitened matrix
\[
X(x_t):=J_\theta(x_t)\Lambda_t^{-1/2}.
\]
Assume the $x_i(x_t)$ are isotropic and sub-Gaussian with uniformly bounded sub-Gaussian norm.
Then
\[
\mathbb{E}[J_\theta(x_t)^\top J_\theta(x_t)] = d\,\Lambda_t,
\]
so the projector-like matrix from the main text becomes
\[
P(x_t)
=
J_\theta(x_t)\,\mathbb{E}[J_\theta(x_t)^\top J_\theta(x_t)]^{-1}J_\theta(x_t)^\top
=
\frac{1}{d}X(x_t)X(x_t)^\top.
\]

\subsection{Scalar expectation under independence}

If $\Sigma_t(x_t)$ is deterministic or independent of $J_\theta(x_t)$, then
\[
\mathbb{E}[P(x_t)\Sigma_t(x_t)]
=
\mathbb{E}[P(x_t)]\,\mathbb{E}[\Sigma_t(x_t)].
\]
It remains to compute $\mathbb{E}[P(x_t)]$. Since
\[
(X X^\top)_{ab} = x_a^\top x_b,
\]
isotropy gives
\[
\mathbb{E}[x_a^\top x_a] = p,
\qquad
\mathbb{E}[x_a^\top x_b] = 0 \quad (a\neq b),
\]
hence
\[
\mathbb{E}[X X^\top] = p\,I_d.
\]
Therefore
\[
\mathbb{E}[P(x_t)] = \frac{p}{d}I_d = r I_d.
\]
Substituting yields
\[
\mathbb{E}[\operatorname{Tr}(P(x_t)\Sigma_t(x_t))]
=
r\,\operatorname{Tr}\!\left(\mathbb{E}[\Sigma_t(x_t)]\right).
\]
Using the conditional de~Bruijn identity
\[
\operatorname{Tr}\!\left(\mathbb{E}[\Sigma_t(x_t)]\right)
=
\frac{\dot{\Ent}[x_0\mid x_t;t,\varphi]}{\dot{\mathrm{SNR}}(t)},
\]
one obtains the heuristic proxy
\[
L_t(n,t)
\approx
\frac{r}{n}
\frac{\dot{\Ent}[x_0\mid x_t;t,\varphi]}{\dot{\mathrm{SNR}}(t)}.
\]

\subsection{Operator-norm bound without independence}

When $P(x_t)$ and $\Sigma_t(x_t)$ are not independent, the exact factorization above fails.
However, for positive semidefinite matrices,
\[
\operatorname{Tr}(P\Sigma)\le \lambda_{\max}(P)\,\operatorname{Tr}(\Sigma).
\]
Under Assumption~A, $P(x_t)=d^{-1}XX^\top$ is a sample covariance matrix of isotropic
sub-Gaussian rows. Standard results (e.g.\ Vershynin, 2018) imply the following concentration
statement.

\begin{lemma}
\label{lemma:MP}
For any $\varepsilon>0$, there exists $d_0(\varepsilon)$ such that for all $d\ge d_0$, with
probability at least $1-\exp(-cd)$,
\[
\lambda_{\max}(P(x_t))
\le
(1+\sqrt{r}+\varepsilon)^2.
\]
If the sub-Gaussian norms and covariance condition numbers are uniform in $t$, the same bound
holds uniformly over $t$.
\end{lemma}

Define
\[
C_r(\varepsilon):=\frac{(1+\sqrt{r}+\varepsilon)^2}{r}.
\]
Then Lemma~\ref{lemma:MP} gives
\[
\mathbb{E}[\operatorname{Tr}(P(x_t)\Sigma_t(x_t))]
\lesssim
C_r(\varepsilon)\,r\,
\operatorname{Tr}\!\left(\mathbb{E}[\Sigma_t(x_t)]\right),
\]
and therefore the upper bound
\[
L_t(n,t)
\lesssim
\frac{C_r(\varepsilon)\,r}{n}\,
\frac{\dot{\Ent}[x_0\mid x_t;t,\varphi]}{\dot{\mathrm{SNR}}(t)}.
\]

\subsection{Interpretation}

The random-row model should be read as a mechanism, not a theorem about real network Jacobians.
Its message is that isotropy plus weak feature-noise correlation make the average projector
look scalar, while concentration gives a fallback upper bound even when exact independence is lost.
Thus the random-matrix heuristic supports the entropy-rate proxy used in the main text.

%% file: sections/feature_noise_decoupling.tex
\section{Simplification under feature-noise decoupling}
\label{sec:feature_noise_decoupling}
The formula in Eq.~\ref{eq:general_prediction_error} gives the correct asymptotic scaling. However, the expression is hard to estimate and obscures the dominant sources of time dependence. To isolate the dominant time dependence, we consider the following \emph{heuristic model}: $J_{\theta}(x_t)$ is a random (not necessarily Gaussian) matrix with independent row distribution
\begin{align}
    & \mean{J^j_{\theta}(x_t)}{x_t} = 0 \\
    & \mean{{J^j_{\theta}(x_t)}^\top J^j_{\theta}(x_t)}{x_t} = \Lambda_t
\end{align}
where $\Lambda_t$ is a time-dependent covariance matrix that regulates the correlations in the Jacobian columns for different values of $x_t$. Note that random matrix models for Jacobians are widely used in the theoretical deep learning literature, especially in the context of wide networks \citep{jacot2018neural,arora2019exact, matthews2018gaussian}. Note however that in this case the randomness comes from the input $x_t$, not from the random initialization.

Notice that in Eq.~\ref{eq:general_prediction_error}, the error depends on the Jacobian only through a projector-like matrix:
\begin{align}
    & \mean{\mathcal{P}(x_t) \Sigma_t(x_t)}{x_t} \\
    & \mathcal{P}(x_t) = J_{\theta}(x_t) \mean{J_{\theta}(x_t)^\top J_{\theta}(x_t)}{x_t}^{-1} J_{\theta}^\top(x_t)
\end{align}
A useful cancellation mechanism is that $\mathcal P(x_t)$ be approximately isotropic in output space and only weakly correlated with $\Sigma_t(x_t)$. In the random-row model of Supp.~\ref{supp sec: random matrix}, one obtains $\mean{\mathcal P(x_t)}{x_t}=r I$ with $r=p/d$; if in addition $\mathcal P(x_t)$ and $\Sigma_t(x_t)$ are independent, then
\begin{equation}
    \mean{\mathcal{P}(x_t) \Sigma_t(x_t)}{x_t}
    =
    \mean{\mathcal{P}(x_t)}{x_t}\,\mean{\Sigma_t(x_t)}{x_t}
    \approx
    r\,\mean{\Sigma_t(x_t)}{x_t}.
\end{equation}
In this regime, the dependence on $\Lambda_t$ enters only through the scalar ratio $r$, leading to the proxy asymptotics
\begin{equation} \label{eq:general_asymptotic_error_uncoupled}
   \mathscr{L}(n;t) \approx \frac{r \dot{\Ent}[x_0|x_t;t,\phi]}{n \dot{\text{SNR}}(t)}~,
 \end{equation}
which matches the decoupled reciprocal-cost analysis after the additional entropy-rate approximation. Even without such independence, Marchenko--Pastur concentration gives a useful upper bound in the large-dimensional random-row model under sub-Gaussian assumptions (Supp.~\ref{supp sec: random matrix}):
\begin{equation} \label{eq:general_asymptotic_error_bound}
   \mathscr{L}(n;t) \leq \frac{r C_r(\epsilon) \dot{\Ent}[x_0|x_t;t,\phi]}{n \dot{\text{SNR}}(t)} +
o\!\left(\frac{1}{n}\right)~,
 \end{equation}
where $C_r(\epsilon)$ is time-independent and approaches $1$ in the high dimensional limit. Here, the bound comes from the fact that we can bound the maximum eigenvalue of $\mathcal{P}_t$ independently of $t$.

The random-matrix model above is only a heuristic device. For deterministic feature maps or kernel/NTK-like regimes, the entropy-rate proxy is accurate only when the average projector is near-scalar and feature-noise correlations remain weak.

%% file: appendices/toy_experiments.tex
\section{Experiments}\label{appx:experiments}

Across all settings we compare the atomic optimizer of the coupled theory with the smooth entropic schedule of the uncoupled theory. Parametric models provide a controlled setting with closed-form scores, where the coupled objective is optimized directly near the optimum; the neural-network experiments repeat the analysis with a genuine shared-parameter denoiser, on 2D manifolds and on image data (MNIST). In our experiments, we follow \cref{alg:schedule_opt} to infer the optimal schedule via gradient descent optimization over \(K=10^4\) optimization steps with the Adam optimizer, a total of \(N=100\) time points and learning rate \(\eta=0.1\).

\input{appendices/schedule_opt_algo}

\subsection{Common setup}

We study one-dimensional VE diffusion tasks (a mixture of Dirac masses with 2,3, or 4 components), 2D datasets (Swiss Roll and Moons), and image data (MNIST). For each dataset, we train a model until convergence; this model is then used to compute the entropic and atomic schedules for a specified set of parameters of the model. Afterwards, we retrain the chosen parameters, while keeping other parameters frozen. We do not reinitialize the trainable parameters from scratch; instead, we perturb them with Gaussian noise to place them near the optimum, as the theory requires.
Finally, we retrain with each schedule over 30 different random seeds for toy data and with 20 seeds for image data. The corresponding learning curves report the Negative Log-Likelihood (NLL) in Bits per Dimension (BPD), along with the \gls{SEM}.

\subsubsection{Dirac mixture}\label{sec:dirac}

\begin{figure}[!ht]
    \centering
    \begin{subfigure}[b]{0.48\textwidth}
        \centering
        \includegraphics[width=\textwidth]{figures_toy_case/objs_log_scale/objs_delta_param_2.pdf}
    \end{subfigure}
    \hfill 
    \begin{subfigure}[b]{0.48\textwidth}
        \centering
        \includegraphics[width=\textwidth]{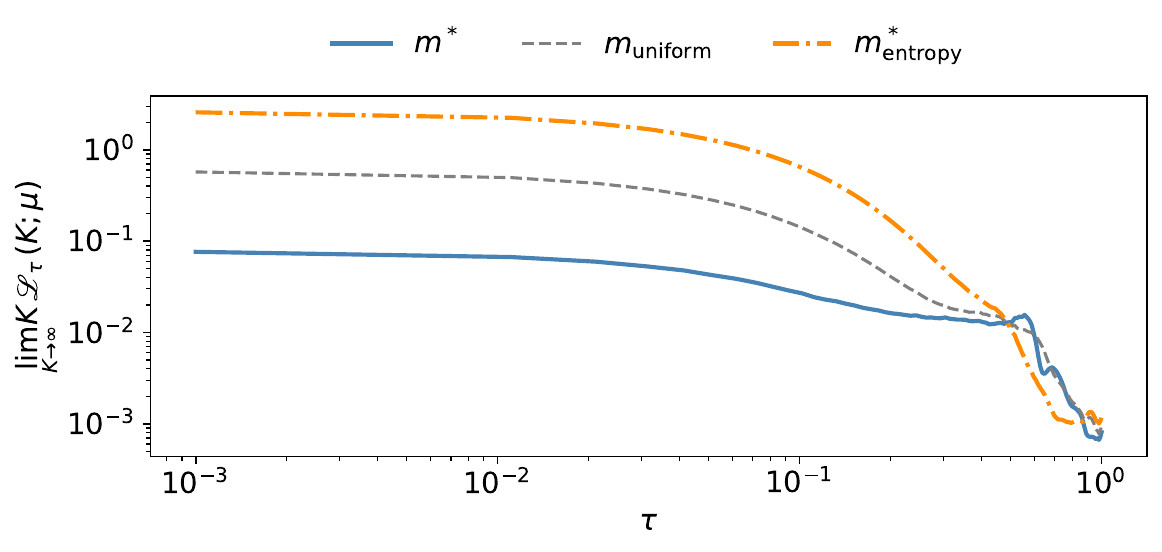}
    \end{subfigure}

    \vspace{0.3cm} 

    \begin{subfigure}[b]{0.48\textwidth}
        \centering
        \includegraphics[width=\textwidth]{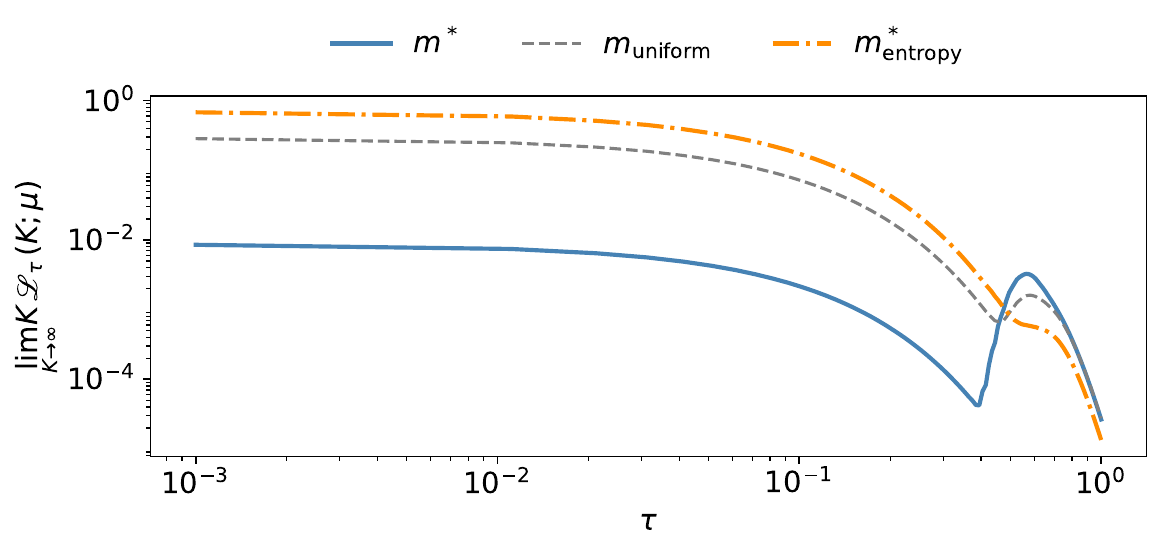}
    \end{subfigure}
    \hfill
    \begin{subfigure}[b]{0.48\textwidth}
        \centering
        \includegraphics[width=\textwidth]{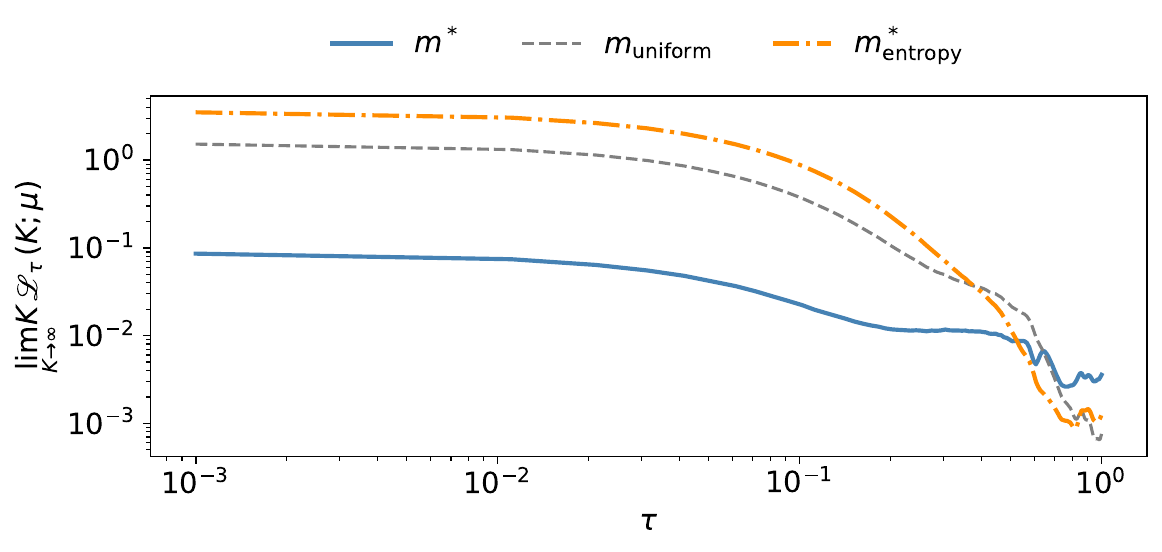}
    \end{subfigure}

    \vspace{0.3cm} 

    \begin{subfigure}[b]{0.48\textwidth}
        \centering
        \includegraphics[width=\textwidth]{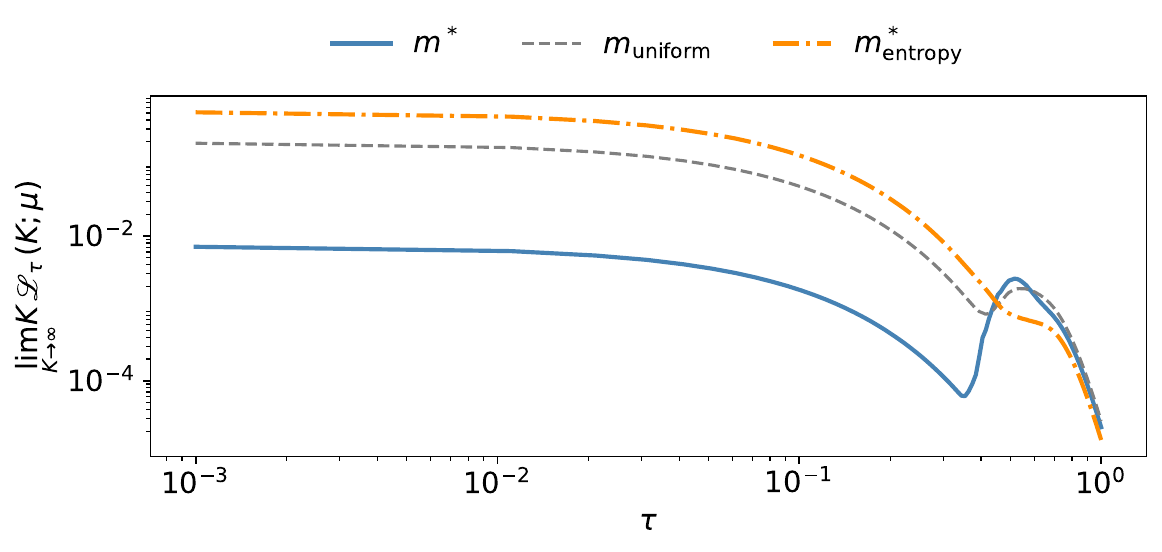}
    \end{subfigure}
    \hfill
    \begin{subfigure}[b]{0.48\textwidth}
        \centering
        \includegraphics[width=\textwidth]{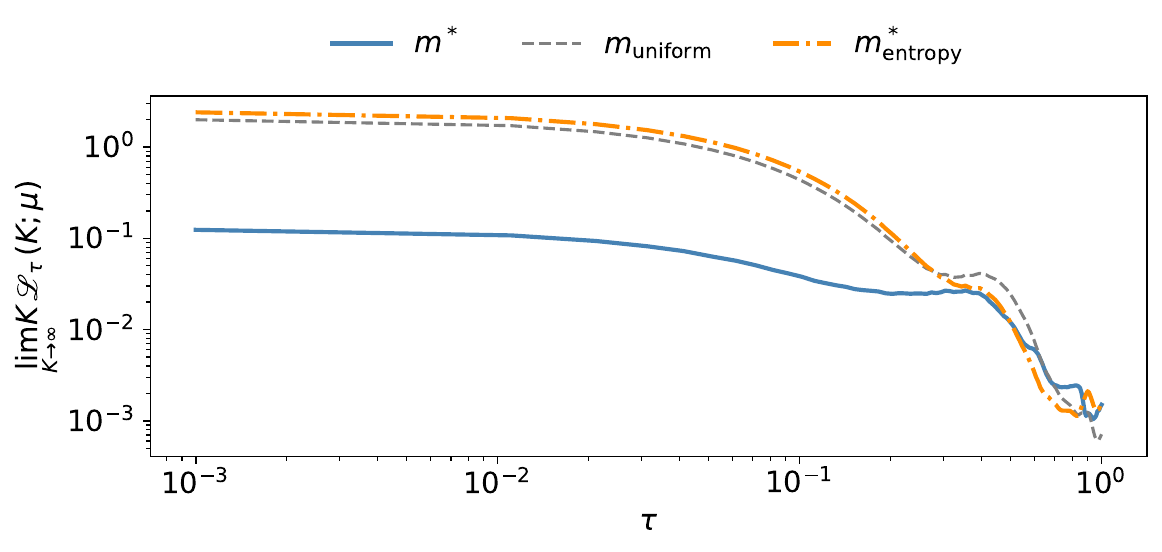}
    \end{subfigure}

    \caption{ELBO-weighted objective for the atomic, entropic and uniform schedules. The curves on the left refer to the parametric model, the ones on the right to the \gls{MLP}. From top to bottom, the first row correspond to the 2 mixture component distribution, the second to 3 components and the last to 4.}
    \label{fig:dirac_sched_objs}
\end{figure}

\begin{figure}[!ht]
    \centering
    \begin{subfigure}[b]{0.48\textwidth}
        \centering
        \includegraphics[width=\textwidth]{figures_toy_case/schedules_log_scale/schedules_delta_param_2.pdf}
    \end{subfigure}
    \hfill 
    \begin{subfigure}[b]{0.48\textwidth}
        \centering
        \includegraphics[width=\textwidth]{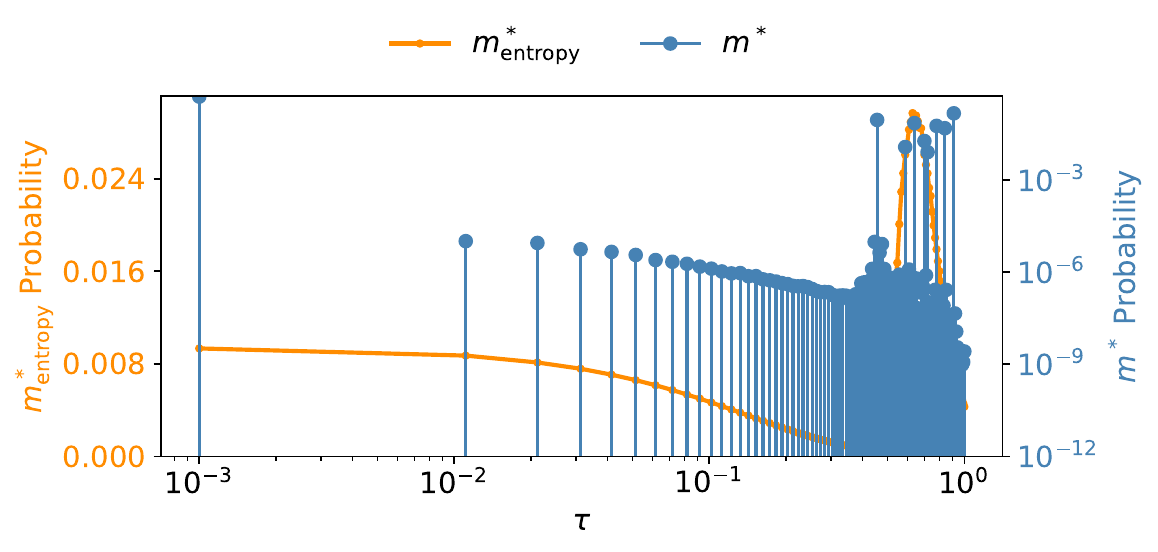}
    \end{subfigure}

    \vspace{0.3cm} 

    \begin{subfigure}[b]{0.48\textwidth}
        \centering
        \includegraphics[width=\textwidth]{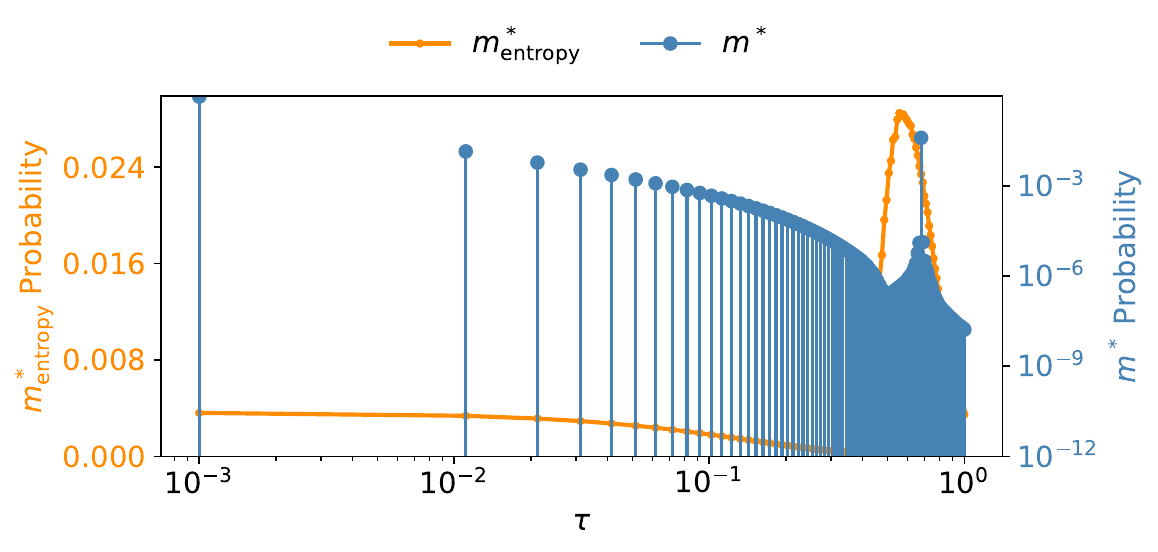}
    \end{subfigure}
    \hfill
    \begin{subfigure}[b]{0.48\textwidth}
        \centering
        \includegraphics[width=\textwidth]{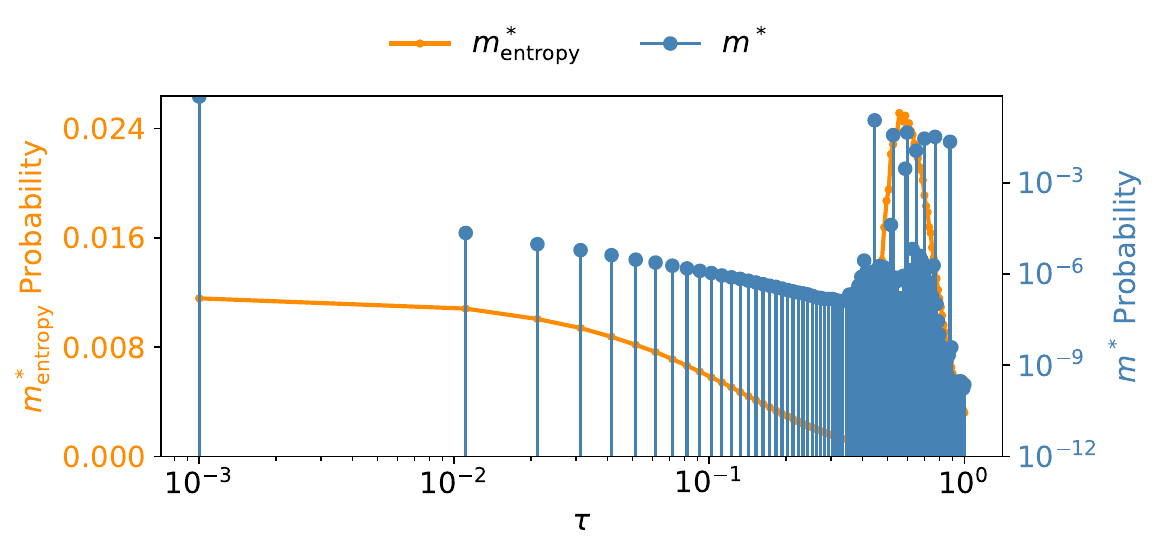}
    \end{subfigure}

    \vspace{0.3cm} 

    \begin{subfigure}[b]{0.48\textwidth}
        \centering
        \includegraphics[width=\textwidth]{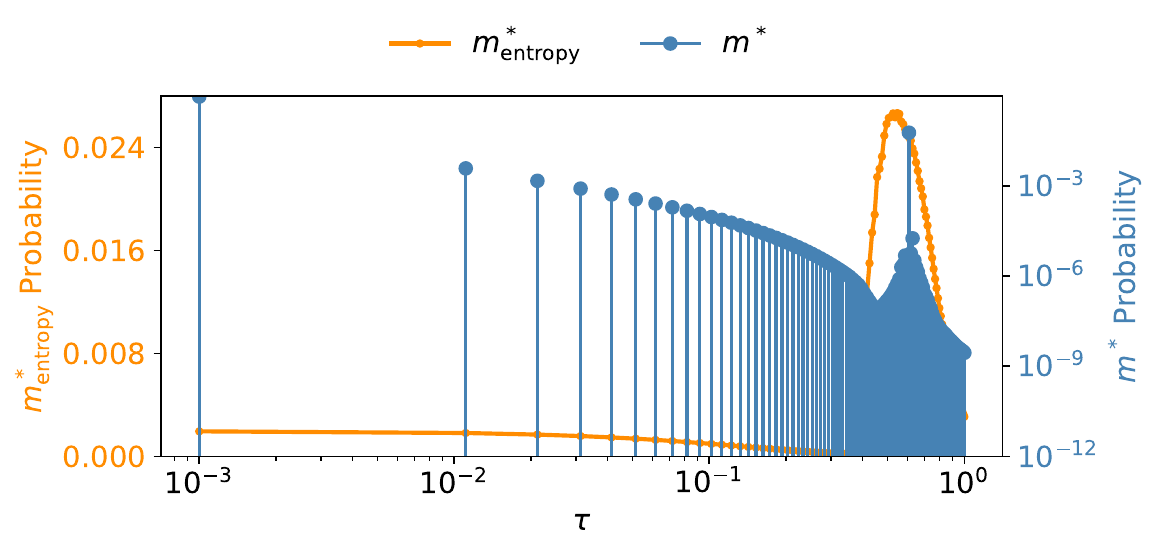}
    \end{subfigure}
    \hfill
    \begin{subfigure}[b]{0.48\textwidth}
        \centering
        \includegraphics[width=\textwidth]{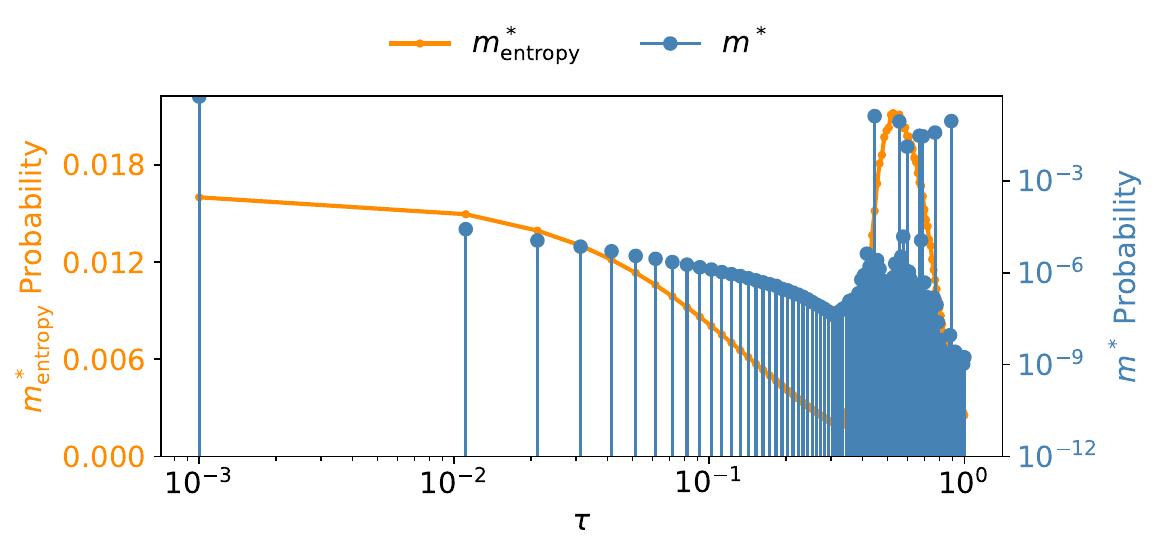}
    \end{subfigure}

    \caption{Atomic and entropic schedules for the mixture of delta distribution. The curves on the left refer to the parametric model, the ones on the right to the \gls{MLP}. From top to bottom, the first row correspond to the 2 mixture component distribution, the second to 3 components and the last to 4.}
    \label{fig:dirac_schedules}
\end{figure}

\paragraph*{Data and models} For the Dirac-mixture setting, we study two different models: a simple \gls{MLP} and a closed-form model parameterized by \(\theta_\delta=\{\{\pi_k\},\{m_k\}\}\),
\begin{equation}
    p_0(x;\theta_\delta)=\sum_{k=1}^{K}\pi_k\,\delta(x-m_k),
\end{equation}
with VE perturbation
\begin{equation}
    p_t(x;\theta_\delta)=\sum_{k=1}^{K}\pi_k\,\mathcal{N}(x;m_k,\sigma(t)^2 I),
\end{equation}
and score
\begin{equation}
    \nabla_x\log p_t(x;\theta_\delta)=\sum_k \rho_k(x,t;\theta_\delta)\,\frac{m_k-x}{\sigma(t)^2},
\end{equation}
where
\[
    \rho_k(x,t;\theta_\delta)
    =
    \frac{\pi_k\,\mathcal{N}(x;m_k,\sigma(t)^2 I)}
         {\sum_{\ell=1}^{K}\pi_\ell\,\mathcal{N}(x;m_\ell,\sigma(t)^2 I)}
\]
are the corresponding posterior responsibilities after noising. In the \gls{MLP} case, we employ a 50k parameters architecture. Differently from the previous setup, we only focus on the last layer, with 32 parameters. In this case, we restrict the coupled optimization to the selected parameters and compute an optimized schedule with frozen features from previous layers. Similarly, we benchmark the efficacy of the atomic schedule by freezing intermediate layers and only training the last layer. In all cases, \cref{fig:dirac_learning_curves} shows that the atomic schedule outperforms the alternatives. The entropic schedule, gives suboptimal performances, especially in the parametric model case. This is consistent with the theory: the parametric model lives in a fully temporally coupled regime and its limited set of parameters is independent of the diffusion time, opposite to the scenario depicted in \cref{sec:uncoupled}. This is validated by the cross-time coupling strength in \cref{fig:dirac_couplings}.
We report the atomic and entropic schedules in \cref{fig:dirac_schedules}, as we can see, in all cases the number of atoms is smaller than the square of the number of parameters, confirming our theoretical results. The ELBO-weighted objective in \cref{fig:dirac_sched_objs} shows consistent results with the generative performance.

\begin{figure}[!ht]
    \centering
    \begin{subfigure}[b]{0.48\textwidth}
        \centering
        \includegraphics[width=\textwidth]{figures_toy_case/delta/analytical/ParametricModel_mod_n2_bpd_learning_curves.pdf}
    \end{subfigure}
    \hfill 
    \begin{subfigure}[b]{0.48\textwidth}
        \centering
        \includegraphics[width=\textwidth]{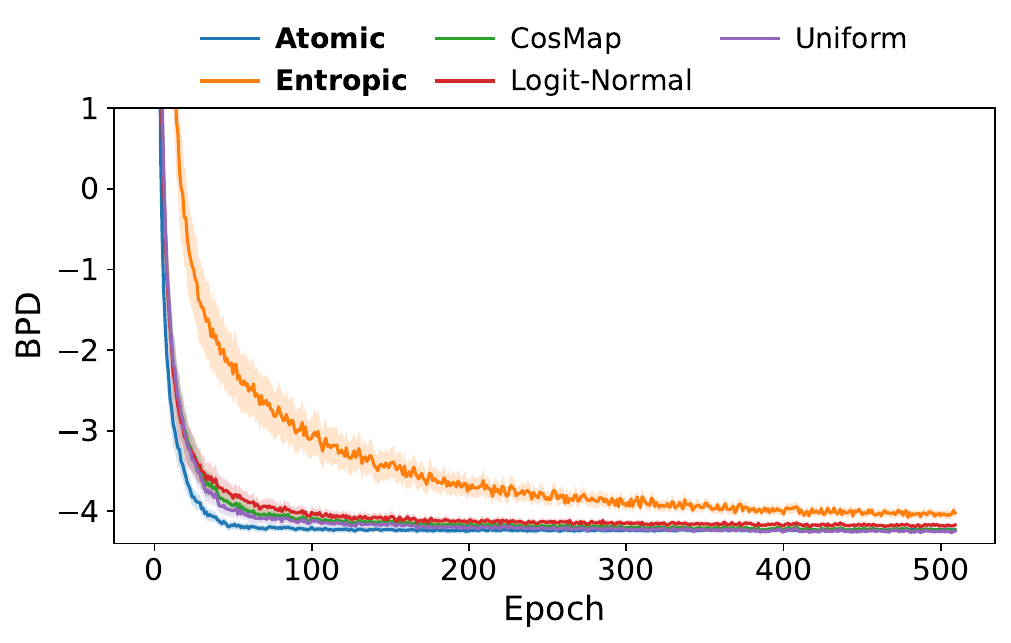}
    \end{subfigure}

    \vspace{0.3cm} 

    \begin{subfigure}[b]{0.48\textwidth}
        \centering
        \includegraphics[width=\textwidth]{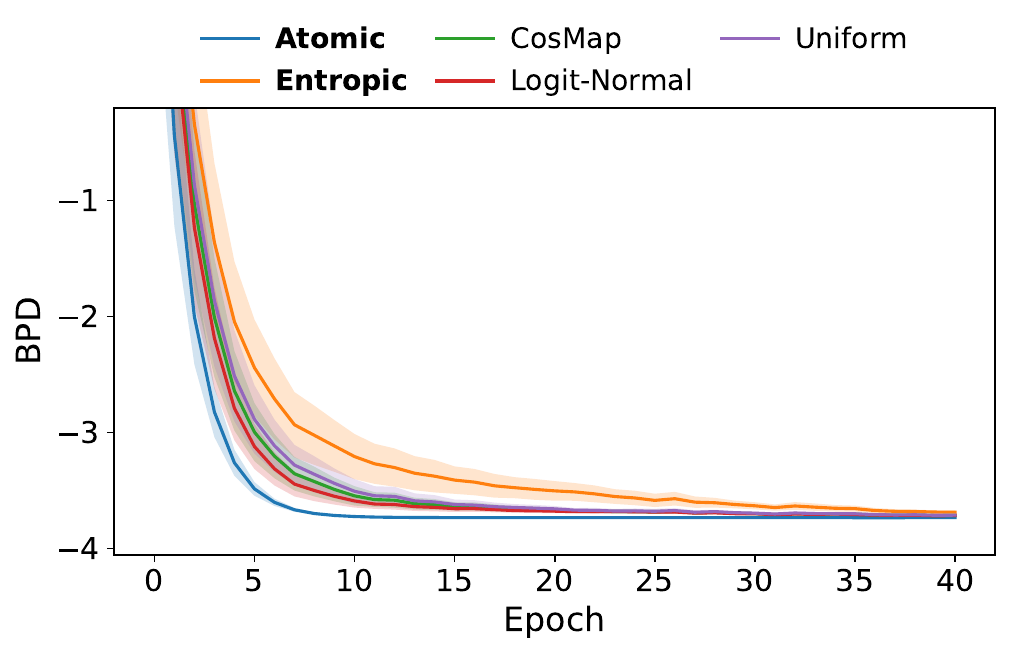}
    \end{subfigure}
    \hfill
    \begin{subfigure}[b]{0.48\textwidth}
        \centering
        \includegraphics[width=\textwidth]{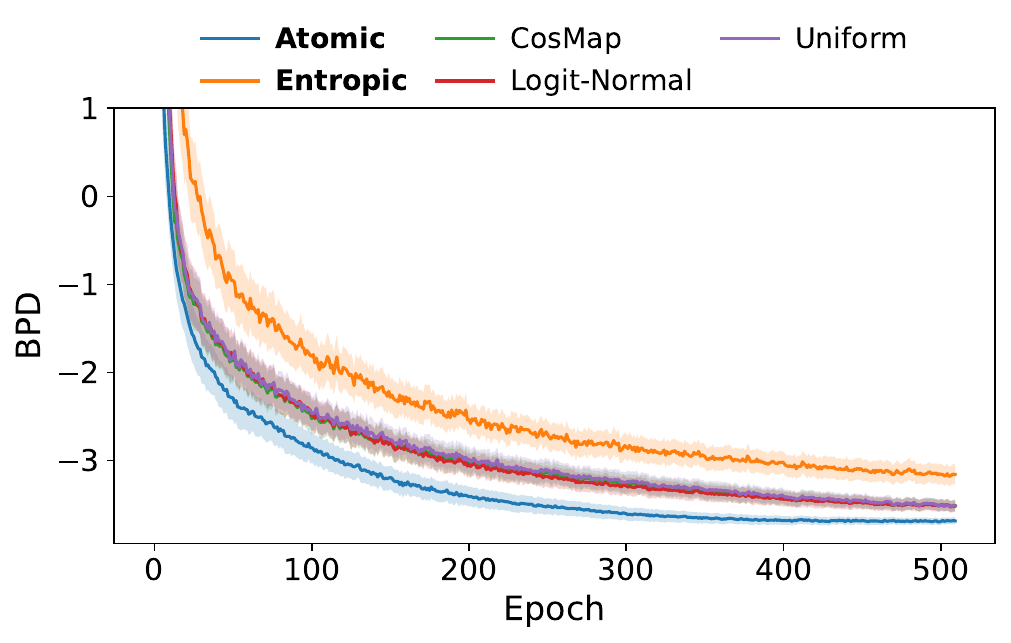}
    \end{subfigure}

    \vspace{0.3cm} 

    \begin{subfigure}[b]{0.48\textwidth}
        \centering
        \includegraphics[width=\textwidth]{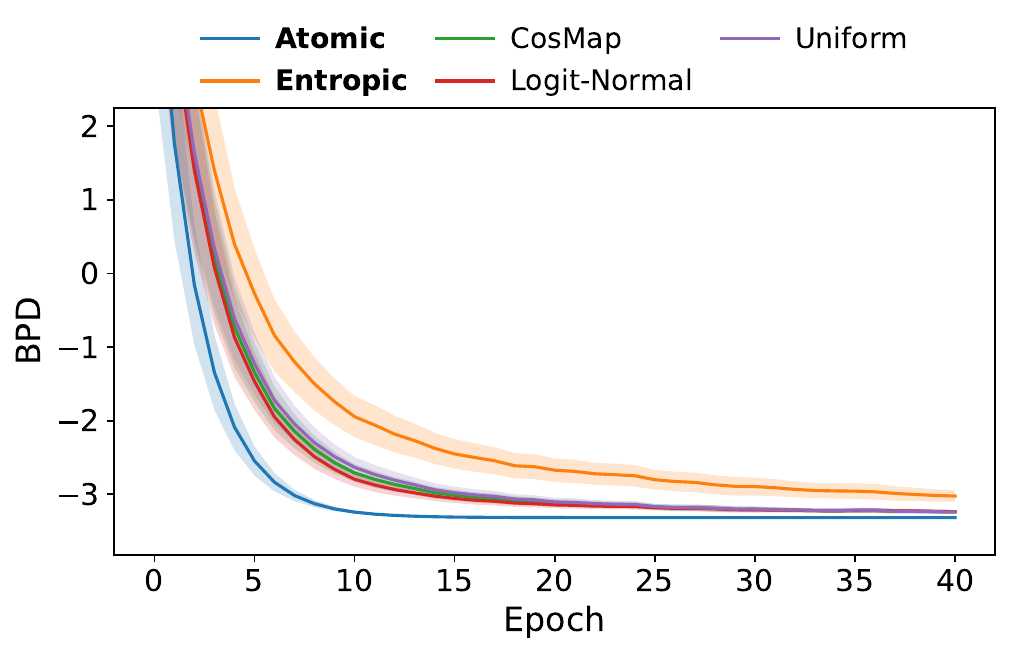}
    \end{subfigure}
    \hfill
    \begin{subfigure}[b]{0.48\textwidth}
        \centering
        \includegraphics[width=\textwidth]{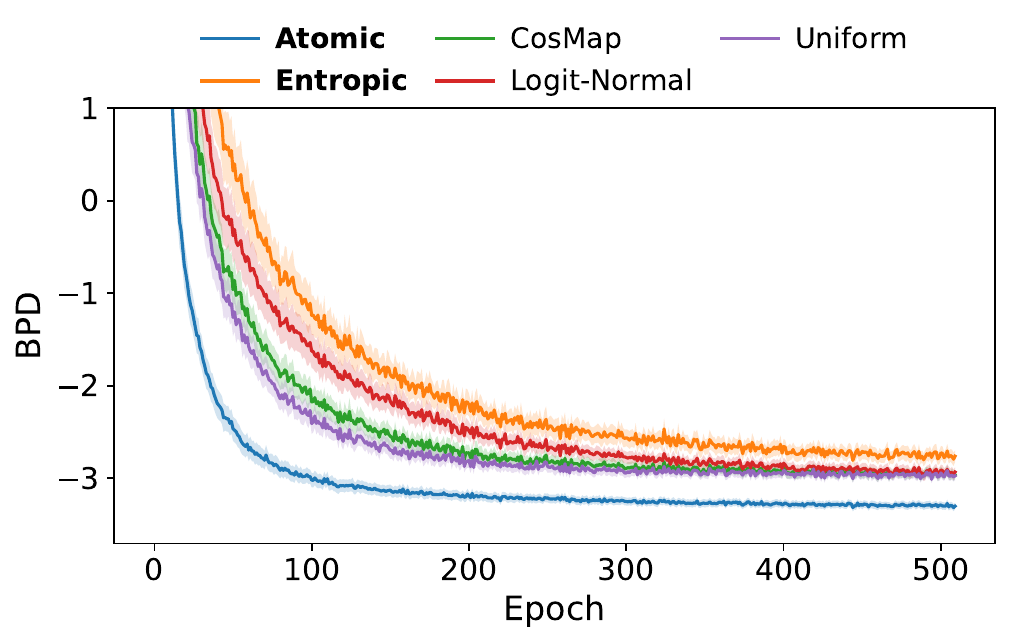}
    \end{subfigure}

    \caption{NLL Learning curve for each time schedule with the mixture of delta distribution. The curves on the left refer to the parametric model, the ones on the right to the \gls{MLP}. From top to bottom, the first row correspond to the 2 mixture component distribution, the second to 3 components and the last to 4.}
    \label{fig:dirac_learning_curves}
\end{figure}

\begin{figure}[!ht]
    \centering
    \begin{subfigure}[b]{0.48\textwidth}
        \centering
        \includegraphics[width=\textwidth]{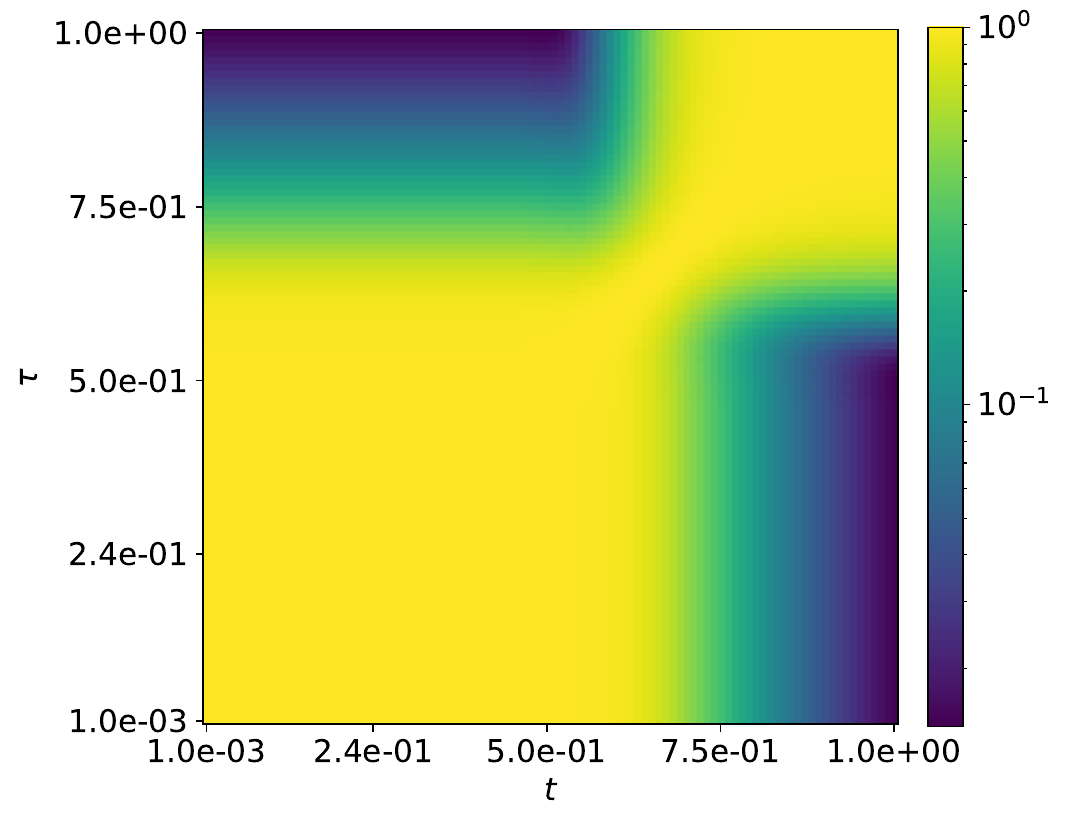}
    \end{subfigure}
    \hfill 
    \begin{subfigure}[b]{0.48\textwidth}
        \centering
        \includegraphics[width=\textwidth]{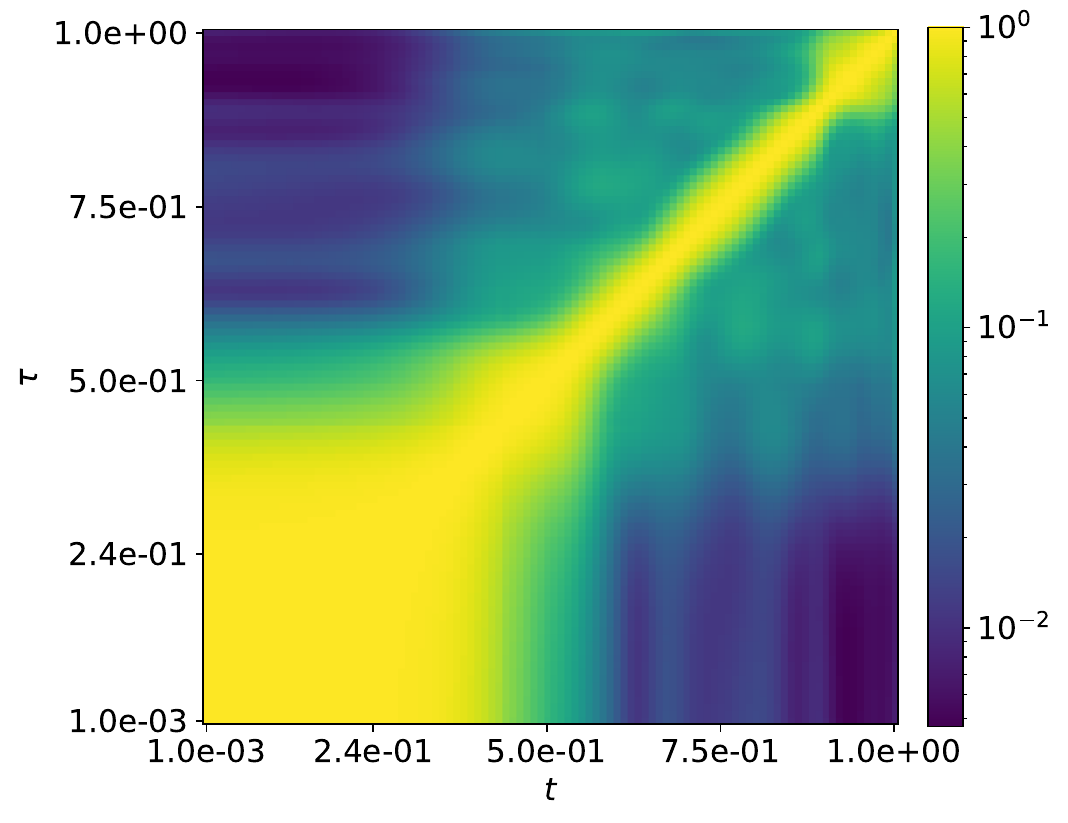}
    \end{subfigure}

    \vspace{0.3cm} 

    \begin{subfigure}[b]{0.48\textwidth}
        \centering
        \includegraphics[width=\textwidth]{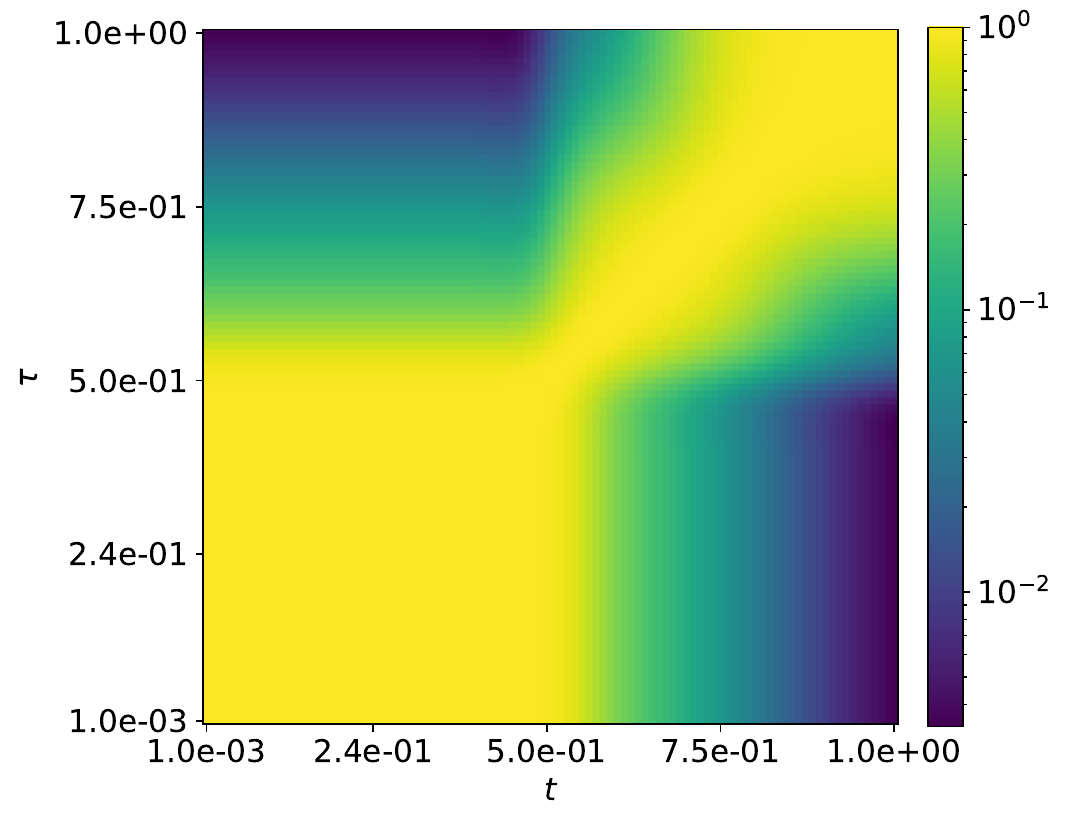}
    \end{subfigure}
    \hfill
    \begin{subfigure}[b]{0.48\textwidth}
        \centering
        \includegraphics[width=\textwidth]{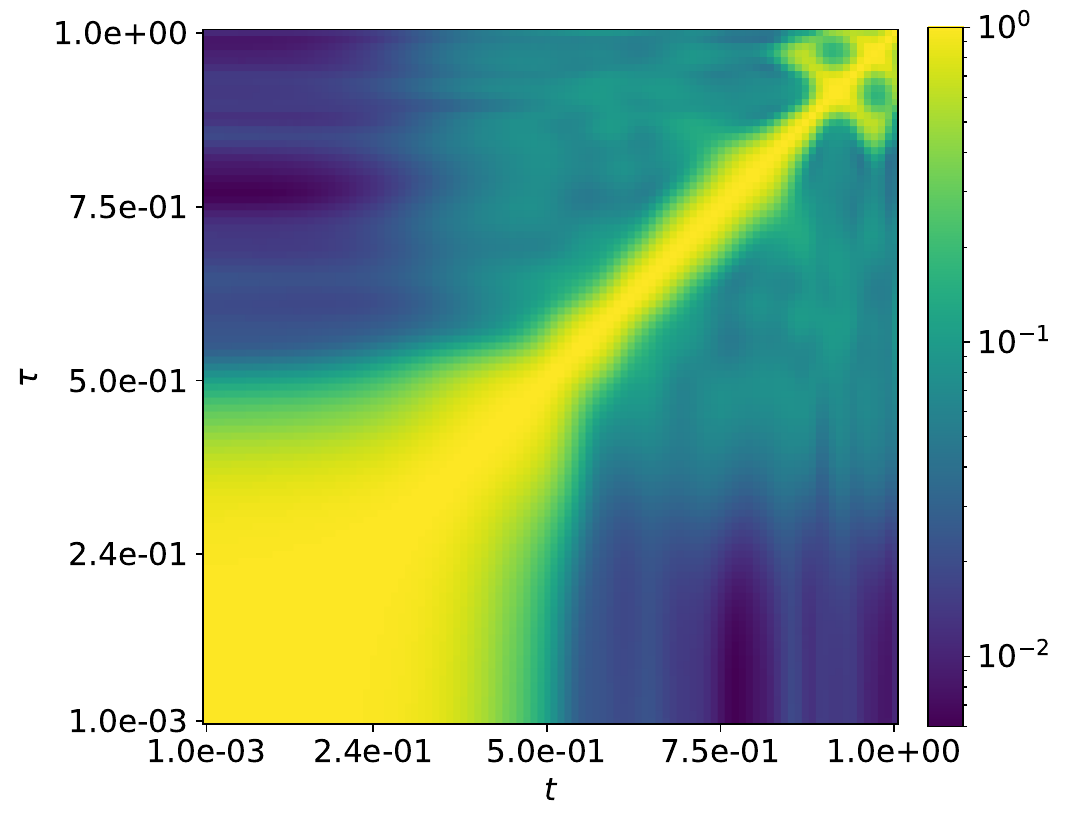}
    \end{subfigure}

    \vspace{0.3cm} 

    \begin{subfigure}[b]{0.48\textwidth}
        \centering
        \includegraphics[width=\textwidth]{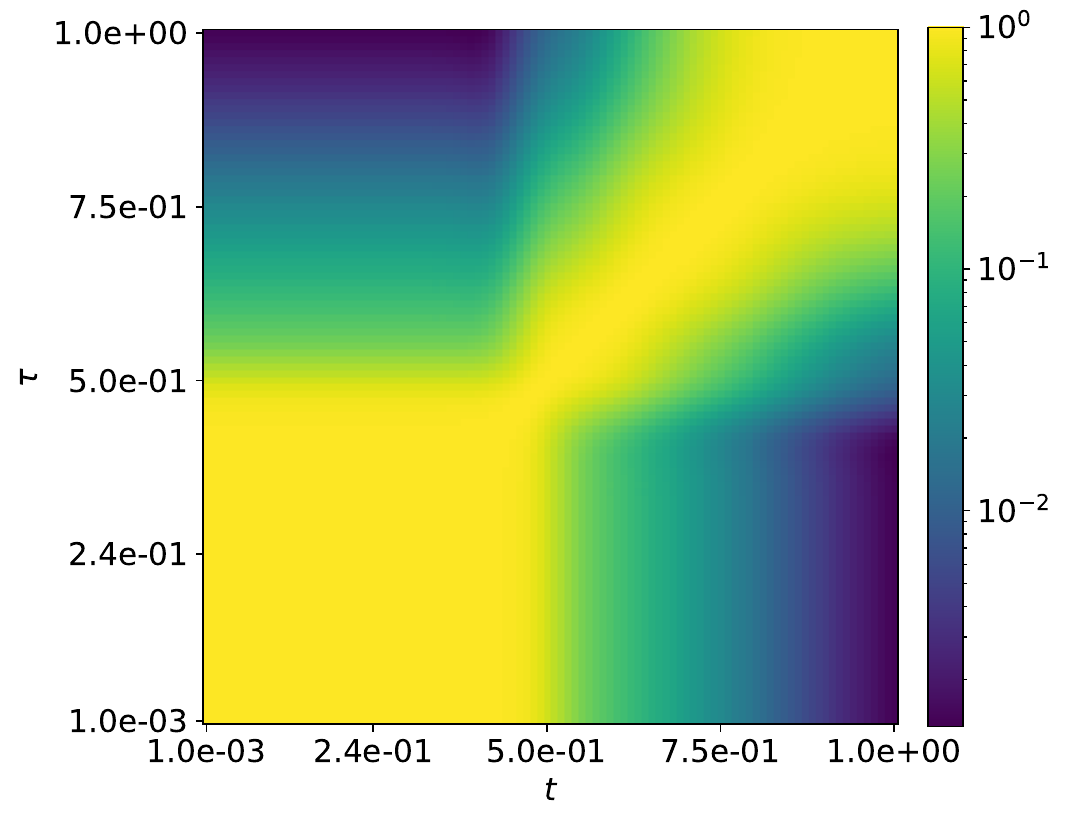}
    \end{subfigure}
    \hfill
    \begin{subfigure}[b]{0.48\textwidth}
        \centering
        \includegraphics[width=\textwidth]{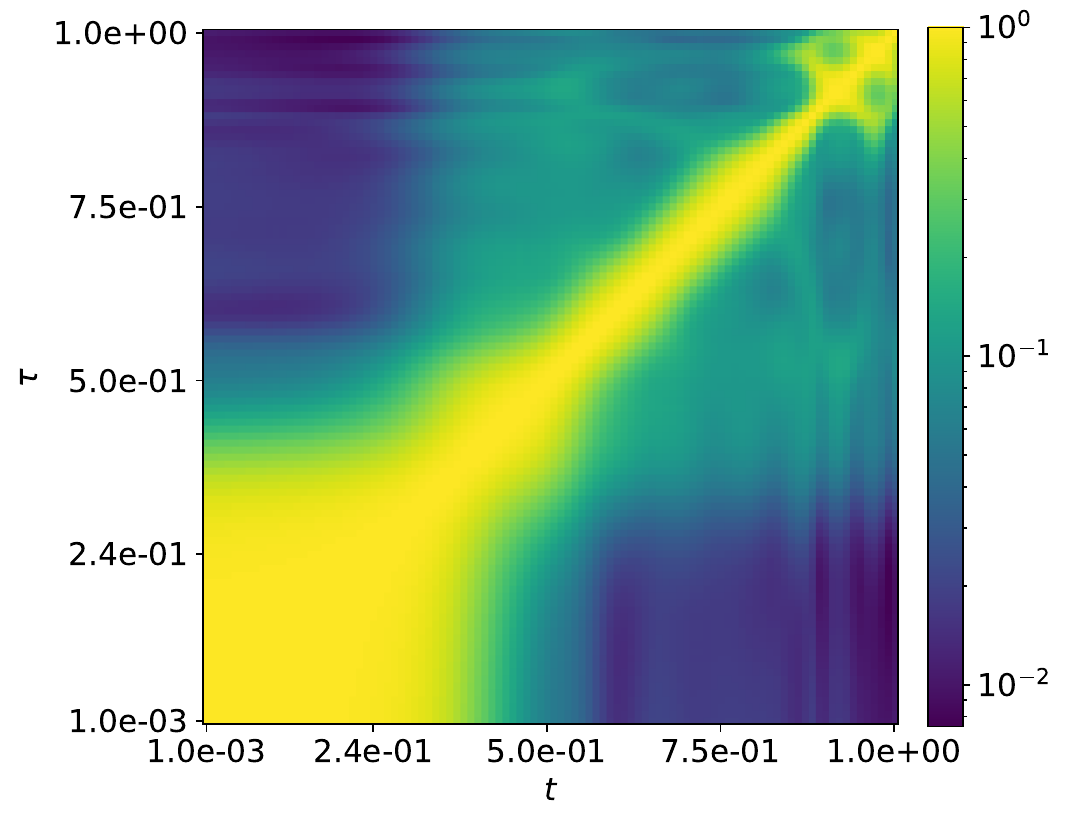}
    \end{subfigure}

    \caption{Strength of cross-time feature coupling between different times using the expected trace norm of \cref{eq:projection_operator_general} with a uniform time schedule with the mixture of delta distribution. The curves on the left refer to the parametric model, the ones on the right to the \gls{MLP}. From top to bottom, the first row correspond to the 2 mixture component distribution, the second to 3 components and the last to 4.}
    \label{fig:dirac_couplings}
\end{figure}

\subsection{Swissroll and Moons}\label{sec:2dtoy}

We also evaluate training performance on the Swiss Roll and Moons 2D datasets using the same configuration as \cref{sec:dirac} for the neural network, with visual depiction of the generated samples available in \cref{fig:2d_examples}. The corresponding learning curves are provided in \Cref{fig:2d_learning_curves}; these results show that the optimal schedule retains the best performances, closely followed by the entropic schedule. \Cref{fig:2d_sched_objs} and \cref{fig:2d_schedules} confirm what we already noticed in \cref{sec:dirac} about the number of atoms and the relation between the ELBO-weighted objective and the generative performance. The strength of cross-time couplings in \cref{fig:2d_couplings} shows a similar trend to the neural network experiments on the Dirac mixture as well, with weak couplings on high noise regions.

\begin{figure}[!ht]
    \centering
    \begin{subfigure}[b]{0.48\textwidth}
        \centering
        \includegraphics[width=\linewidth]{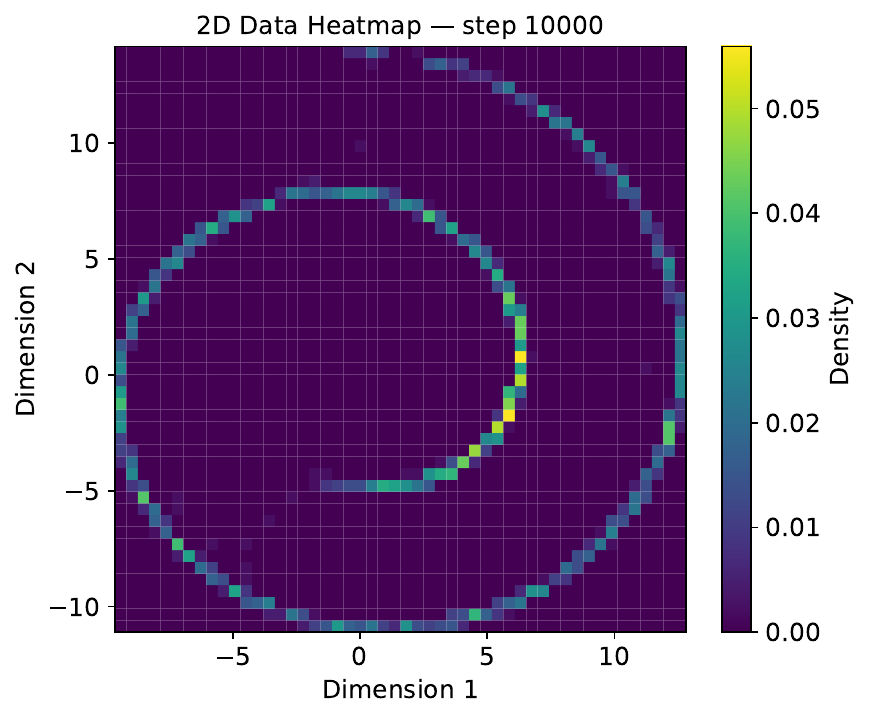}
        \label{fig:swissroll_ex}
    \end{subfigure}
    \hfill 
    \begin{subfigure}[b]{0.48\textwidth}
        \centering
        \includegraphics[width=\linewidth]{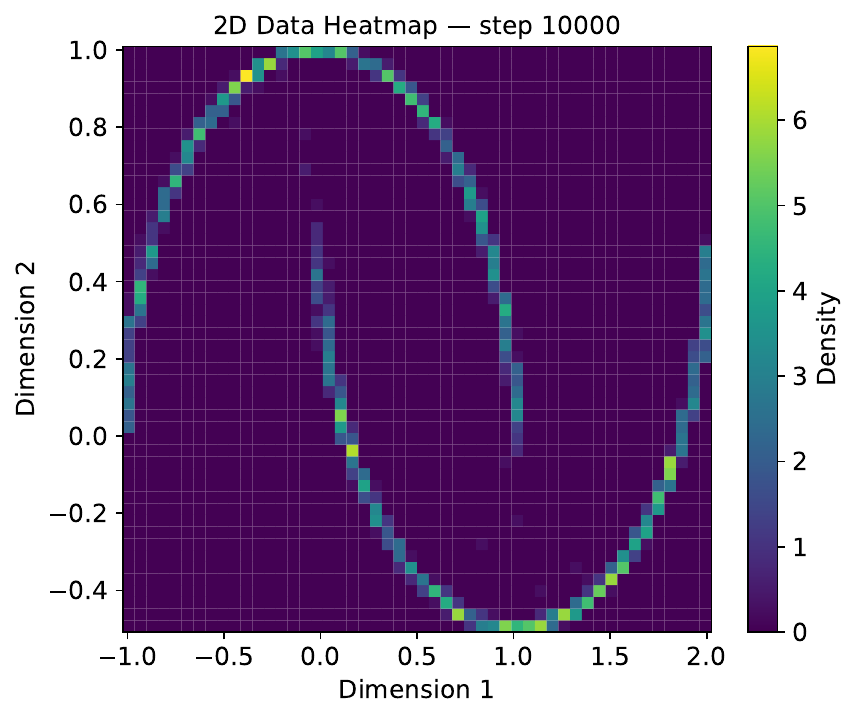}
        \label{fig:moons_ex}
    \end{subfigure}

    \caption{Empirical density of the samples generated by the diffusion models trained on the swissroll (left-hand side) and moons (right-hand side) distributions.}
    \label{fig:2d_examples}
\end{figure}

\begin{figure}[!ht]
    \centering
    \begin{subfigure}[b]{0.48\textwidth}
        \centering
        \includegraphics[width=\linewidth]{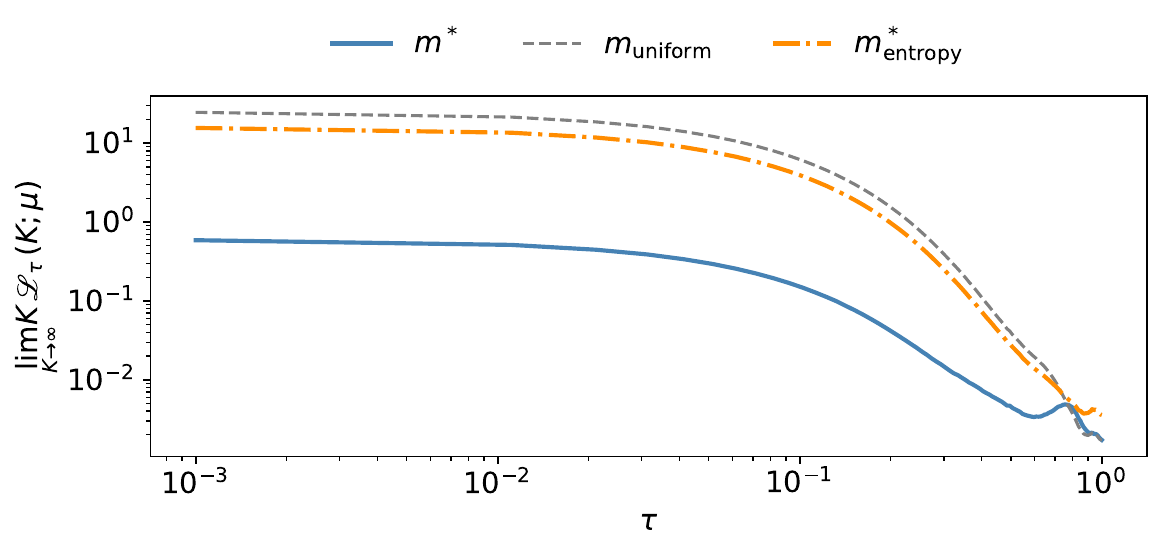}
        \label{fig:swissroll_objs}
    \end{subfigure}
    \hfill 
    \begin{subfigure}[b]{0.48\textwidth}
        \centering
        \includegraphics[width=\linewidth]{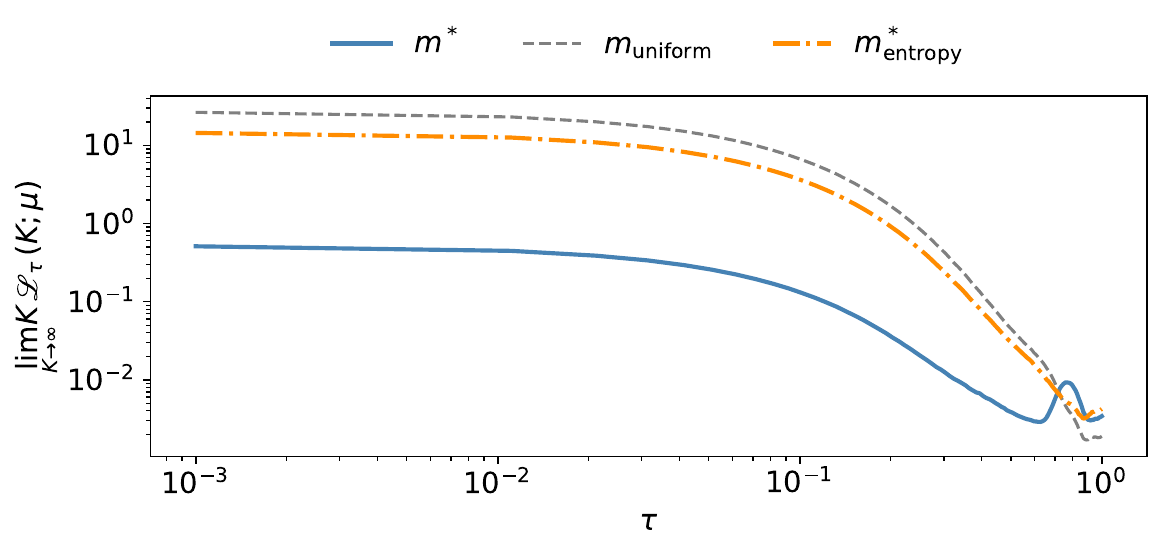}
        \label{fig:moons_objs}
    \end{subfigure}

    \caption{ELBO-weighted objective for the atomic, entropic and uniform schedules. We have the swissroll model on the left-hand side and the moons model on the right-hand side.}
    \label{fig:2d_sched_objs}
\end{figure}

\begin{figure}[!ht]
    \centering
    \begin{subfigure}[b]{0.48\textwidth}
        \centering
        \includegraphics[width=\linewidth]{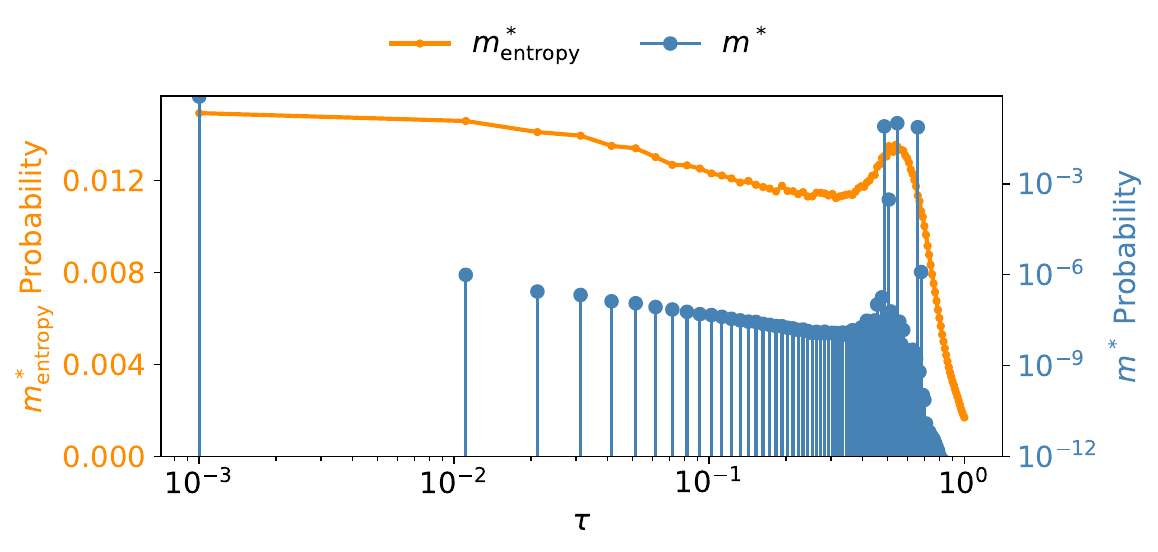}
        \label{fig:swissroll_schedules}
    \end{subfigure}
    \hfill 
    \begin{subfigure}[b]{0.48\textwidth}
        \centering
        \includegraphics[width=\linewidth]{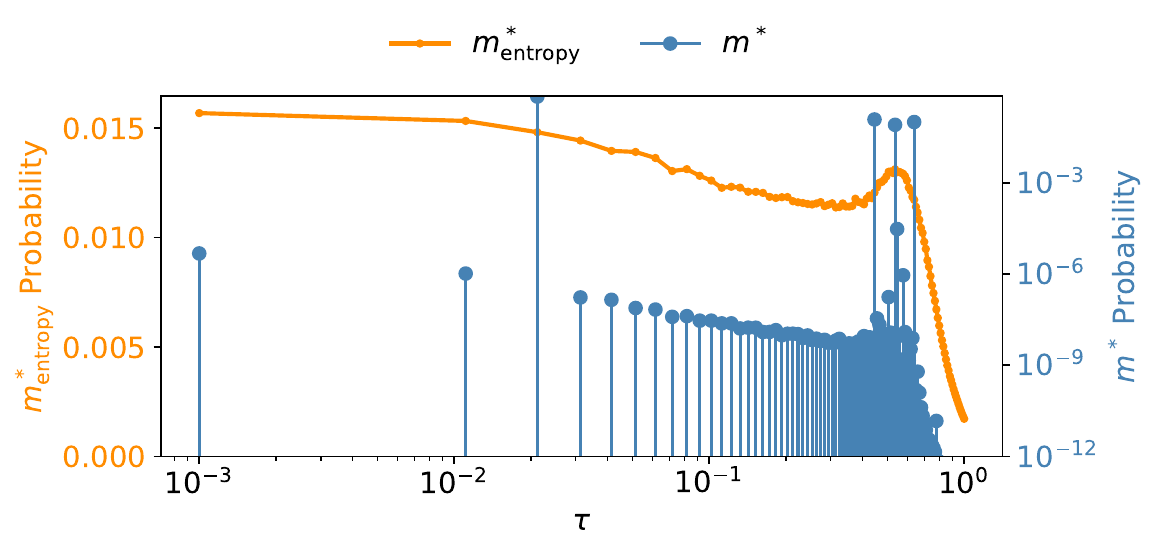}
        \label{fig:moons_schedules}
    \end{subfigure}

    \caption{Entropic and atomic schedules. for the swissroll model on the left-hand side and for the moons model on the right-hand side.}
    \label{fig:2d_schedules}
\end{figure}

\begin{figure}[!ht]
    \centering
    \begin{subfigure}[b]{0.48\textwidth}
        \centering
        \includegraphics[width=\linewidth]{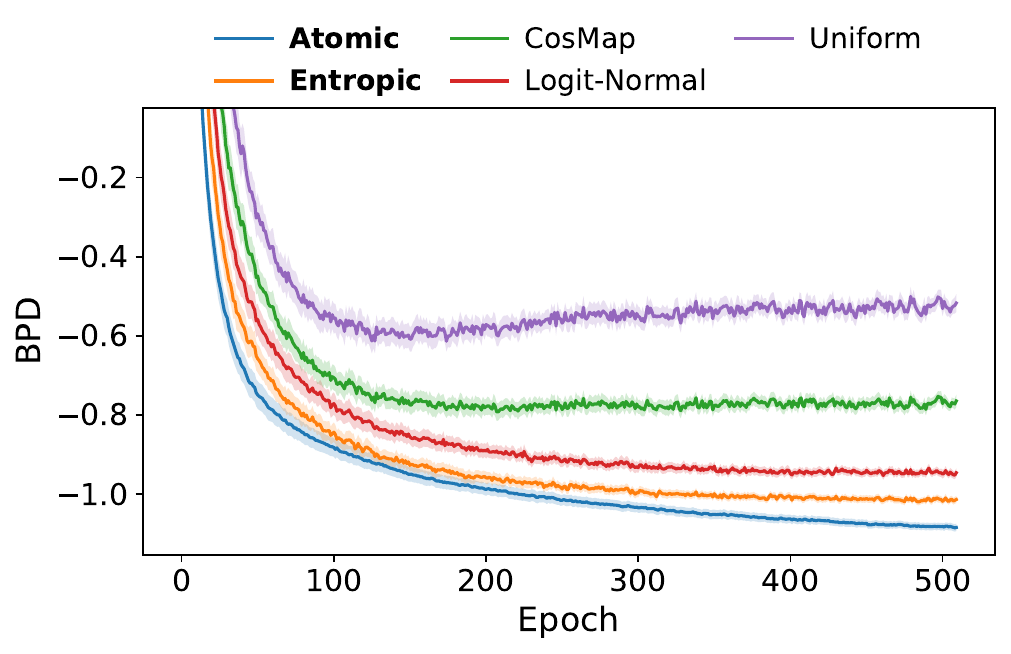}
    \end{subfigure}
    \hfill 
    \begin{subfigure}[b]{0.48\textwidth}
        \centering
        \includegraphics[width=\linewidth]{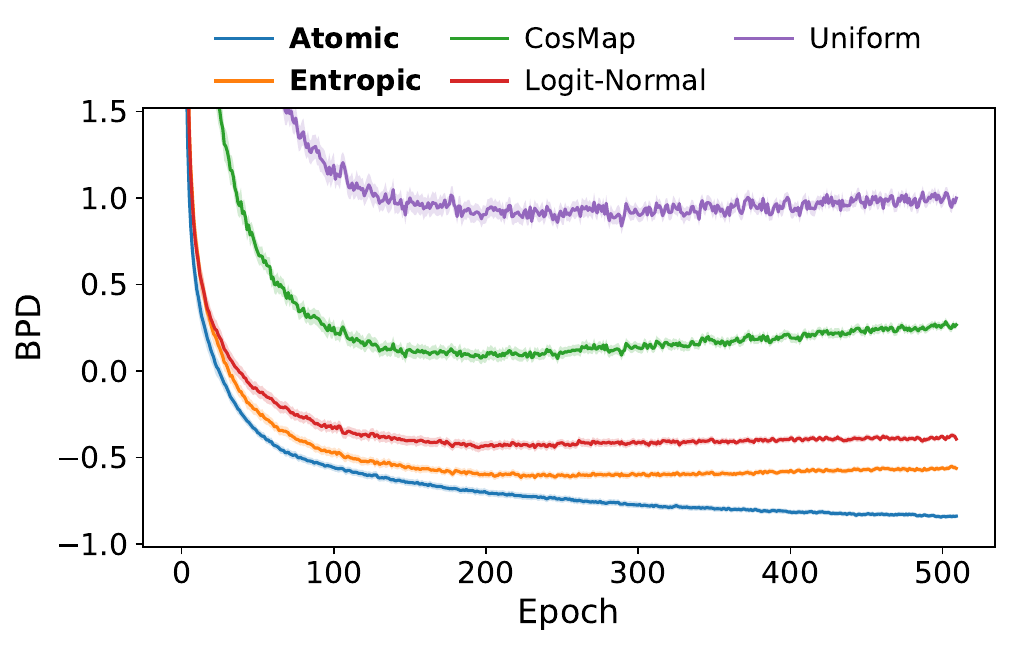}
    \end{subfigure}

    \caption{NLL Learning curve for each time schedule for the swissroll (left-hand side) and moons (right-hand side) distributions. Shaded areas indicate the \gls{SEM}.}
    \label{fig:2d_learning_curves}
\end{figure}

\begin{figure}[!ht]
    \centering
    \begin{subfigure}[b]{0.48\textwidth}
        \centering
        \includegraphics[width=\linewidth]{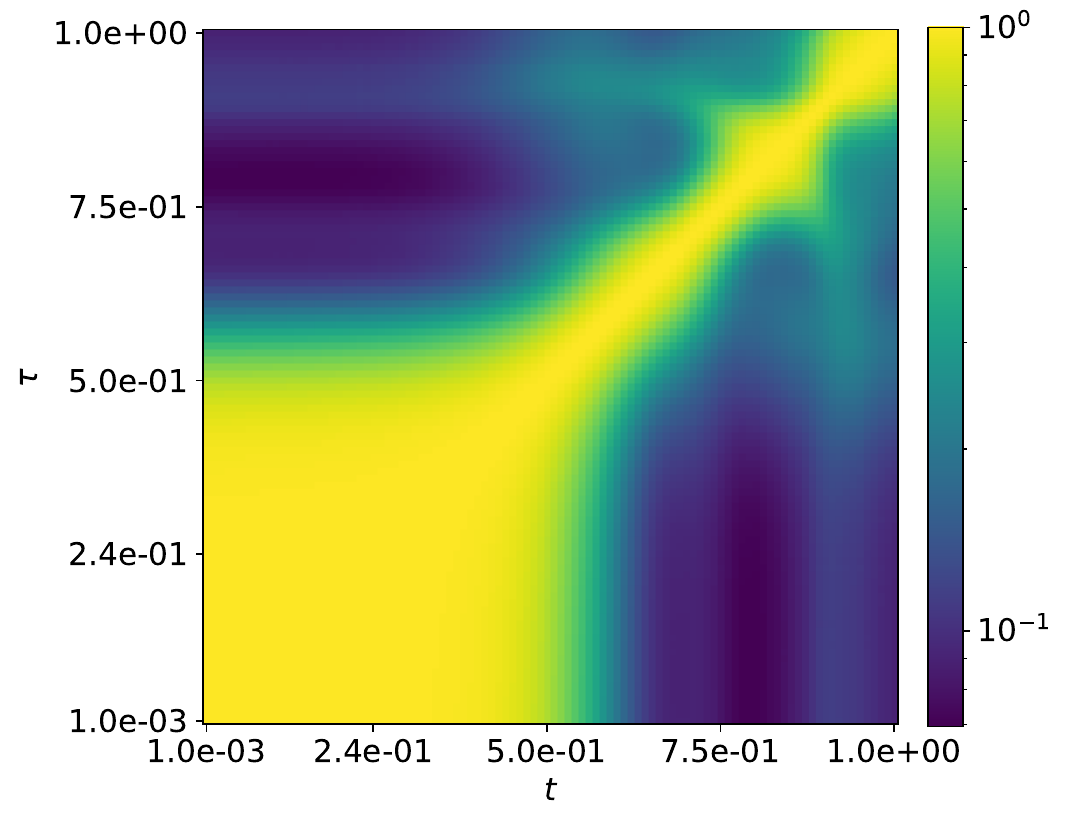}
    \end{subfigure}
    \hfill 
    \begin{subfigure}[b]{0.48\textwidth}
        \centering
        \includegraphics[width=\linewidth]{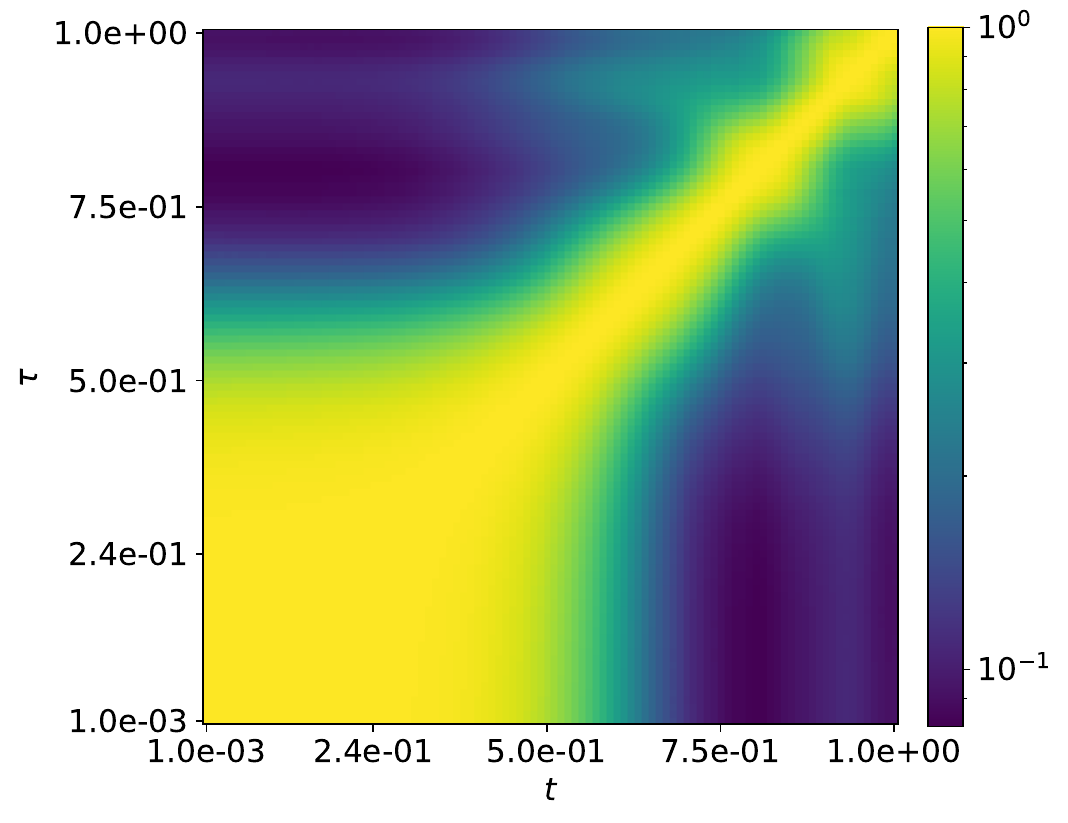}
    \end{subfigure}

    \caption{Strength of cross-time feature coupling between different times using the expected trace norm of \cref{eq:projection_operator_general} with a uniform time schedule for the swissroll (left-hand side) and moons (right-hand side) distributions.}
    \label{fig:2d_couplings}
\end{figure}

\subsection{MNIST}

In order to extend our analysis to higher-dimensional data, we train a variance-exploding diffusion model on MNIST using a 2D UNet with 14 million parameters. Because analyzing asymptotic objectives on the whole parameter count is unfeasible, we restrict the coupled optimization to the weights of the last convolutional kernel in the network, which amounts to 288 parameters. Again, \cref{fig:mnist_lc} shows that entropic and optimal schedules retain the best performance, with no clear winner. This is consistent with \cref{fig:mnist_coupling}: cross-time couplings are rather weak compared to other distributions, which means that the setting is closer to the assumptions required for the entropic schedule to perform well. We report the atomic and entropic schedules in \cref{fig:mnist_scheds} with their ELBO-weighted objectives in \cref{fig:mnist_objs}. The atomic schedule attains a massive atom close to clean data, while the rest follow more closely the shape of the entropic proxy.

\begin{figure}[!ht]
    \centering
    \begin{subfigure}[b]{0.48\textwidth}
        \centering
        \includegraphics[width=\textwidth]{figures_toy_case/objs_log_scale/objs_mnist.pdf}
        \caption{ELBO-weighted objective for the atomic, entropic and uniform schedules.}
        \label{fig:mnist_objs}
    \end{subfigure}
    \hfill 
    \begin{subfigure}[b]{0.48\textwidth}
        \centering
        \includegraphics[width=\textwidth]{figures_toy_case/schedules_log_scale/schedules_mnist.pdf}
        \caption{Atomic and entropic schedules.}
        \label{fig:mnist_scheds}
    \end{subfigure}

    \vspace{0.3cm} 

    \begin{subfigure}[b]{0.48\textwidth}
        \centering
        \includegraphics[width=\textwidth]{figures_toy_case/mnist/mnist_bpd_learning_curves.pdf}
        \caption{NLL Learning curve for each time schedule. Shaded areas indicate the \gls{SEM}.}
        \label{fig:mnist_lc}
    \end{subfigure}
    \hfill 
    \begin{subfigure}[b]{0.48\textwidth}
        \centering
        \includegraphics[width=\textwidth]{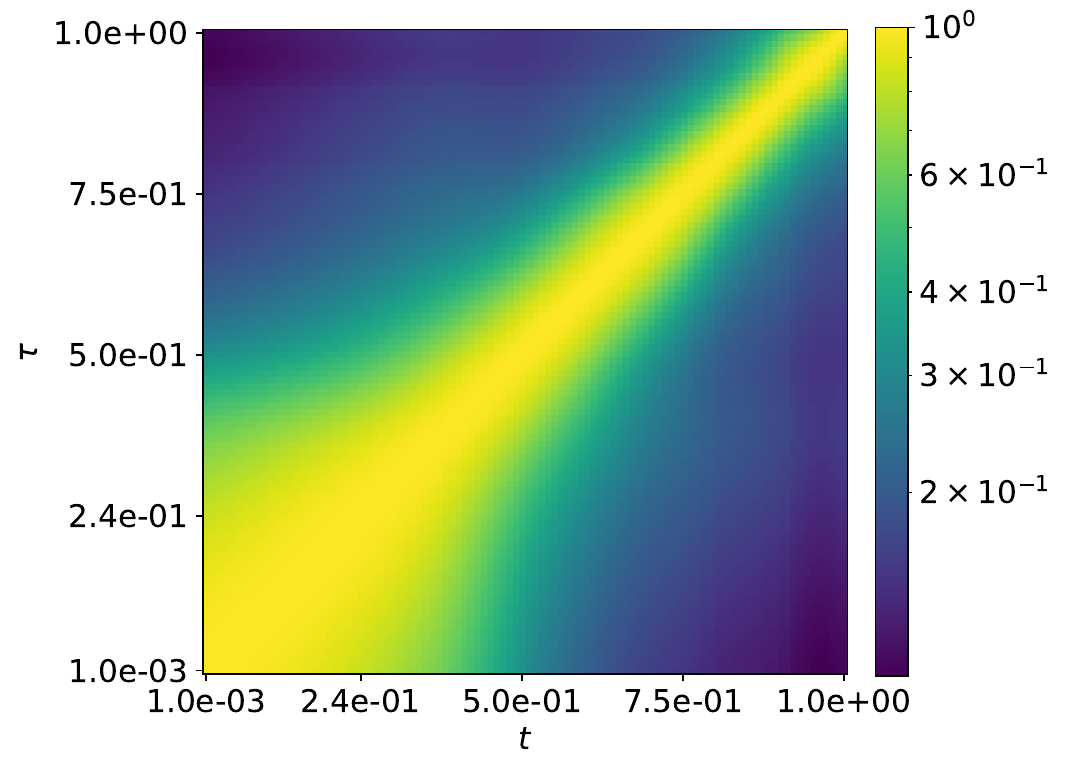}
        \caption{Strength of cross-time feature coupling between different times using the expected trace norm of \cref{eq:projection_operator_general} with a uniform time schedule.}
        \label{fig:mnist_coupling}
    \end{subfigure}
    \caption{Key visuals for the MNIST experiments.}
    \label{fig:mnist_all}
\end{figure}

%% file: appendices/schedule_opt_algo.tex
\begin{algorithm}[ht]
\caption{Asymptotic Schedule Optimization via Gradient Descent}
\label{alg:schedule_opt}
\begin{algorithmic}[1]
\Require Calibration dataset $\mathcal{D}$, evaluation time bounds $[t_{\min}, t_{\max}]$, number of bins $N$, optimization steps $K$, learning rate $\eta$.
\Ensure Optimal discrete time-sampling measure $\mu^\star$.

\State Initialize uniform time grid $\mathcal{T} = \{\tau_1, \dots, \tau_N\}$ over $[t_{\min}, t_{\max}]$.
\State Initialize schedule logits $\phi^{(0)} = \mathbf{0} \in \mathbb{R}^N$.

\Statex \textbf{Phase 1: Pre-compute Local Curvature and Noise Operators}
\For{$i = 1, \dots, N$}
    \State Compute pointwise curvature: 
    \Statex \qquad $A_{\tau_i} = \mathbb{E}_{x_0 \sim \mathcal{D}, x_{\tau_i}|x_0}\!\left[J_{\tau_i}(x_{\tau_i})^\top J_{\tau_i}(x_{\tau_i})\right]$
    \State Compute pointwise gradient noise: 
    \Statex \qquad $B_{\tau_i} = \mathbb{E}_{x_0 \sim \mathcal{D}, x_{\tau_i}|x_0}\!\left[J_{\tau_i}(x_{\tau_i})^\top \Sigma_{\tau_i}(x_{\tau_i}) J_{\tau_i}(x_{\tau_i})\right]$
\EndFor

\Statex \textbf{Phase 2: Schedule Optimization Loop}
\For{$k = 1, \dots, K$}
    \State Apply softmax to obtain the current probability measure: 
    \Statex \qquad $\mu^{(k)} = \operatorname{Softmax}(\phi^{(k-1)})$
    \State Assemble global curvature and noise operators:
    \Statex \qquad $H[\mu^{(k)}] = \sum_{i=1}^N \mu^{(k)}_i A_{\tau_i}$ \quad \text{and} \quad $\Gamma[\mu^{(k)}] = \sum_{i=1}^N \mu^{(k)}_i B_{\tau_i}$
    \For{$j = 1, \dots, N$}
        \State Evaluate the unweighted cross-time objective at $\tau_j$:
        \Statex \qquad $U_j = \operatorname{Tr}\!\left(A_{\tau_j} H[\mu^{(k)}]^{-1} \Gamma[\mu^{(k)}] H[\mu^{(k)}]^{-1}\right)$
        \State Weight by the signal-to-noise ratio derivative:
        \Statex \qquad $E_j = \dot{\mathrm{SNR}}(\tau_j) \, U_j$
    \EndFor
    \State Approximate the continuous objective integral via trapezoidal rule: 
    \Statex \qquad $\mathcal{J}(\phi^{(k-1)}) \approx \operatorname{Trapz}(\{E_j\}_{j=1}^N, \mathcal{T})$
    \State Update logits using an optimizer (e.g., Adam or SGD):
    \Statex \qquad $\phi^{(k)} \leftarrow \phi^{(k-1)} - \eta \nabla_\phi \mathcal{J}(\phi^{(k-1)})$
\EndFor
\State \Return $\mu^\star = \operatorname{Softmax}(\phi^{(K)})$
\end{algorithmic}
\end{algorithm}

%% file: appendices/experimental_details.tex
\section{Practical Estimation of the Entropy Rate}\label{appx:experimental_details}

From the de Bruijn identity
$\dot{\Ent}[x_0|x_t;t,\phi]=\dot{\mathrm{SNR}}(t)\,\mathrm{MMSE}(t)$
(Supp.~\ref{supp sec: conditional entropy}), equivalently
$-\frac{d}{dt}I(X_0;X_t)$, it follows that, given a trained denoising model
$s_\theta(x_t,t)$, we can estimate the entropy rate from a dataset
$\{y^j, \dots, y^N\}$ as the scaled denoising error
\begin{equation}\label{eq:entropy_rate_estimator}
    \hat{\dot{\Ent}}(t) = \dot{\mathrm{SNR}}(t)\,\frac{1}{N}\sum_{j=1}^N
    \bigl\|y^{j} - s_\theta(y^{j} + \sigma(t)\,z^{j},\,t)\bigr\|^2,
\end{equation}
where $z^j\sim\mathcal{N}(0,I)$ independently. The estimator is unbiased when the network is
perfectly trained. Since \eqref{eq:entropy_rate_estimator} is essentially the training loss
evaluated on a held-out set, it does not need to be re-computed separately and can be accumulated
during training at negligible overhead. \citet{stancevic2026entropic} show that these
estimates are stable and accurate even for dataset sizes of a few hundred points.

In practice the entropy curve, or equivalently the mutual-information decay curve, must be
estimated either in advance from a pre-trained model or adaptively adjusted as training
progresses. In all experiments we estimate these curves in advance from pre-trained diffusion models,
focusing on schedule optimization rather than the online estimation procedure.